\documentclass[10pt,journal,compsoc]{IEEEtran}
\ifCLASSOPTIONcompsoc
  \usepackage[nocompress]{cite}
\else
  \usepackage{cite}
\fi
\ifCLASSINFOpdf
\else
\fi
\usepackage{amsmath}
\def\cost{\mbox{cost}}

\def\sign{\mbox{sign}}

\def\lu#1{\textcolor{blue}{#1}} 
\def\rev#1{\textcolor{black}{#1}} 

\usepackage{amsfonts}
\usepackage{subfigure}
\usepackage{tabularx}
\usepackage{multirow}
\usepackage{graphicx}
\usepackage{amsthm}

\def\Sbf{\mathbf{S}}

\def\Pcal{\mathcal{P}}
\def\Scal{\mathcal{S}}

\newtheorem{definition}{Definition}[section]
\newtheorem{lemma}{Lemma}[section]
\newtheorem{theorem}{Theorem}[section]
\newtheorem{corollary}{Corollary}[theorem]
\usepackage[lined,ruled,linesnumbered,noend]{algorithm2e}
\usepackage{comment}
\usepackage{xcolor}
\usepackage{makecell}

\begin{document}
%
\title{Communication-efficient $k$-Means for Edge-based Machine Learning}
\author{Hanlin~Lu,~\IEEEmembership{Student~Member,~IEEE,}
        Ting~He,~\IEEEmembership{Senior~Member,~IEEE,}
        Shiqiang~Wang,~\IEEEmembership{Member,~IEEE,}
        Changchang~Liu,
        Mehrdad~Mahdavi,
        Vijaykrishnan~Narayanan,~\IEEEmembership{Fellow,~IEEE,}
        Kevin~S.~Chan,~\IEEEmembership{Senior~Member,~IEEE,}
        and~Stephen~Pasteris
\IEEEcompsocitemizethanks{\IEEEcompsocthanksitem H. Lu, T. He, M. Mahdavi, and V. Narayanan are with Pennsylvania State University, University Park, PA 16802, USA (email: \{hzl263, tzh58, mzm616, vxn9\}@psu.edu). 
\IEEEcompsocthanksitem S. Wang and C. Liu are with IBM T. J. Watson Research Center, Yorktown Heights, NY 10598, USA (email: \{wangshiq@us., Changchang.Liu33@\}ibm.com).
\IEEEcompsocthanksitem K. Chan is with Army Research Laboratory, Adelphi, MD 20783, USA (email: kevin.s.chan.civ@mail.mil). 
\IEEEcompsocthanksitem S. Pasteris is with University College London, London WC1E 6EA, UK (email: s.pasteris@cs.ucl.ac.uk).
}
\thanks{\scriptsize A preliminary version of this work was presented at ICDCS'20. \cite{Lu20ICDCS}.}
}

\IEEEtitleabstractindextext{%
\begin{abstract}
We consider the problem of computing the $k$-means centers for a large high-dimensional dataset in the context of edge-based machine learning, where data sources offload machine learning computation to nearby edge servers. $k$-Means computation is fundamental to many data analytics, and the capability of computing provably accurate $k$-means centers by leveraging the computation power of the edge servers, at a low communication and computation cost to the data sources, will greatly improve the performance of these analytics. We propose to let the data sources send small summaries, generated by joint dimensionality reduction (DR), cardinality reduction (CR), and quantization (QT), to support approximate $k$-means computation at reduced complexity and communication cost. By analyzing the complexity, the communication cost, and the approximation error of $k$-means algorithms based on carefully designed composition of DR/CR/QT methods, we show that: (i) it is possible to compute near-optimal $k$-means centers at a near-linear complexity and a constant or logarithmic communication cost, (ii) the order of applying DR and CR significantly affects the complexity and the communication cost, and (iii) combining DR/CR methods with a properly configured quantizer can further reduce the communication cost without compromising the other performance metrics. Our theoretical analysis has been validated through experiments based on real datasets. 
\end{abstract}

}

\maketitle

\IEEEdisplaynontitleabstractindextext

%
\IEEEpeerreviewmaketitle

\IEEEraisesectionheading{\section{Introduction}\label{sec:introduction}}
\IEEEPARstart{E}{dge}-based machine learning~\cite{park2019wireless} is an emerging application scenario, 
where mobile/wireless devices collect data and transmit them (or their summaries) to nearby edge servers for processing. Compared to alternative approaches, e.g., transmitting locally learned model parameters as in federated learning \cite{WangJSAC2019}, transmitting data summaries has the advantages that: (i) only one round of communications is required,\footnote{In cases that the raw data are spread over multiple nodes, another round of communications is needed to decide the sizes of data summaries to collect from each node \cite{Balcan13NIPS}. However, each node only sends one scalar in this round and hence the communication cost is negligible. } (ii) the transmitted data can potentially be used to compute other machine learning models \cite{Lu19Globecom,lu2020robust}, and (iii) the edge server can solve the machine learning problem closer to the optimality than the data-collecting devices within the same time. 
In this work, we focus on $k$-means clustering under the framework of edge-based machine learning. 


$k$-Means clustering is one of the most widely-used machine learning techniques. Algorithms for $k$-means are used in many areas of data science, e.g., for data compression, quantization, hashing; see the survey in \cite{Jain10PRL} for more details. Recently, it was shown in \cite{Lu19Globecom,lu2020robust} that the centers of $k$-means can be used as a proxy of the original dataset in computing a broader set of machine learning models with sufficiently continuous cost functions. Thus, efficient and accurate computation of $k$-means can bring broad benefits to machine learning applications.

However, solving $k$-means is nontrivial. The problem is known to be NP-hard, even for two centers \cite{ADHP09} or in the plane \cite{MNV12}. Due to its fundamental importance, how to speed up the $k$-means computation for large datasets has received significant attention. Most existing solutions can be classified into two approaches: \emph{dimensionality reduction (DR)} methods that aim at generating a ``thinner'' dataset with a reduced number of attributes~\cite{Makarychev19STOC}, and \emph{cardinality reduction (CR)} methods that aim at generating a ``smaller'' dataset with a reduced number of data points (i.e., samples)~\cite{Feldman13SODA:report}. \rev{However, these solutions assumed that the full dataset is available locally at the server that performs $k$-means computation, and hence ignored the communication cost.} 

To our knowledge, we are the first to explicitly analyze the communication cost in computing $k$-means over remote and possibly distributed data. \rev{The need of communications arises in the application scenario of edge-based machine learning, where edge devices collecting data wish to offload machine learning computation to nearby edge servers through wireless links.} 
%
%
Given a large high-dimensional dataset, i.e., $n, d\gg 1$ ($n$: cardinality, $d$: dimension), residing at one or multiple data sources at the network edge, an obvious solution of solving $k$-means at the data sources and sending the centers to the server will incur a high computational complexity at the data sources \rev{that is not suitable for the limited computation power of edge devices}, while another obvious solution of sending the raw data to the server and solving $k$-means there will incur a high communication cost \rev{that imposes too much stress on the wireless links}. 
We seek to achieve a better tradeoff by letting the data sources send small data summaries generated by efficient data reduction methods and leaving the $k$-means computation to the server. \looseness=-1

\rev{Besides DR and CR, quantization (QT)~\cite{Sayood:2012:IDC:3050831}  
can also reduce the communication cost by representing each data point with a smaller number of bits. 
While $k$-means itself has been used to design vector quantizers~\cite{gersho2012vector}, we will show that simpler quantizers can be combined with DR/CR methods to compute approximate $k$-means at an even lower communication cost without negatively affecting the complexity or the quality of solution.}

\subsection{Summary of Contributions}

\rev{We want to develop efficient $k$-means algorithms suitable for edge-based machine learning, by offloading as much computation as possible to edge servers at a low communication cost to data sources.} Our contributions include:\looseness=-1
 
 1) If the data reside at a single data source, we show that (i) it is possible to solve $k$-means arbitrarily close to the optimal with constant communication cost and near-linear complexity at the data source by combining suitably selected DR/CR methods, (ii) the order of applying DR and CR methods will not affect the approximation error, but will lead to different tradeoffs between communication cost and computational complexity, and (iii) repeating DR both before and after CR can further improve the performance.  

2) If the data are distributed over multiple data sources, we show that suitably combining DR/CR methods can solve $k$-means arbitrarily close to the optimal with near-linear complexity at the data sources and a total communication cost that is logarithmic in the data size. 

3) We further extend our solution to include quantization. Using the rounding-based quantizer as an example, we demonstrate how to configure the quantizer to minimize the communication cost while guaranteeing a given approximation error. \looseness=-1

4) Through experiments on real datasets, we verify that (i) joint DR and CR can drastically reduce the communication cost without incurring a high complexity at the data sources or significantly degrading the solution quality, (ii) the proposed joint DR-CR algorithms can achieve a solution quality similar to state-of-the-art algorithms while notably reducing the communication cost and the complexity, and (iii) combining DR/CR with quantization can further reduce the communication cost without compromising the other performance metrics.     

\textbf{Roadmap.} 
Sections~\ref{sec:Related Work}--\ref{sec:Background} review the background on DR/CR methods. Section~\ref{sec:Joint Dimension and Cardinality Reduction} presents our results on joint DR/CR in the centralized setting, and Section~\ref{sec:Multiple Data Sources} presents those in the distributed setting.  Section~\ref{sec:Extension} presents further improvement via joint DR, CR, and quantization. Section~\ref{sec:evaluations for DR+CR+QT} evaluates our solutions on real datasets. Finally, Section~\ref{sec:Conclusion} concludes the paper. Proofs are given in Appendix. \looseness=-1

\section{Related Work}\label{sec:Related Work}

Our work belongs to the studies on data reduction for approximate $k$-means. \rev{Existing solutions can be classified into the following categories:} 

\textbf{Dimensionality reduction (DR):} DR for $k$-means, initiated by \cite{BZD10}, aims at speeding up $k$-means by reducing the number of features (i.e., the dimension). Two approaches have been proposed: 1) \emph{feature selection} that selects a subset of the original features, and 2) \emph{feature extraction} that constructs a smaller set of new features.
For feature selection, the best known algorithms are from \cite{CEM15}, including a random sampling algorithm that achieves a $(1+\epsilon)$-approximation using $O(k\log k/\epsilon^2)$ features, and a deterministic algorithm that achieves a $(1+\epsilon)$-approximation using $O(k/\epsilon^2)$ features.
For feature extraction, there are two methods with guaranteed approximation, both based on linear projections. The first method is based on \emph{principal component analysis (PCA)} via computing the \emph{singular value decomposition (SVD)}, where exact SVD gives $2$-approximation using $k$ features \cite{Drineas04ML} or $(1+\epsilon)$-approximation using $\lceil k/\epsilon \rceil$ features \cite{CEM15}, and approximate SVD gives $(2+\epsilon)$-approximation using $k$ features \cite{Boutsidis15IT} or $(1+\epsilon)$-approximation using $\lceil k/\epsilon \rceil$ features \cite{CEM15}. The second method is based on \emph{random projections} that preserve vector $\ell$-2 norms with an arbitrarily high probability, whose existence is guaranteed by the \emph{Johnson-Lindenstrauss (JL) lemma} \cite{JL82}. The best known algorithm there is given by \cite{Makarychev19STOC}, which achieves a $(1+\epsilon)$-approximation using $O(log(k/\epsilon)/\epsilon^2)$ features.

\textbf{Cardinality reduction (CR):} CR for $k$-means, initiated by \cite{Har-Peled04STOC}, aims at using a small weighted set of points in the same space, referred to as a \emph{coreset}, to replace the original dataset. A coreset is called an $\epsilon$-coreset (for $k$-means) if it can approximate the $k$-means cost of the original dataset for every candidate set of centers up to a factor of $1\pm \epsilon$. Many coreset construction algorithms have been proposed for $k$-means. 
Early algorithms use geometric partitions to merge each group of nearby points into a single coreset point \cite{Har-Peled04STOC,Frahling05STOC,Har-Peled07DCG}, which cause the cardinality of the coreset to be exponential in the dimension $d$. Later, \cite{Chen09JC} showed that sampling can be used to reduce the coreset cardinality to a polynomial in $k, \epsilon, \log{n}$, and $d$. Most state-of-the-art coreset construction algorithms are based on the \emph{sensitivity sampling} framework that was first proposed in \cite{Langberg10SODA} and then formalized in \cite{Feldman11STOC}. To generate an $\epsilon$-coreset, the solution in \cite{Feldman11STOC} needs a coreset cardinality of\footnote{We use $\tilde{O}(x)$ to denote a value that is at most linear in $x$ times a factor that is polylogarithmic in $x$.} $\tilde{O}(kd\epsilon^{-4})$, and its followup in \cite{Braverman16CoRR} needs $\tilde{O}(k^2 d\epsilon^{-2})$. The best known solution\footnote{There is another algorithm with a coreset cardinality independent of $n$ and $d$ (precisely, $k^{O(\epsilon^{-2})}$) in \cite{Barger16SDM}. However, the algorithm has a high complexity 
and the coreset cardinality is larger than that in \cite{Feldman13SODA:report}. } is the one in \cite{Feldman13SODA:report} (presented implicitly in the proof of Theorem~36), which showed that by reducing the intrinsic dimension of the dataset and adding a constant term to the coreset-based cost, the cardinality of an $\epsilon$-coreset can be reduced to $\tilde{O}(k^3 \epsilon^{-4})$. 


\textbf{Joint DR-CR:}
Among the above works, only \cite{Feldman13SODA:report,Balcan14NIPS} considered joint DR and CR for $k$-means computation. 

\textbf{Algorithms in the distributed setting:} In this setting, \cite{Balcan13NIPS} proposed a distributed version of sensitivity sampling to construct an $\epsilon$-coreset over a distributed dataset, and \cite{Balcan14NIPS} further combined this algorithm with a distributed PCA algorithm from \cite{Feldman13SODA:report}. 
Besides these theoretical results, there are also system works on adapting centralized $k$-means algorithms for distributed settings, e.g., MapReduce \cite{Mao15ICFCST}, sensor networks \cite{Naldi13BCIS}, and Peer-to-Peer networks \cite{Giannella09TKDE}. 
However, these algorithms are only heuristics. 

\rev{\emph{Limitations \& improvements:} While extensively studied, existing solutions mainly focused on reducing the computation time, leaving open the issue of communication cost. Moreover, we note that: (i) the state-of-the-art data reduction methods \cite{Feldman13SODA:report,Balcan14NIPS} blindly assumed that DR should be applied before CR, leaving open whether it is possible to achieve better performance by reversing the order of DR/CR or applying them repeatedly, and (ii) most of the algorithms for the distributed setting are heuristics without guarantees on how well their solutions approximate the optimal solution. To fill this gap, we will perform a comprehensive analysis in terms of computational complexity, communication cost, and approximation error, while carefully designing the order of applying DR/CR. In addition to DR and CR, QT~\cite{Sayood:2012:IDC:3050831} is also an effective method for reducing the communication cost by lowering the data precision. While $k$-means itself has been used to design certain quantizers~\cite{gersho2012vector}, the use of simpler quantizers for communication-efficient $k$-means computation has not been studied before. In this regard, we will show how to properly combine a simple rounding-based quantizer with DR/CR methods to further reduce the communication cost without compromising the other performance metrics.\looseness=0 
}

\section{Background and Formulation}\label{sec:Background}

We start with an overview of existing results on DR and CR for $k$-means, followed by our problem statement.

\subsection{Notations \& Definitions}\label{subsec:notation}

\textbf{Definitions:} Consider a dataset $P\subset \mathbb{R}^d$ with cardinality $n$ that resides at one or multiple data sources (i.e., data-collecting devices), where both $n\gg 1$ and $d\gg 1$. We want to find, with assistance of an edge server, the $k$ points $X=\{x_i\}_{i=1}^k$ that minimize the following cost function\footnote{The norms in (\ref{eq:cost(P,x)}) and (\ref{eq:cost(C)}) refer to the $\ell$-2 norm. }:
\begin{align}\label{eq:cost(P,x)}
\cost(P,X) := \sum_{p\in P} \min_{x_i\in X}\|p-x_i\|^2.
\end{align}
This is the \emph{$k$-means clustering} problem, and the points in $X$ are called \emph{centers}. Equivalently, the $k$-means clustering problem can be considered as the problem of finding the partition $\Pcal = \{P_1,\ldots,P_k\}$ of $P$ into $k$ clusters that minimizes the following cost function:
\begin{align}\label{eq:cost(C)}
\cost(\Pcal) := \sum_{i=1}^k \min_{x_i\in \mathbb{R}^d} \sum_{p\in P_i}\|p-x_i\|^2.
\end{align}

\textbf{Notations:} 
We will use $\|x\|$ to denote the $\ell$-2 norm if $x$ is a vector, or the Frobenius norm if $x$ is a matrix. We will use $A_P\in \mathbb{R}^{n\times d}$ to denote the matrix representation of a dataset $P\subset \mathbb{R}^d$, where each row corresponds to a data point.
Let $\mu(P)$ denote the optimal $1$-means center of $P$, which is well-known to be the sample mean, i.e., $\mu(P) = { \frac{1}{|P|}}\sum_{p\in P}p$. Let $\Pcal_{P,X}$ denote the partition of dataset $P$ induced by centers $X$, i.e., $\Pcal_{P,X}=\{P_1,\ldots,P_{|X|}\}$ for $P_i := \{p\in P:\: \|p-x_i\|\leq \|p-x_j\|, \: \forall x_j\in X\setminus \{x_i\}\}$ (ties broken arbitrarily). Given scalars $x$, $y$, and $\epsilon$ ($\epsilon>0$), we will use $x\approx_{1+\epsilon} y$ to denote ${ \frac{1}{ 1+\epsilon}}x \leq y \leq (1+\epsilon) x$. In our analysis, we use $O(x)$ to denote a value that is at most linear in $x$, $\Omega(x)$ to denote a value that is at least linear in $x$, and $\tilde{O}(x)$ to denote a value that is at most linear in $x$ times a factor that is polylogarithmic in $x$. \looseness=-1

Given a dimensionality reduction map $\pi: \mathbb{R}^d \to \mathbb{R}^{d'}$ ($d'<d$), we use $\pi(P):=\{\pi(p):\: p\in P\}$ to denote the output dataset for an input dataset $P$, and $\pi(\Pcal) := \{\pi(P_1),\ldots,\pi(P_k)\}$ to denote the partition of $\pi(P)$ corresponding to a partition $\Pcal = \{P_1,\ldots,P_k\}$ of $P$. Moreover, given a partition $\Pcal' = \pi(\Pcal)$, we use $\pi^{-1}(\Pcal')$ to denote the corresponding partition of $P$, which puts $p,q\in P$ into the same cluster if and only if $\pi(p), \pi(q)\in P'$ belong to the same cluster under $\Pcal'$. Finally, given $P' = \pi(P)$, we use $\pi^{-1}(P'):= \{\pi^{-1}(p'):\: p'\in P'\}$ to denote a set of points in $\mathbb{R}^d$ that is mapped to $P'$ by $\pi$.
Note that there is no guarantee that $\pi^{-1}(P') = P$. However, suppose $\tilde{P}$ is the solution which satisfies $\pi(\tilde{P}) = P'$, then $\tilde{P}$ must exist ($P$ is a feasible solution) and $\pi^{-1}(P')$ denotes an arbitrary solution. If $\pi$ is a linear map, i.e., $\pi(P) := A_P \Pi$ for a matrix $\Pi\in \mathbb{R}^{d\times d'}$, then the \emph{Moore-Penrose inverse $\Pi^+$} \cite{Ben-Israel:book} of $\Pi$ gives a feasible solution $\pi^{-1}(P') := A_{P'}\Pi^+$. 

\rev{The main notations and abbreviations used in the paper are summarized in Table~\ref{tab:notation}.}

\begin{table}[tb]
\vspace{-.0em}
\footnotesize
\renewcommand{\arraystretch}{1.3}
\caption{\rev{Key Abbreviations and Notations}} \label{tab:notation}
\vspace{-1em}
\centering
\begin{tabular}{|l|l|}
  \hline
  \rev{Notation} & \rev{Explanation} \\
  \hline
  \rev{DR} & \rev{dimensionality reduction} \\
  \hline 
  \rev{CR} & \rev{cardinality reduction}\\
  \hline
  \rev{QT} & \rev{quantization} \\
  \hline
  \rev{PCA} & \rev{principal component analysis} \\
  \hline
  \rev{JL projection} & \rev{a linear projection satisfying Theorem~\ref{thm:epsilon-projection}} \\
  \hline
  \rev{FSS} & \rev{the algorithm proposed in \cite[Theorem~36]{Feldman13SODA:report}}\\
  \hline
  \rev{BKLW} & \makecell[l]{\rev{the distributed version of FSS proposed} \\ \rev{in \cite[Algorithm~1]{Balcan14NIPS}}}
  \\
  \hline
  $P$ & \rev{original input dataset} \\
  \hline
  $n$ & \rev{cardinality of $P$} \\
  \hline
  $d$ & \rev{dimensionality of $P$} \\
  \hline
  $k$ & \rev{number of clustering centers} \\
  \hline
  $X$ & \rev{a set of clustering centers} \\
  \hline
  $\Pcal$ & \rev{a partition of $P$} \\
  \hline
  $\mu(P)$ & \rev{the optimal 1-means center of $P$}\\
  \hline
  $\Pcal_{P,X}$ & \rev{the partition of dataset $P$ induced by centers $X$} \\
  \hline
\end{tabular}
\vspace{-1em}
\end{table}

\subsection{Dimensionality Reduction for $k$-Means}\label{subsec:dimensionality reduction}

\begin{definition}\label{def:epsilon-projection}
We say that a DR map $\pi:\: \mathbb{R}^d\to \mathbb{R}^{d'}$ ($d'< d$) is an \emph{$\epsilon$-projection} if it preserves the cost of any partition up to a factor of $1+\epsilon$, i.e., $\cost(\Pcal)\approx_{1+\epsilon} \cost(\pi(\Pcal))$ for every partition $\Pcal = \{P_1,\ldots,P_k\}$ of a finite set $P\subset \mathbb{R}^d$. 
\end{definition}

One commonly used method to construct $\epsilon$-projection is random projection, where the cornerstone result is the JL Lemma:

\begin{lemma}[\cite{JL82}]\label{lem:JL}
There exists a family of random linear maps $\pi:\: \mathbb{R}^d \to \mathbb{R}^{d'}$ with the following properties: for every $\epsilon, \delta\in (0,1/2)$, there exists $d'=O({ \frac{\log(1/\delta)}{\epsilon^2}  })$ such that for every $d\geq 1$ and all $x\in \mathbb{R}^d$, we have
$\Pr\{\|\pi(x)\| \approx_{1+\epsilon} \|x\|\} \geq 1-\delta.$
\end{lemma}

Based on this lemma, the best known result achieved by random projection is the following: 

\begin{theorem}[\cite{Makarychev19STOC}]\label{thm:epsilon-projection}
Consider any family of random linear maps $\pi:\: \mathbb{R}^d\to\mathbb{R}^{d'}$ that (i) satisfies Lemma~\ref{lem:JL}, and (ii) is sub-Gaussian-tailed (i.e., the probability for the norm after mapping to be larger than the norm before mapping by a factor of at least $1+t$ is bounded by $e^{-\Omega(d' t^2)}$). Then for every $\epsilon,\delta\in (0, 1/4)$, there exists
$d'=O({1\over \epsilon^2}\log{k\over \epsilon\delta}),$
such that $\pi$ is an $\epsilon$-projection with probability at least $1-\delta$.
\end{theorem}

There are many known methods to construct a random linear map that satisfies the conditions (i--ii) in Theorem~\ref{thm:epsilon-projection}, e.g., maps defined by matrices with i.i.d. \emph{Gaussian} and sub-\emph{Gaussian} entries \cite{IM98,Ach03,KM+05}. We will refer to such a random projection as a \emph{JL projection}. 

\emph{Remark:} Compared with PCA-based DR methods, JL projection has the advantage that the projection matrix is \emph{data-oblivious}, and can hence be pre-generated and distributed, or generated independently by different nodes using a shared random number generation seed, both incurring negligible communication cost at runtime. As is shown later, this can lead to significant savings in the communication cost.

\subsection{Cardinality Reduction for $k$-Means}\label{subsec:Cardinality Reduction}

CR methods, also known as \emph{coreset construction algorithms}, aim at constructing a smaller weighted dataset (\emph{coreset}) with a bounded approximation error as follows. 

\begin{definition}[\cite{Feldman13SODA:report}]\label{def:epsilon-coreset}
We say that a tuple $(S, \Delta, w)$, where $S\subset \mathbb{R}^d$, $w:\: S\to \mathbb{R}$, and $\Delta\in \mathbb{R}$, is an $\epsilon$-coreset of $P\subset \mathbb{R}^d$ if it preserves the cost for every set of $k$ centers up to a factor of $1\pm \epsilon$, i.e.,\looseness=-1
\begin{align}
(1-\epsilon)\cost(P,X) \leq \cost(\Sbf,X) \leq (1+\epsilon)\cost(P,X)
\end{align}
for any $X\subset \mathbb{R}^d$ with $|X|=k$, where 
\begin{align}
\cost(\Sbf,X):= \sum_{q\in S}w(q) \cdot \min_{x_i\in X}\|q-x_i\|^2+\Delta
\end{align}
denotes the $k$-means cost for a coreset $\Sbf:= (S, \Delta, w)$ and a set of centers $X$.
\end{definition}

We note that the above definition generalizes most of the existing definitions of $\epsilon$-coreset, which typically ignore $\Delta$.

The best known coreset construction algorithm for $k$-means was given in \cite{Feldman13SODA:report}, which first reduces the intrinsic dimension of the dataset by PCA, and then applies sensitivity sampling to the dimension-reduced dataset to obtain an $\epsilon$-coreset of the original dataset with a size that is constant in $n$ and $d$.

\begin{theorem}[\cite{Feldman13SODA:report}]\label{thm:epsilon-coreset, Delta > 0}
For any $\epsilon, \delta\in (0,1)$, with probability at least $1-\delta$, an $\epsilon$-coreset $(S, \Delta, w)$ of size
$|S| = O\left({k^3\log^2 k\over \epsilon^4}\log({1\over \delta})\right)$
can be computed in time $O(\min(nd^2, n^2 d) + nk \epsilon^{-2} (d+k \log(1/\delta)))$.
\end{theorem}

However, \cite{Feldman13SODA:report} only focused on minimizing the cardinality of coreset, ignoring the cost of transmitting the coreset. As is shown later (Section~\ref{subsec:Summary of Comparison}), its proposed algorithm can be severely suboptimal in the communication cost.

\subsection{Problem Statement}\label{subsec:Problem Statement}

The motivation of most existing DR/CR methods designed for $k$-means is to speed up $k$-means computation in a setting where the node holding the data is also the node computing $k$-means. In contrast, we want to develop efficient $k$-means algorithms in scenarios where the data generation and the $k$-means computation occur at different locations, such as in the case of edge-based learning. We will refer to the node(s) holding the original data as the \emph{data source(s)}, and the node running $k$-means computation as the \emph{server}. 

We will evaluate each considered algorithm by the following performance metrics:\begin{itemize}
\item \emph{Approximation error:} We say that a set of $k$-means centers $X$ is an $\alpha$-approximation ($\alpha>1$) for $k$-means clustering of $P$ if $\cost(P,X)\leq \alpha\cdot \cost(P, X^*)$, where $X^*$ is the optimal set of $k$-means centers for $P$.
\item \emph{Communication cost:} We say that an algorithm incurs a communication cost of $y$ if a data source employing the algorithm needs to send $y$ scalars to the server.
\item \emph{Complexity:} We say that an algorithm incurs a (time) complexity of $z$ at the data source if a data source employing the algorithm needs to perform $z$ elementary operations. \looseness=-1
\end{itemize}

\section{Joint DR and CR at a Single Data Source}\label{sec:Joint Dimension and Cardinality Reduction}

We will first focus on the scenario where all the data are at a single data source (the \emph{centralized setting}). 
We will show that: 
1) using suitably selected DR/CR methods and a sufficiently powerful server, it is possible to solve $k$-means arbitrarily close to the optimal, while incurring a low communication cost and a low complexity at the data source; 
2) the order of applying DR and CR does not affect the approximation error, but affects the complexity and the communication cost; 
3) repeated application of DR/CR can lead to a better communication-computation tradeoff than applying DR and CR only once.  


\subsection{DR+CR}

We first consider the approach of applying DR and then CR. \looseness=-1

\subsubsection{An Existing DR+CR Algorithm} 

The state-of-the-art joint DR and CR algorithm, referred to as \emph{FSS} following the authors' last names, was implicitly presented in Theorem~36 in \cite{Feldman13SODA:report}. FSS first uses PCA to reduce the intrinsic dimension of the dataset and then applies sensitivity sampling. 
Theorem~\ref{thm:epsilon-coreset, Delta > 0} gives the complexity of FSS, but the approximation error and the communication cost incurred when using FSS to generate a data summary for $k$-means were not given in \cite{Feldman13SODA:report}. Thus, we provide them (proved in Appendix~\ref{appendix: proofs})
to facilitate later comparison.

\begin{theorem}\label{coro:Feldmand13SODA}
Suppose that the data source reports the coreset $\Sbf := (S, \Delta, w)$ computed by FSS \cite{Feldman13SODA:report} and the server computes the optimal $k$-means centers $X$ of $\Sbf$\footnote{Given a coreset $\Sbf = (S,\Delta,w)$, $X$ can be computed by ignoring $\Delta$ and applying a weighted $k$-means algorithm to minimize $\sum_{q\in S}w(q) \cdot \min_{x_i\in X}\|q-x_i\|^2$, or by converting $\Sbf$ into an unweighted dataset by duplicating each $q\in S$ for $w(q)$ times (on the average) and applying an unweighted $k$-means algorithm. }. Then:\begin{enumerate}
\item $X$ is a $(1+\epsilon)/(1-\epsilon)$-approximation for $k$-means clustering of $P$ with probability $\geq 1-\delta$;
\item the communication cost is $O\left({kd / \epsilon^2}\right)$,
\end{enumerate}
assuming $\min(n,d)\gg k, 1/\epsilon$, and $1/\delta$.
\end{theorem}

\subsubsection{Communication-efficient DR+CR} \label{sec: Communication-efficient DR+CR}

Now the question is: can we further reduce the communication cost without hurting the approximation error and the complexity?

Our key observation is that the linear communication cost in $d$ for FSS is due to the transmission of a basis of the projected subspace. In contrast, JL projections are \emph{data-oblivious}. Thus, we can circumvent the cost of transmitting the projected subspace by employing a JL projection as the DR method, as the projected subspace can be predetermined. The following is directly implied by the JL Lemma (Lemma~\ref{lem:JL}); see the proof in Appendix~\ref{appendix: proofs}.

\begin{lemma}\label{lem:JL projection}
Let $\pi:\: \mathbb{R}^d \to\mathbb{R}^{d'}$ be a JL projection. Then there exists $d'=O(\epsilon^{-2}\log(nk/\delta))$ such that for any $P\subset \mathbb{R}^d$ with $|P|=n$ and $X, X^*\subset \mathbb{R}^d$ with $|X|=|X^*|=k$, the following holds with probability at least $1-\delta$:
\begin{align}
\cost(P,X) &\approx_{(1+\epsilon)^2} \cost(\pi(P), \pi(X)), \label{eq:JL projection cost(P,X)}\\
\cost(P,X^*) &\approx_{(1+\epsilon)^2} \cost(\pi(P), \pi(X^*)). \label{eq:JL projection cost(P,X*)}
\end{align}
\end{lemma}

Using a JL projection for DR and  FSS for CR, we propose Algorithm~\ref{Alg:JL DR+CR}, where the data source computes and reports a coreset in a low-dimensional space \rev{by first applying JL projection and then applying FSS} (lines~\ref{JL DR+CR: P'}--\ref{JL DR+CR: report}). \rev{Based on the dimension-reduced coreset $(S',\Delta,w)$,} the server solves the $k$-means problem (line~\ref{JL DR+CR: kmeans}) and then converts the centers back to the original space (line~\ref{JL DR+CR: recover X}). \rev{Here, $\mbox{kmeans}(S',w,k)$ denotes a (centralized) $k$-means algorithm that returns the $k$-means centers for the data points in $S'$ with weights $w$, and $\pi_1^{-1}$ denotes an inverse of the JL projection $\pi_1$. We note that the inverse of $\pi_1$ is generally not unique as $\pi_1$ is noninvertible, but our analysis holds for any inverse (e.g., Moore-Penrose inverse).} 
The following theorem quantifies the performance of Algorithm~\ref{Alg:JL DR+CR} (proved in Appendix~\ref{appendix: proofs}). 

\begin{algorithm}[tb]
\label{Alg:JL DR+CR}
\small
\SetKwInOut{Input}{input}\SetKwInOut{Output}{output}
\SetKwFor{Node}{data source:}{}{}
\SetKwFor{Server}{server:}{}{}
\Input{Original dataset $P$, number of centers $k$, JL projection $\pi_1$, FSS-based CR method $\pi_2$}
\Output{Centers for $k$-means clustering of $P$}
\Node{}{
$P'\leftarrow \pi_1(P)$\; \label{JL DR+CR: P'}
$(S',\Delta,w) \leftarrow \pi_2(P')$\; \label{JL DR+CR: S'}
report $(S',\Delta,w)$ to the server\; \label{JL DR+CR: report}
}
\Server{}{
$X'\leftarrow\mbox{kmeans}(S',w,k)$\; \label{JL DR+CR: kmeans} 
$X\leftarrow \pi_1^{-1}(X')$\; \label{JL DR+CR: recover X}
return $X$\; 
}
\caption{ Communication-efficient $k$-Means under DR+CR}
\vspace{-.25em}
\end{algorithm}
\normalsize

\begin{theorem}\label{thm:JL-based DR+CR}
For any $\epsilon, \delta\in (0,1)$, if in Algorithm~\ref{Alg:JL DR+CR}, $\pi_1$ satisfies Lemma~\ref{lem:JL projection}, $\pi_2$ generates an $\epsilon$-coreset with probability at least $1-\delta$, and kmeans$(S',w,k)$ returns the optimal $k$-means centers of the dataset $S'$ with weights $w$, then\looseness=-1 \begin{enumerate}
    \item the output $X$ is a $(1+\epsilon)^5/(1-\epsilon)$-approximation for $k$-means clustering of $P$ with probability at least $(1-\delta)^2$,
\item the communication cost is $O\left({k \epsilon^{-4} \log{n}}\right)$, and
\item the complexity at the data source is $\tilde{O}\left(nd\epsilon^{-2}\right)$,
\end{enumerate}
assuming $\min(n,d)\gg k, 1/\epsilon$, and $1/\delta$.
\end{theorem}

\emph{Remark:} We only focus on the complexity at the data source as the server is usually much more powerful.  Theorem~\ref{thm:JL-based DR+CR} shows that Algorithm~\ref{Alg:JL DR+CR} can solve $k$-means arbitrarily close to the optimal with an arbitrarily high probability, while incurring a complexity at the data source that is roughly linear in the data size (i.e., $nd$) and a communication cost that is roughly logarithmic in the data cardinality $n$.\looseness=-1

\subsection{CR+DR} \label{sec: CR+DR}

While Algorithm~\ref{Alg:JL DR+CR} can reduce the communication cost without incurring much computation at the data source, it remains unclear whether its order of applying DR and CR is optimal. To this end, we consider applying CR first.\looseness=-1

We again choose JL projection as the DR method to avoid transmitting the projection matrix at runtime, and choose FSS as the CR method as it generates an $\epsilon$-coreset with the minimum cardinality among the existing CR methods for $k$-means. The algorithm, shown in Algorithm~\ref{Alg:FSS CR+JL DR}, differs from Algorithm~\ref{Alg:JL DR+CR} in that the order of applying DR and CR is reversed. 
\rev{That is, the data source first applies FSS (line~\ref{FSS CR+DR:S}) and then applies JL projection (line~\ref{FSS CR+DR:S'}) to compute a dimension-reduced coreset $(S',\Delta,w)$, based on which the server computes a set of $k$-means centers $X$ in the same way as Algorithm~\ref{Alg:JL DR+CR}.}

\begin{algorithm}[tb]
\label{Alg:FSS CR+JL DR}
\small
\SetKwInOut{Input}{input}\SetKwInOut{Output}{output}
\SetKwFor{Node}{data source:}{}{}
\SetKwFor{Server}{server:}{}{}
\Input{Original dataset $P$, number of centers $k$, JL projection $\pi_1$, FSS-based CR method $\pi_2$}
\Output{Centers for $k$-means clustering of $P$}
\Node{}{
$(S,\Delta,w)\leftarrow \pi_2(P)$\; \label{FSS CR+DR:S}
$S' \leftarrow \pi_1(S)$\; \label{FSS CR+DR:S'}
report $(S',\Delta,w)$ to the server\;
}
\Server{}{
$X'\leftarrow\mbox{kmeans}(S',w,k)$\; 
$X\leftarrow \pi_1^{-1}(X')$\;
return $X$\; 
}
\caption{ Communication-efficient $k$-Means under CR+DR}
\vspace{-.25em}
\end{algorithm}
\normalsize

We now analyze the performance of Algorithm~\ref{Alg:FSS CR+JL DR}, starting with a counterpart of Lemma~\ref{lem:JL projection} (proved in Appendix~\ref{appendix: proofs}). \looseness=-1

\begin{lemma}\label{lem:JL projection for S}
Let $\pi:\: \mathbb{R}^d \to \mathbb{R}^{d'}$ be a JL projection. Then there exists $d' = O(\epsilon^{-2}\log(n'k/\delta))$ such that for any coreset $\Sbf:=(S, \Delta, w)$, where $S\subset \mathbb{R}^d$ with $|S|=n'$, $w: S\to \mathbb{R}$, and $\Delta\in \mathbb{R}$, and any $X, X^*\subset \mathbb{R}^d$ with $|X|=|X^*|=k$, the following holds with probability at least $1-\delta$:
\begin{align}
\cost(\Sbf, X) &\approx_{(1+\epsilon)^2}\cost((\pi(S),\Delta,w), \pi(X)), \label{eq:JL for S, cost(S,X)}\\
\cost(\Sbf, X^*) &\approx_{(1+\epsilon)^2} \cost((\pi(S),\Delta,w), \pi(X^*)). \label{eq:JL for S, cost(S,X*)}
\end{align}
\end{lemma}

Below, we will show that Algorithm~\ref{Alg:FSS CR+JL DR} achieves the same approximation error as Algorithm~\ref{Alg:JL DR+CR}, but at different cost and complexity (proved in Appendix~\ref{appendix: proofs}). 

\begin{theorem}\label{thm:FSS CR+JL DR}
For any $\epsilon, \delta\in (0,1)$, if in Algorithm~\ref{Alg:FSS CR+JL DR}, $\pi_1$ satisfies Lemma~\ref{lem:JL projection for S}, $\pi_2$ generates an $\epsilon$-coreset with probability at least $1-\delta$, and kmeans$(S',w,k)$ returns the optimal $k$-means centers of the dataset $S'$ with weights $w$, then \looseness=-1
\begin{enumerate}
    \item the output $X$ is an ${(1+\epsilon)^5/ (1-\epsilon)}$-approximation for $k$-means clustering of $P$ with probability  $\geq (1-\delta)^2$, 
    \item the communication cost is $\tilde{O}(k^3/\epsilon^6)$, and
    \item the complexity at the data source is $O\left(nd \cdot \min(n,d)\right)$,
\end{enumerate}
assuming $\min(n,d) \gg k, 1/\epsilon$, and $1/\delta$.
\end{theorem}

\subsection{Repeated DR/CR}\label{subsec:Repeated DR or CR}




Theorems~\ref{thm:JL-based DR+CR} and \ref{thm:FSS CR+JL DR} state that to achieve the same approximation error with the same probability, DR+CR (Algorithm~\ref{Alg:JL DR+CR}) incurs a communication cost of ${O}(k\epsilon^{-4}\log{n})$ and a complexity of $\tilde{O}(\epsilon^{-2}nd)$, while CR+DR (Algorithm~\ref{Alg:FSS CR+JL DR}) incurs a communication cost of $\tilde{O}(k^3\epsilon^{-6})$ and a complexity of $O(nd \cdot \min(n,d))$.  This shows a \emph{communication-computation tradeoff}: the approach of DR+CR incurs a linear complexity and a logarithmic communication cost, whereas the approach of CR+DR incurs a super-linear complexity (which is still less than quadratic) and a constant communication cost. 
One may wonder whether it is possible to combine the strengths of both of the algorithms. Below we give an affirmative answer by applying some of these repeatedly. \looseness=-1

We know from Theorem~\ref{thm:epsilon-coreset, Delta > 0} that applying FSS once already reduces the cardinality to a constant (in $n$ and $d$), and hence there is no need to repeat FSS. The same theorem also implies that if we apply FSS first, we will incur a super-linear complexity, and hence we need to apply JL projection before FSS. 
Meanwhile, we see from Lemmas~\ref{lem:JL projection} and \ref{lem:JL projection for S} that applying JL projection on a dataset of cardinality $n'$ can reduce its dimension to $O(\epsilon^{-2}\log(n'k/\delta))$ while achieving a $(1+O(\epsilon))$-approximation with high probability. Thus, we can further reduce the dimension by applying JL projection again after reducing the cardinality by FSS. The above reasoning suggests a three-step procedure: JL$\to$FSS$\to$JL, presented in Algorithm~\ref{Alg:JL+FSS+JL}. \rev{The data source applies JL projection both before and after FSS (lines~\ref{JL+FSS+JL:first JL} and \ref{JL+FSS+JL:second JL}), 
where $\pi^{(1)}_1$ projects from $\mathbb{R}^d$ to $\mathbb{R}^{O(\log{n}/\epsilon^2)}$, and $\pi^{(2)}_1$ projects from $\mathbb{R}^{O(\log{n}/\epsilon^2)}$ to $\mathbb{R}^{O(\log{|S|}/\epsilon^2)}$. The server first computes the $k$-means centers in the twice-protected space, and then converts them back to the original space (line~\ref{JL+FSS+JL:circ}).} Note that by convention, $\pi^{(2)}_1\circ \pi^{(1)}_1(X)$ 
means $\pi^{(2)}_1(\pi^{(1)}_1(X))$.

\begin{algorithm}[tb]
\label{Alg:JL+FSS+JL}
\small
\SetKwInOut{Input}{input}\SetKwInOut{Output}{output}
\SetKwFor{Node}{data source:}{}{}
\SetKwFor{Server}{server:}{}{}
\Input{Original dataset $P$, number of centers $k$, JL projection $\pi^{(1)}_1$ for $P$, FSS-based CR method $\pi_2$, JL projection $\pi^{(2)}_1$ for the output of $\pi_2$}
\Output{Centers for $k$-means clustering of $P$}
\Node{}{
$P'\leftarrow \pi^{(1)}_1(P)$\; \label{JL+FSS+JL:first JL}
$(S,\Delta,w)\leftarrow \pi_2(P')$\; \label{JL+FSS+JL:FSS}
$S' \leftarrow \pi^{(2)}_1(S)$\; \label{JL+FSS+JL:second JL}
report $(S',\Delta,w)$ to the server\;
}
\Server{}{
$X'\leftarrow\mbox{kmeans}(S',w,k)$\; 
$X\leftarrow (\pi^{(2)}_1 \circ \pi^{(1)}_1)^{-1}(X')$\; \label{JL+FSS+JL:circ}
return $X$\;
}
\caption{Communication-efficient $k$-Means under DR+CR+DR}
\vspace{-.25em}
\end{algorithm}
\normalsize

Below, we will show that this seemingly small change is able to combine the low communication cost of Algorithm~\ref{Alg:FSS CR+JL DR} and the low complexity of Algorithm~\ref{Alg:JL DR+CR}, at a small increase in the approximation error; see Appendix~\ref{appendix: proofs} for the proof. 

\begin{theorem}\label{thm:JL+FSS+JL}
For any $\epsilon, \delta\in (0,1)$, if in Algorithm~\ref{Alg:JL+FSS+JL}, $\pi^{(1)}_1$ satisfies Lemma~\ref{lem:JL projection}, $\pi^{(2)}_1$ satisfies Lemma~\ref{lem:JL projection for S}, $\pi_2$ generates an $\epsilon$-coreset of its input dataset with probability at least $1-\delta$, and kmeans$(S',w,k)$ returns the optimal $k$-means centers of the dataset $S'$ with weights $w$, then \begin{enumerate}
    \item the output $X$ is a $(1+\epsilon)^9/(1-\epsilon)$-approximation for $k$-means clustering of $P$ with probability $\geq (1-\delta)^3$,
    \item the communication cost is $\tilde{O}(k^3/\epsilon^6)$, and
    \item the complexity at the data source is $\tilde{O}(nd/\epsilon^2)$,
\end{enumerate}
assuming $\min(n,d)\gg k, 1/\epsilon,$ and $1/\delta$.
\end{theorem}

\emph{Remark:} Theorem~\ref{thm:JL+FSS+JL} implies that Algorithm~\ref{Alg:JL+FSS+JL} is essentially ``optimal'' in the sense that it achieves a $(1+O(\epsilon))$-approximation with an arbitrarily high probability, at a near-linear complexity and a constant communication cost at the data source. Thus, no qualitative improvement will be achieved by applying further DR/CR methods. 


\section{Joint DR and CR across Multiple Data Sources}\label{sec:Multiple Data Sources}

Consider the scenario where the dataset $P$ is split across $m$ data sources ($m  \geq 2$). Let $P_i$ denote the dataset at data source $i$ and $n_i$ be its cardinality.  
As shown below, the previous algorithms 
can be adapted to the distributed setting. \looseness=-1

\subsection{Distributed Version of FSS}

It turns out that the state-of-the-art distributed DR and CR algorithm, proposed in \cite[Algorithm~1]{Balcan14NIPS}, is exactly a distributed version of FSS, referred to as \emph{BKLW} following the authors' last names. As in FSS, BKLW first uses PCA to reduce the intrinsic dimension of the dataset and then applies sensitivity sampling. However, it uses  distributed algorithms to perform these steps.\looseness=-1 

For distributed PCA, BKLW applies an algorithm \emph{disPCA} from \cite{Feldman13SODA:report} (formalized in Algorithm~1 in \cite{Balcan14NIPS:report}), where: 
\begin{enumerate}
    \item each data source $i$ ($i=1,\ldots,m$) computes local SVD $A_{P_i} = U_i \Sigma_i V_i^T$, and sends $\Sigma_i^{(t_1)}$ and $V_i^{(t_1)}$ to the server ($\Sigma_i^{(t_1)}$ and $V_i^{(t_1)}$ contain the first $t_1$ columns of $\Sigma_i$ and $V_i$, respectively);
    \item the server constructs $Y^T = [Y_1^T,\ldots,Y_m^T]$, with $Y_i = \Sigma_i^{(t_1)}(V_i^{(t_1)})^T$, computes a global SVD $Y = U \Sigma V^T$;
    \item the first $t_2$ columns of $V$ are returned as an approximate solution to the PCA of $\bigcup_{i=1}^m P_i$.
\end{enumerate}

For distributed sensitivity sampling, BKLW applies an algorithm \emph{disSS} from \cite{Balcan13NIPS} (Algorithm~1), where:  
\begin{enumerate}
\item each data source $i$ ($i=1,\ldots,m$) computes a bicriteria approximation $X_i$ for $P_i$ and reports $\cost(P_i, X_i)$;
\item the server allocates a global sample size $s$ to each data source proportionally to its cost, i.e., $s_i= {s \cdot \cost(P_i, X_i)/\big( \sum_{j=1}^m \cost(P_j, X_j)\big)}$;
\item each data source $i$ draws $s_i$ i.i.d. samples $S_i$ from $P_i$ with probability proportional to $\cost(\{p\}, X_i)$, and reports $S_i\cup X_i$ with their weights $w:\: S_i\cup X_i\to \mathbb{R}$, that are set to match the number of points per cluster;
\item the union of the reported sets $(\bigcup_{i=1}^m (S_i\cup X_i), 0, w)$ is returned as a coreset of $\bigcup_{i=1}^m P_i$.
\end{enumerate}

BKLW first applies disPCA, followed by disSS with $s = O(\epsilon^{-4}(k^2/\epsilon^2 + \log(1/\delta)) + mk\log(mk/\delta))$ to the dimension-reduced dataset $\{A_{P_i}V^{(t_2)}(V^{(t_2)})^T\}_{i=1}^m$ to compute a coreset $(S,0,w)$ at the server. Finally, the server computes the optimal $k$-means centers $X$ on $(S,0,w)$ and returns it as an approximation to the optimal $k$-means centers of $\bigcup_{i=1}^m P_i$.

Although a theorem was given in \cite{Balcan14NIPS} without proof on the performance of BKLW, the result is imprecise and incomplete. Here, we provide the complete analysis to facilitate later comparison. We will leverage the following results.\looseness=-1

\begin{theorem}[\cite{Balcan14NIPS:report}]\label{thm:disPCA}
For any $\epsilon\in (0, 1/3)$, let $t_1=t_2\geq k+\lceil 4k/\epsilon^2\rceil -1$ in disPCA and $\tilde{P}_i$ be the projected dataset at data source $i$ (i.e., the set of rows of $A_{P_i}V^{(t_2)}(V^{(t_2)})^T$). Then there exists a constant $\Delta\geq 0$ such that for any set $X\subset \mathbb{R}^d$ with $|X|=k$, 
\begin{align}
\hspace{-.5em}    (1\hspace{-.25em}-\hspace{-.25em}\epsilon)\cost(P, X) \hspace{-.05em} \leq\hspace{-.05em}  \cost(\tilde{P}, X)\hspace{-.25em} +\hspace{-.25em} \Delta \hspace{-.05em}\leq\hspace{-.05em} (1\hspace{-.25em}+\hspace{-.25em}\epsilon) \cost(P, X), \label{eq:disPCA}
\end{align}
where $P:= \bigcup_{i=1}^m P_i$ and $\tilde{P}:= \bigcup_{i=1}^m \tilde{P}_i$. 
\end{theorem}

\begin{theorem}[\cite{Balcan13NIPS}]\label{thm:disSS}
For a distributed dataset $\{P_i\}_{i=1}^m$ with $P_i\subset \mathbb{R}^d$ and any $\epsilon, \delta\in (0,1)$, with probability at least $1-\delta$, the output $(S, 0, w)$ of disSS is an $\epsilon$-coreset of $\bigcup_{i=1}^m P_i$ of size
\begin{align}
|S| = O\left({1\over \epsilon^4}\Big(kd + \log({1\over \delta}) \Big) + mk \log({mk\over \delta}) \right).
\end{align}
\end{theorem}

Theorems~\ref{thm:disPCA} and \ref{thm:disSS} bound the performance of disPCA and disSS, respectively, based on which we have the following results for BKLW (see proof in Appendix~\ref{appendix: proofs}).

\begin{theorem}\label{thm:BKLW}
For any $\epsilon\in (0, 1/3)$ and $\delta\in (0, 1)$, suppose that in BKLW, disPCA satisfies Theorem~\ref{thm:disPCA} for the input dataset $\{P_i\}_{i=1}^m$ and disSS satisfies Theorem~\ref{thm:disSS} for the input dataset $\{\tilde{P}_i\}_{i=1}^m$. Then \begin{enumerate}
    \item the output $X$ is a $(1+\epsilon)^2/(1-\epsilon)^2$-approximation for $k$-means clustering of $\bigcup_{i=1}^m P_i$ with probability  $\geq 1-\delta$, 
    \item the total communication cost over all the data sources is $O(mkd/\epsilon^2)$, and
    \item the complexity at each data source is $O(nd \cdot \min(n,d))$,
\end{enumerate}
assuming $\min(n,d)\gg m, k, 1/\epsilon,$ and $1/\delta$.
\end{theorem}

\subsection{Enhancements}\label{subsec:Enhancement - JL+BKLW}

It is easy to see that each data source can apply JL projection independently at no additional communication cost. 
Following the ideas in Algorithms~\ref{Alg:JL DR+CR} and \ref{Alg:FSS CR+JL DR}, we wonder: (i) Can we improve BKLW by combining it with JL projection? (ii) Is there an optimal order of applying BKLW and JL projection? \looseness=0

We first consider applying JL projection before invoking BKLW. For consistency with Algorithm~\ref{Alg:JL DR+CR}, we only use the first two steps of BKLW, i.e., disPCA and disSS, that construct a coreset, which we refer to as a \emph{BKLW-based CR method}. The algorithm, shown in Algorithm~\ref{Alg:distributed kmeans}, is essentially the distributed counterpart of Algorithm~\ref{Alg:JL DR+CR}. \rev{First, each source independently applies JL projection to its local dataset (line~\ref{distributed:JL}). Then, the sources cooperatively run BKLW, i.e., disPCA + disSS (line~\ref{distributed: pi_2}). Finally, the server uses the received dimension-reduced coreset to solve $k$-means and converts the centers back to the original space (lines~\ref{distributed:kmeans}--\ref{distributed:inverse JL}). } 

\begin{algorithm}[tb]
\label{Alg:distributed kmeans}
\small
\SetKwInOut{Input}{input}\SetKwInOut{Output}{output}
\SetKwFor{Node}{each data source $i$ ($i=1,\ldots,m$):}{}{}
\SetKwFor{Server}{server:}{}{}
\Input{Distributed dataset $\{P_i\}_{i=1}^m$, number of centers $k$, JL projection $\pi_1$, BKLW-based CR method $\pi_2$ }
\Output{Centers for $k$-means clustering of $P$}
\Node{}{
$P_i'\leftarrow \pi_1(P_i)$\; \label{distributed:JL}
}
run $\pi_2$ on the distributed dataset $\{P_i'\}_{i=1}^m$, which results in each data source $i$ reporting a local coreset $(S_i', 0, w)$ to the server\; \label{distributed: pi_2}
\Server{}{
$X'\leftarrow\mbox{kmeans}(\bigcup_{i=1}^m S'_i,w,k)$\; \label{distributed:kmeans}
$X\leftarrow \pi_1^{-1}(X')$\; \label{distributed:inverse JL}
return $X$\;
}
\caption{Communication-efficient Distributed $k$-Means under DR+CR}
\vspace{-.25em}
\end{algorithm}
\normalsize

We now analyze the performance of Algorithm~\ref{Alg:distributed kmeans}, starting from a coreset-like property of the BKLW-based CR method $\pi_2$ (see proof in Appendix~\ref{appendix: proofs}). 

\begin{lemma}\label{lem:BKLW-based CR method}
Let $P:=\bigcup_{i=1}^m P_i$ be the union of the input datasets for the BKLW-based CR method $\pi_2$ and $\Sbf := (S, 0, w)$ be the resulting coreset reported to the server. For any $\epsilon\in (0, 1/3)$ and $\delta\in (0, 1)$, $\exists t_1=t_2=O(k/\epsilon^2)$, $s = O(\epsilon^{-4}(k^2/\epsilon^2 + \log(1/\delta)) + mk\log(mk/\delta))$, and $\Delta\geq 0$, such that with probability at least $1-\delta$, $\pi_2$ with parameters $t_1$, $t_2$, and $s$ satisfies
\begin{align}\label{eq:BKLW corollary}
\hspace{-.7em} (1\hspace{-.25em}-\textbf{}\epsilon)^2 \cost(P,X) \hspace{-.1em} \leq \hspace{-.1em}\cost(\Sbf,X) \hspace{-.25em}+\hspace{-.25em} \Delta \hspace{-.1em}\leq\hspace{-.1em} (1\hspace{-.25em}+\hspace{-.25em}\epsilon)^2 \cost(P,X)
\end{align}
for any set $X$ of $k$ points in the same space as $P$. 
\end{lemma}

\emph{Remark:} Comparing Lemma~\ref{lem:BKLW-based CR method} with Definition~\ref{def:epsilon-coreset}, we see that $\pi_2$ does not construct an $O(\epsilon)$-coreset of its input dataset. Nevertheless, its output can approximate the $k$-means cost of the input dataset up to a constant shift, which is sufficient for computing  approximate $k$-means. 

We have the following performance guarantee for Algorithm~\ref{Alg:distributed kmeans} (see proof in Appendix~\ref{appendix: proofs}). 

\begin{theorem}\label{thm:distributed kmeans}
For any $\epsilon\in (0, 1/3)$ and $\delta\in (0, 1)$, suppose that in Algorithm~\ref{Alg:distributed kmeans}, $\pi_1$ satisfies Lemma~\ref{lem:JL projection}, $\pi_2$ satisfies Lemma~\ref{lem:BKLW-based CR method}, and kmeans$(\bigcup_{i=1}^m S'_i,w,k)$ returns the optimal $k$-means centers of the dataset $\bigcup_{i=1}^m S'_i$ with weights $w$. Then 
\begin{enumerate}
    \item the output $X$ is a $(1+\epsilon)^6/(1-\epsilon)^2$-approximation for $k$-means clustering of $\bigcup_{i=1}^m P_i$ with probability $\geq (1-\delta)^2$,
    \item the total communication cost over all the data sources is $O(mk \epsilon^{-4}\log{n})$, and
    \item the complexity at each data source is $\tilde{O}(nd \epsilon^{-4})$,
\end{enumerate}
assuming $\min(n,d)\gg m, k, 1/\epsilon,$ and $1/\delta$. 
\end{theorem}

\begin{table*}[tb]
\vspace{-.5em}
\footnotesize
\renewcommand{\arraystretch}{1.3}
\caption{Summary of Comparison} \label{tab:comparison}
\vspace{-.5em}
\centering
\begin{tabular}{|l|c|c|}
  \hline
  Algorithm & \begin{tabular}{c} Communication cost\end{tabular} & \begin{tabular}{c} Computational complexity \end{tabular} \\
  \hline
  FSS~\cite{Feldman13SODA:report}  &  ${O}(k d/\epsilon_2^2)$ & $O(nd \cdot \min(n,d))$ \\
  \hline
  JL + FSS (Alg.~\ref{Alg:JL DR+CR}) &  ${O}(k \log{n}/\epsilon_1^{4})$ & $\tilde{O}(nd/\epsilon_1^{2})$ \\
  \hline
  FSS + JL (Alg.~\ref{Alg:FSS CR+JL DR}) & $\tilde{O}(k^3/\epsilon_1^{6})$ & $O(nd \cdot \min(n,d))$ \\
  \hline
  JL + FSS + JL (Alg.~\ref{Alg:JL+FSS+JL}) & $\tilde{O}(k^3/\epsilon_3^6)$ & $\tilde{O}(nd/\epsilon_3^2)$\\
  \hline 
  \hline
  BKLW \cite{Balcan14NIPS} & $O(mkd/\epsilon_4^2)$ & $O(nd \cdot \min(n,d))$ \\
  \hline 
  JL + BKLW (Alg.~\ref{Alg:distributed kmeans}) & $O(mk\log{n}/\epsilon_5^4)$ & $\tilde{O}(nd/\epsilon_5^4)$ \\
   \hline
\end{tabular}
\vspace{-.05em}
\end{table*}

\emph{Discussion:}
Comparing Theorems~\ref{thm:distributed kmeans} and \ref{thm:BKLW}, we see that for $d\gg \log{n}$ (e.g., $d = \Omega(n)$), Algorithm~\ref{Alg:distributed kmeans} can significantly reduce the communication cost and the complexity at data sources, while achieving a similar $(1+O(\epsilon))$-approximation as BKLW. Note that although the possibility of applying another DR method before BKLW was mentioned in \cite{Balcan14NIPS}, no result was given there.  

Meanwhile, although one could develop a distributed counterpart of Algorithm~\ref{Alg:FSS CR+JL DR} that applies JL projection after BKLW, its performance will not be competitive. Specifically, using similar analysis, this approach incurs the same order of communication cost and complexity as BKLW. 
Meanwhile, the JL projection introduces additional error, causing its overall approximation error to be larger. It is thus unnecessary to consider this algorithm.  


Furthermore, we note that repeated application of DR/CR is unnecessary in the distributed setting. This is because after one round of BKLW (with or without applying JL projection beforehand), we already reduce the cardinality to $O(\epsilon^{-4}(k^2/\epsilon^2 + \log(1/\delta)) + mk\log(mk/\delta))$ and the dimension to $O(k/\epsilon^2)$, both constant in the size $(n,d)$ of the original dataset. Meanwhile, this round incurs a communication cost that scales with $(n,d)$ as $O(\log{n})$ (with JL projection) or $O(d)$ (without JL projection), and a complexity that scales as $\tilde{O}(nd)$ (with JL projection) or $O(nd \cdot \min(n,d))$ (without JL projection). Therefore, any possible reduction in the cost (or the complexity) achieved by further reducing the cardinality or dimension will be dominated by the cost (or the complexity) in the first round. Hence, repeated application of DR/CR will not improve the order of the communication cost or the complexity.

\subsection{Summary of Comparison}\label{subsec:Summary of Comparison}

We are now ready to compare the performances of all the proposed algorithms and their best existing counterparts in both centralized (i.e., single data source) and distributed (i.e., multiple data sources) settings. 

To ensure the same approximation error for all the algorithms, we set the error parameter `$\epsilon$' to $\epsilon_1$ for Algorithms~\ref{Alg:JL DR+CR} and \ref{Alg:FSS CR+JL DR}, $\epsilon_2$ for FSS, $\epsilon_3$ for Algorithm~\ref{Alg:JL+FSS+JL}, $\epsilon_4$ for BKLW, and $\epsilon_5$ for Algorithm~\ref{Alg:distributed kmeans}, where for any $\epsilon\in (0,1)$, $\epsilon_1$ satisfies  $(1+\epsilon_1)^5/(1-\epsilon_1)=1+\epsilon$, $\epsilon_2$ satisfies $(1+\epsilon_2)/(1-\epsilon_2) = 1+\epsilon$, $\epsilon_3$ satisfies $(1+\epsilon_3)^9/(1-\epsilon_3) = 1+\epsilon$, $\epsilon_4$ satisfies $(1+\epsilon_4)^2/(1-\epsilon_4)^2 = 1+\epsilon$, and $\epsilon_5$ satisfies $(1+\epsilon_5)^6/(1-\epsilon_5)^2 = 1+\epsilon$. 

The comparison, summarized in Table~\ref{tab:comparison}, is in terms of the communication cost and the complexity at the data source(s) for achieving a $(1+\epsilon)$-approximation for $k$-means clustering of an input dataset of cardinality $n$ and dimension $d$, where the first four rows are for the centralized setting and the last two rows are for the distributed setting. Our focus is on the scaling with $n$ and $d$, which are assumed to dominate the other parameters (i.e., $k$, $m$, $1/\epsilon_i$). Clearly, for high-dimensional datasets satisfying $d\gg \log{n}$, the best proposed algorithms (Algorithm~\ref{Alg:JL+FSS+JL} and Algorithm~\ref{Alg:distributed kmeans}) significantly outperforms the best existing algorithms (FSS and BKLW) in both centralized and distributed settings.

\section{Extension to Joint DR, CR, and QT}\label{sec:Extension}

Besides cardinality and dimensionality, the volume of a dataset also depends on its \emph{precision}, defined as the number of bits used to represent each attribute in the dataset. While DR and CR methods can be used to reduce the dimensionality and the cardinality, quantization techniques \cite{Sayood:2012:IDC:3050831} can be used to reduce the precision and hence further reduce the communication cost. While the optimal efficient quantization in support of $k$-means is worth a separate study, our focus here is on properly combining a given quantizer with the proposed DR/CR methods. To this end, we will use a simple rounding-based quantizer as a concrete example.\looseness=-1


\subsection{Rounding-based Quantization}
\label{subsec: Incorporating quantization}

Given a scalar $x \in \mathbb{R}$, we denote the $b_0$-bit binary floating number representation of $x$ by 
\begin{align}
    &\hspace{-.75em} x = (-1)^{\sign(x)} \times 2^{e_x} \times \nonumber\\ 
    &\hspace{-.75em} (a(0) \hspace{-.15em}+\hspace{-.15em} a(1) \hspace{-.15em}\times\hspace{-.15em} 2^{-1} \hspace{-.15em}+\hspace{-.15em} \ldots  \hspace{-.15em}+\hspace{-.15em} a(b_0\hspace{-.15em}-\hspace{-.15em}1\hspace{-.15em}-\hspace{-.15em}m_e) \hspace{-.15em}\times\hspace{-.15em} 2^{-(b_0-1-m_e)}),
\end{align} 
where $\sign(x)=0$ if $x\geq 0$ and $\sign(x)=1$ if $x<0$, $m_e$ is the number of exponent bits, $e_x$ is the $m_e$-bit exponent of $x$, and $a(\cdot) \in \{0, 1\}$ are the significant bits ($a(0)\equiv 1$). The rounding-based quantizer $\Gamma$ with $s$ significant bits is \looseness=0
\begin{align}
&\hspace{-1em}\Gamma(x) \hspace{-.05em}:=\hspace{-.05em} (-1)^{\sign(x)} \times 2^{e_{x}} \times \nonumber\\ 
&\hspace{-1em}(a(0) + a(1) \times 2^{-1} + \ldots + a(s)\times 2^{-s} + a'(s) \times 2^{-s}),
\end{align}
where $a'(s)$ is the result of rounding the remaining bits (0: rounding down; 1: rounding up).

For simplicity of notation, we also use $p':=\Gamma(p)$ to denote the element-wise rounding-based quantization of a data point $p=(p_i)_{i=1}^d \in \mathbb{R}^d$. Defining the maximum quantization error as $\Delta_{QT} := \max_{p \in P} \|p-p'\|$, we know that by the definition of rounding-based quantizer, the quantization error in each element satisfies $|p_i-p'_i| \leq 2^{e_{p_i} - s} \leq |p_i| 2^{-s}$ since $|p_i| \geq 2^{e_{p_i}}$. Therefore, the maximum quantization error is bounded as
\begin{align}\label{eq:Delta_QT}
    \Delta_{QT} &= \max_{p\in P} \sqrt{\sum_{i=1}^d(p_i-p'_i)^2} \nonumber\\
    &\leq \max_{p\in P} \sqrt{\sum_{i=1}^d 2^{-2s}p_i^2} = 2^{-s} \max_{p \in P} \|p\|.
\end{align}

\subsection{Approximation Error Analysis}\label{subsec:Approximation Error Analysis}

We now analyze the performance after adding quantization to the proposed communication-efficient $k$-means algorithms. As DR and CR can generate data points of arbitrary values that may not be representable with a given number of significant bits, we add quantization after all the DR/CR steps. That is, we assume that right before a data source reports its dimension-reduced coreset $(S,\Delta,w)$ to the server, it will apply the rounding-based quantizer $\Gamma$ and report $(S_{QT},\Delta,w)$ instead\footnote{Here we only apply quantization to the coreset points in $S$ as their transfer dominates the communication cost, but our approach can be extended to other cases. }, where $S_{QT}:=\{\Gamma(p):\: p\in S\}$. 
Obviously, the quantization further reduces the communication cost. \rev{It also incurs a computational complexity that is linear in the size of $S$, which is sub-linear in the size of the original dataset (i.e., $nd$) and thus subsumed by the complexity of DR/CR (as shown in Table~\ref{tab:comparison}).} 
The only performance metric it can negatively impact is the approximation error, which is analyzed in the following theorem (see proof in Appendix~\ref{appendix: proofs}).

\begin{theorem} \label{thm: DR, CR, QT}
Let $X$ denote the optimal $k$-means centers computed by the server based on the received coreset, $X^*$ denote the optimal $k$-means centers based on the original dataset, and $\Delta_D$ denote the diameter of the input space. 
\begin{enumerate}
    \item In Algorithm \ref{Alg:JL DR+CR}, suppose that $\pi_1$ satisfies Lemma~\ref{lem:JL projection} with $\epsilon_1$, $\pi_2$ generates an $\epsilon_2$-coreset with probability $\geq 1-\delta$, and $\pi_{QT}$ is a quantizer with maximum error $\Delta_{QT}$. If we update Line 4 to: $S'_{QT} \leftarrow \pi_{QT}(S')$ and report $(S'_{QT},\Delta,w)$ to the server, then the approximation error will be
    \begin{align}
        \cost(P,X) \leq 
        &\frac{(1+\epsilon_1)^4(1+\epsilon_2)}{(1-\epsilon_2)} \cost(P, X^*) \nonumber \\
        &+ \frac{(1+\epsilon_1)^2}{(1-\epsilon_2)} 4n\Delta_D\Delta_{QT} 
    \end{align}
    with probability at least $(1-\delta)^2$.
    \item In Algorithm \ref{Alg:FSS CR+JL DR}, suppose that $\pi_1$ satisfies Lemma~\ref{lem:JL projection for S} with $\epsilon_1$, $\pi_2$ generates an $\epsilon_2$-coreset with probability $\geq 1-\delta$, and $\pi_{QT}$ is a quantizer with maximum error $\Delta_{QT}$. If we update Line 4 to: $S'_{QT} \leftarrow \pi_{QT}(S')$ and report $(S'_{QT},\Delta,w)$ to the server, then the approximation error will be 
    \begin{align}
        \cost(P,X) \leq &\frac{(1+\epsilon_1)^4(1+\epsilon_2)}{(1-\epsilon_2)} \cost(P, X^*) \nonumber \\
        &+ \frac{(1+\epsilon_1)^2}{(1-\epsilon_2)} 4n\Delta_D\Delta_{QT} 
    \end{align}
    with probability at least $(1-\delta)^2$. 
    \item In Algorithm \ref{Alg:JL+FSS+JL}, suppose that $\pi_1^{(1)}$ satisfies Lemma~\ref{lem:JL projection} with $\epsilon_1^{(1)}$, $\pi_1^{(2)}$ satisfies Lemma~\ref{lem:JL projection for S} with $\epsilon_1^{(2)}$, $\pi_2$ generates an $\epsilon_2$-coreset with probability $\geq 1-\delta$, and $\pi_{QT}$ is a quantizer with maximum error $\Delta_{QT}$. If we update Line~5 to: $S'_{QT} \leftarrow \pi_{QT}(S')$ and report $(S'_{QT},\Delta,w)$ to the server, then the approximation error will be 
    \begin{align}\label{eq:JL+FSS+JL+QT}
        \cost(P,X) \leq 
        & \frac{(1+\epsilon_1^{(1)})^4(1+\epsilon_2)(1+\epsilon_1^{(2)})^4}{(1-\epsilon_2)} \cost(P, X^*) \nonumber \\
        &\hspace{-1.5em} + \frac{(1+\epsilon_1^{(1)})^2(1+\epsilon_1^{(2)})^2}{(1-\epsilon_2)} 4n\Delta_D\Delta_{QT} 
    \end{align}
    with probability at least $(1-\delta)^3$.
    \item In Algorithm \ref{Alg:distributed kmeans}, suppose that $\pi_1$ satisfies Lemma~\ref{lem:JL projection}, $\pi_2$ satisfies Lemma~\ref{lem:BKLW-based CR method}, and $\pi_{QT}$ is a quantizer with maximum error $\Delta_{QT}$. If we update Line 3 to: run $\pi_2$ on the distributed dataset $\{P_i'\}_{i=1}^m$ to compute a local coreset $(S_i', 0, w)$ at each data source $i$ and  report  $(\pi_{QT}(S_{i}'), 0, w)$ to the server,  then the approximation error will be
    \begin{align}
        \cost(P,X) \leq  &\frac{(1+\epsilon_1)^4(1+\epsilon_2)^2}{(1-\epsilon_2)^2} \cost(P, X^*) \nonumber \\
        &+ \frac{(1+\epsilon_1)^2}{(1-\epsilon_2)^2} 4n\Delta_D\Delta_{QT}
    \end{align}
    with probability at least $(1-\delta)^2$.
\end{enumerate}
\end{theorem}

\subsection{Configuration of Joint DR, CR, and QT} \label{sec: Configuration of joint DR, CR and Quantization}

Based on the analysis of the approximation error under given DR, CR, and QT (quantization) methods, 
we aim to answer the following question: how can we configure 
the DR, CR, and QT methods such that we can minimize the communication cost while keeping the approximation error within a given bound? 
We will present a detailed solution for the four-step procedure JL+FSS+JL+QT, as the solutions for the other procedures are similar. 

\subsubsection{Problem Formulation} \label{subsubsec: DR, CR, QT problem statement}

Let $\mathcal{Y}_0$ denote a desired bound on the approximation error and $1-\delta_0$ the desired confidence level, i.e., $\cost(P,X)\leq \mathcal{Y}_0 \cost(P,X^*)$ with probability $\geq 1-\delta_0$, where $X$ is the computed $k$-means solution and $X^*$ the optimal solution. By
Theorem~\ref{thm: DR, CR, QT}, the QT step introduces an additive error. We now convert it into a multiplicative error to enforce the bound $\mathcal{Y}_0$. To this end, suppose that we are given a lower bound $\mathcal{E}$ on the optimal $k$-means cost $\cost(P,X^*)$. For example, by \cite{aggarwal2009adaptive}, we can estimate $\mathcal{E}$ by selecting $O(k)$ points from $P$ according to a certain probability distribution, repeating this process for $\log(1/\delta)$ times, and outputting the set $X$ of selected points with the minimum $\cost(P,X)$. This result is proven to be at most 20-time worse than the optimal solution, i.e., $\mathcal{E}:= \cost(P,X)/20\leq \cost(P,X^*)$, with probability at least $1-\delta$. 
Define $\epsilon_{QT} := \frac{4n\Delta_D\Delta_{QT}}{\mathcal{E}}$. 
Then based on \eqref{eq:JL+FSS+JL+QT}, we have: 
\begin{align}
   &\cost(P,X) \\
   &\leq \frac{(1+\epsilon_1^{(1)})^4(1+\epsilon_2)(1+\epsilon_1^{(2)})^4}{(1-\epsilon_2)} \cost(P, X^*) \nonumber \\
   &~~~~+ \frac{(1+\epsilon_1^{(1)})^2(1+\epsilon_1^{(2)})^2}{(1-\epsilon_2)} 4n\Delta_D\Delta_{QT}  \nonumber\\
   &\leq \frac{(1+\epsilon_1^{(1)})^4(1+\epsilon_2)(1+\epsilon_1^{(2)})^4}{(1-\epsilon_2)} \cost(P, X^*) \nonumber \\ 
   &~~~~+ \frac{(1+\epsilon_1^{(1)})^2(1+\epsilon_1^{(2)})^2}{(1-\epsilon_2)} \frac{4n\Delta_D\Delta_{QT}}{\mathcal{E}}\cost(P,X^*) \nonumber \\ 
   &= \frac{(1+\epsilon_1^{(1)})^2(1+\epsilon_1^{(2)})^2}{(1-\epsilon_2)} \times \nonumber \\
   &~~~~((1+\epsilon_1^{(1)})^2(1+\epsilon_2)(1+\epsilon_1^{(2)})^2 + \epsilon_{QT}) \cost(P,X^*).
\end{align}

Let $f(\epsilon_1^{(1)}, \epsilon_2, \epsilon_1^{(2)}, \epsilon_{QT})$ denote the communication cost as a function of the configuration parameters
$\epsilon_1^{(1)}$, $\epsilon_2$, $\epsilon_1^{(2)}$, and $\epsilon_{QT}$. Our goal is to find the optimal configuration that minimizes the communication cost while satisfying the given bound on the approximation error: 
\begin{subequations}\label{prob: DR, CR, QT}
\begin{align}
\min_{\epsilon_1^{(1)}, \epsilon_2, \epsilon_1^{(2)}, \epsilon_{QT}} \quad & \mathcal{X} := f(\epsilon_1^{(1)}, \epsilon_2, \epsilon_1^{(2)}, \epsilon_{QT}) \label{prob: DR, CR, QT:obj}\\ 
\textrm{s.t.} \quad   \mathcal{Y} &:= \frac{(1+\epsilon_1^{(1)})^2(1+\epsilon_1^{(2)})^2}{(1-\epsilon_2)} \times \nonumber \\ &~~~~~~((1+\epsilon_1^{(1)})^2(1+\epsilon_2)(1+\epsilon_1^{(2)})^2 + \epsilon_{QT}) \nonumber \\
&~\leq \mathcal{Y}_0,\label{prob: DR, CR, QT:constraint}
\end{align} 
\end{subequations}
where $\mathcal{X}$ denotes the communication cost and $\mathcal{Y}$ denotes (an upper bound on) the approximation error. The parameter $\delta$ is set to $1-(1-\delta_0)^{1/3}$ such that the desired confidence level is satisfied. \looseness=-1 

\subsubsection{Analysis}
The communication cost $\mathcal{X}$ is dominated by the transfer of the dimension-reduced, quantized coreset $S'_{QT}$. Let its cardinality, dimensionality, and precision be $n'$, $d'$, and $b'$.  
%
By \cite{Feldman13SODA:report}, the cardinality of an $\epsilon_2$-coreset generated by FSS is $n'=O(\frac{k^3\log^2(k) \log(1/\delta)}{\epsilon_2^4})$. To satisfy Lemma~\ref{lem:JL projection} with $\epsilon_1^{(2)}$, the dimensionality needs to satisfy $d'=O(\frac{\log(n'k/\delta)}{(\epsilon_1^{(2)})^2})$. By the analysis of quantization error in Section~\ref{subsec: Incorporating quantization}, $b'=O(\log(\frac{n\sqrt{d}}{\epsilon_{QT}}))$. 
Denoting the constant factors in these big-$O$ terms by $C_1$, $C_2$, and $C_3$, we have
\begin{align}
\mathcal{X} &\approx n' \cdot d' \cdot b' \\
&= C_1\frac{k^3\log^2(k) \log(1/\delta)}{\epsilon_2^4} \cdot C_2 \frac{\log(n'k/\delta)}{(\epsilon_1^{(2)})^2} \cdot C_3\log\Big(\frac{n\sqrt{d}}{\epsilon_{QT}}\Big) \label{eq:communication cost} \\
&= \tilde{C_1} \cdot \frac{\log(\tilde{C_2} / \epsilon_2^4)}{\epsilon_2^4 (\epsilon_1^{(2)})^2} \cdot \nonumber 
\\ &~~\log\Bigg(\frac{\tilde{C_3}}{\frac{\mathcal{Y}_0(1-\epsilon_2)}{(1+\epsilon_1^{(1)})^2(1+\epsilon_1^{(2)})^2} - (1+\epsilon_1^{(1)})^2(1+\epsilon_2)(1+\epsilon_1^{(2)})^2 }\Bigg), \label{eq: comm. cost w.r.t. Y}
\end{align}
where 
$\tilde{C_1} = k^3\log^2(k)\log(\frac{1}{\delta})C_1C_2C_3$, $\tilde{C_2} = \frac{k^4\log^2(k)\log(\frac{1}{\delta})}{\delta}$, and $\tilde{C_3} = n\sqrt{d}$. Assuming $k\geq 2$, by plugging the constant factors from \cite{Langberg10SODA, blumer1989learnability, li2001improved} into Theorem~36 from \cite{Feldman13SODA:report}, we can use $C_1 = 54912(1+\log_2(3))(1+\log_2(26/3))/225$. 
By JL projection, $d' \leq \lceil{8*\log(\frac{4n'k}{\delta}) / \epsilon^2}\rceil$, and therefore $C_2$ could be $24$. 
Assuming $n>8$ and $\mathcal{E}\geq 1$, $C_3$ could be 2. 
Equation \eqref{eq: comm. cost w.r.t. Y} is obtained by replacing $\epsilon_{QT}$ by its maximum value derived from \eqref{prob: DR, CR, QT:constraint}. 

While solving \eqref{prob: DR, CR, QT} in the general case is nontrivial, we note that for the rounding-based quantizer defined in Section~\ref{subsec: Incorporating quantization}, $\epsilon_{QT}$ only has a finite number of possible values, corresponding to the number of significant bits $s=1,\ldots,b_0-1-m_e$. Thus, under the simplifying constraint of $\epsilon_1^{(1)}=\epsilon_2 = \epsilon_1^{(2)} =: \epsilon$, we can enumerate each possible value of $\epsilon_{QT}$, compute the maximum $\epsilon$ under this $\epsilon_{QT}$ from \eqref{prob: DR, CR, QT:constraint}, and plug it into \eqref{eq: comm. cost w.r.t. Y} to evaluate $\mathcal{X}$. We can then select the configuration that yields the minimum $\mathcal{X}$. 

\section{Performance Evaluation}\label{sec:evaluations for DR+CR+QT}

We now use experiments on real datasets to validate our analysis about the proposed 
joint DR, CR and QT algorithms in comparison with the state of the art in both the single-source and the multiple-source cases. 

\subsection{Datasets, Metrics, and Test Environment}\label{subsec:Platform, Datasets, Metrics} 

We use two datasets: (1) MNIST training dataset \cite{MNIST}, a handwritten digits dataset which has $60,000$ images in $784$-dimensional space; (2) NeurIPS Conference Papers 1987-2015 dataset \cite{perrone2016poisson}, a word counts dataset of the NeurIPS conference papers published from 1987 to 2015, which has $11,463$ instances (words) with $5,812$ attributes (papers). Both of these two datasets are normalized to $[-1, 1]$ with zero mean. In the case of multiple data sources, we randomly partition each dataset among $10$ data sources. 

We measure the performance by (i) the approximation error, measured by the normalized $k$-means cost $\cost(P,X)/\cost(P,X^*)$, where $X$ is the set of centers returned by the evaluated algorithm and $X^*$ is the set of centers computed from $P$, (ii) the normalized communication cost, measured by the ratio between the number of bits 
transmitted by the data source(s) and the size of $P$, and (iii) the complexity at the data source(s), measured by the running time of the evaluated DR/CR algorithm.  We set $k=2$ in all the experiments. Because of the randomness of the algorithms, we repeat each test for $10$ Monte Carlo runs\footnote{The number of Monte Carlo runs is limited by the running time of the experiment, which takes around 30 hours to complete one Monte Carlo run for all the algorithms and all the settings in Section \ref{sec: Results for Joint DR, CR and QT}.}.\looseness=-1

\rev{We run an edge-based machine learning system in a simulated environment and consider both cases of single and multiple data sources. All the experiments are conducted on a Windows machine with Intel i7-8700 CPU and 48GB DDR4 memory. We note that our simulated results closely resemble those in an actual distributed system, as the performance metrics we measure are either independent of the test environment (approximation error and communication cost) or only dependent through a scaling factor (running time). Although the absolute value of the running time depends on the processor speed at the data source, different processor speeds will only cause different scaling factors, and thus the running times obtained in our experiments can still be used to compare the complexities between algorithms.}\looseness=-1

\subsection{Results for Joint DR and CR}

\subsubsection{Evaluated Algorithms}\label{subsubsec:DR+CR: Evaluated Algorithms}

In the case of a single data source, we evaluate the following algorithms:
\begin{itemize}
    \item ``FSS": the benchmark algorithm introduced in \cite{Feldman13SODA:report},
    \item ``JL+FSS": Algorithm \ref{Alg:JL DR+CR}, where we use JL projection before applying FSS, 
    \item ``FSS+JL": Algorithm \ref{Alg:FSS CR+JL DR}, where we use JL projection after applying FSS, and 
    \item ``JL+FSS+JL": Algorithm \ref{Alg:JL+FSS+JL}, where we apply JL projection both before and after FSS. 
\end{itemize}
In the case of multiple data sources, we evaluate the following algorithms:\looseness=-1
\begin{itemize}
    \item ``BKLW": the benchmark algorithm from \cite{Balcan14NIPS}, and 
    \item ``JL+BKLW": Algorithm \ref{Alg:distributed kmeans}, where we apply JL projection before BKLW. 
\end{itemize}
In both cases, we have tuned the parameters of both the benchmark and proposed algorithms to make all the algorithms achieve a similar empirical approximation error. As a baseline, we also include the naive method of ``no reduction (NR)'', i.e., transmitting the raw data.


\subsubsection{Results} \label{subsubsec:Joint DR and CR: Results}

\rev{In the case of a single data source, the data source computes and reports a data summary using the evaluated DR/CR algorithms, based on which a server computes $k$-means centers.}
The results are given in Figure~\ref{fig:single} and Table~\ref{tab:single source - comm cost}. Note that by definition, the baseline (NR) has a normalized $k$-means cost of $1$, a normalized communication cost of 1, and no computation at the data source. We observe the following: 
(i) Compared to the naive method of transmitting the raw data (NR), the proposed algorithms can dramatically reduce the communication cost (by $>99\%$) with a moderate increase in the $k$-means cost ($<10\%$). 
(ii) Compared to the benchmark (FSS), the proposed algorithms can achieve a similar or smaller $k$-means cost while significantly reducing the communication cost and/or complexity, \rev{which is thanks to the proper application of JL projection and consistent with our theoretical analysis in Table~\ref{tab:comparison}}. 
(iii) Between the approaches of DR+CR (JL+FSS) and CR+DR (FSS+JL), we see that the DR+CR approach yields a better performance for the NeurIPS dataset, where JL+FSS has a substantially shorter running time than FSS+JL but similar $k$-means cost and communication cost. This is because $\log{n}\ll \min(n,d)$ for this dataset, allowing JL+FSS to significantly reduce the complexity compared with FSS+JL without blowing up the communication cost \rev{according to our analysis in Table~\ref{tab:comparison}}. 
(iv) For a sufficiently high-dimensional dataset such as NeurIPS, JL+FSS+JL can \rev{further improve the communication-computation tradeoff while achieving a similar $k$-means cost, which is again consistent with our analysis in Table~\ref{tab:comparison}.} 

\begin{figure}[t]
\vspace{-.5em}
\begin{minipage}{0.95\linewidth}
\begin{minipage}{0.49\textwidth}
\centerline{
\includegraphics[width=\textwidth,height=3.5cm]{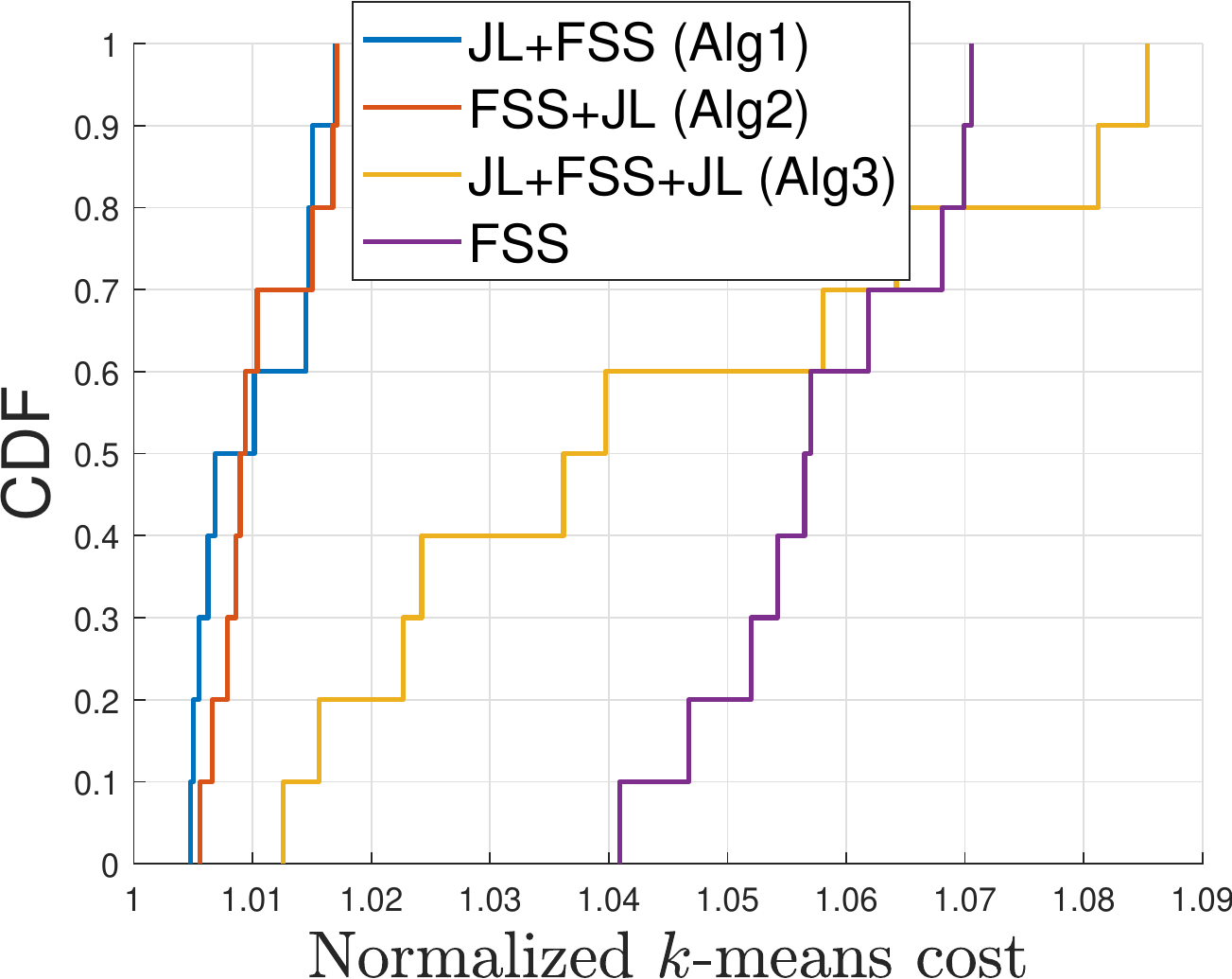}}
\vspace{+.6em}
\end{minipage}
\centering
\begin{minipage}{.49\textwidth}
\centerline{
\includegraphics[width=\textwidth,,height=3.5cm]{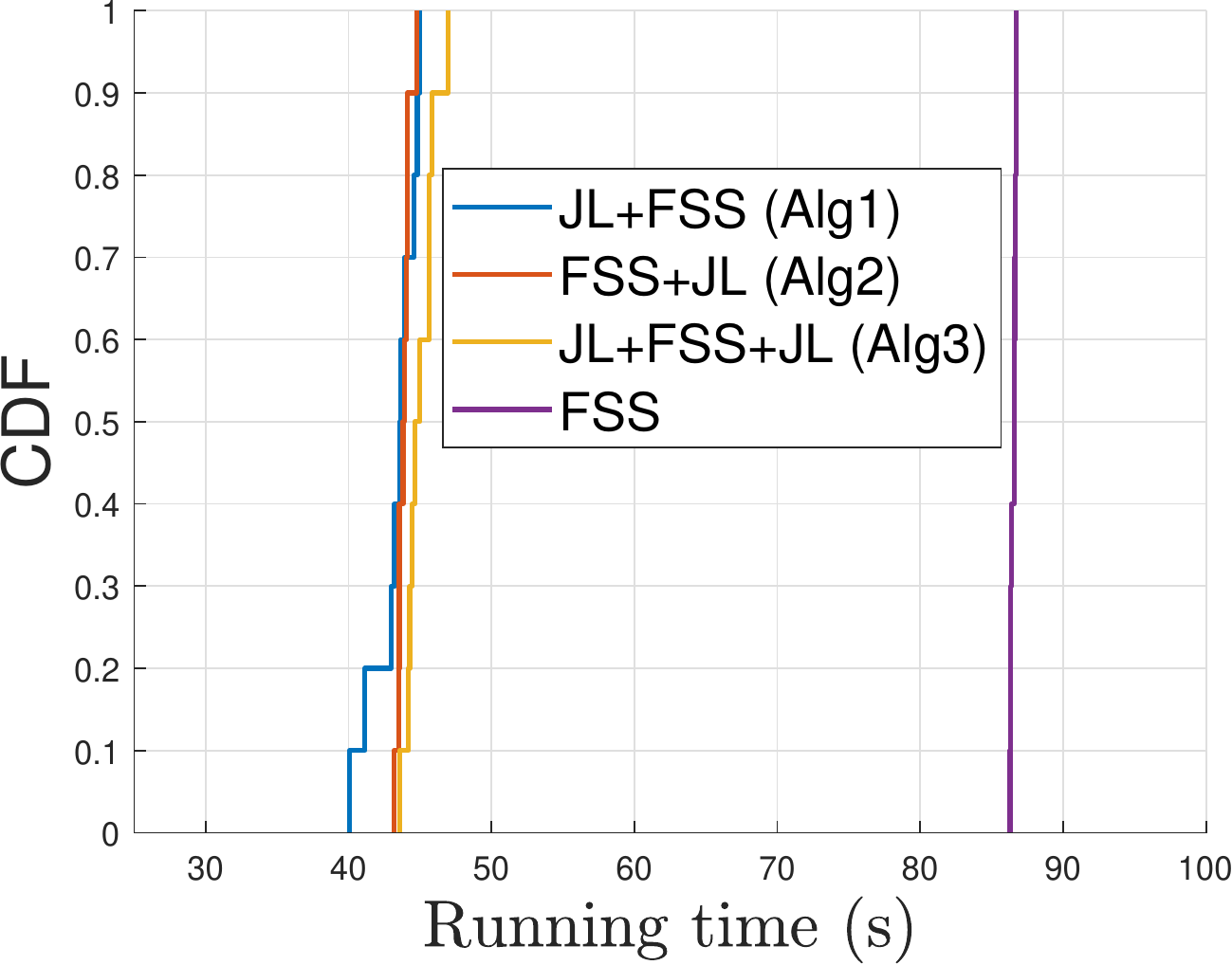}}
\vspace{+.4em}
\end{minipage}
\centerline{\scriptsize (a)  MNIST 
}
\end{minipage}\hfill
\vspace{+.4em}
 \begin{minipage}{0.95\linewidth}
\begin{minipage}{0.49\textwidth}
\centerline{
\includegraphics[width=\textwidth,height=3.5cm]{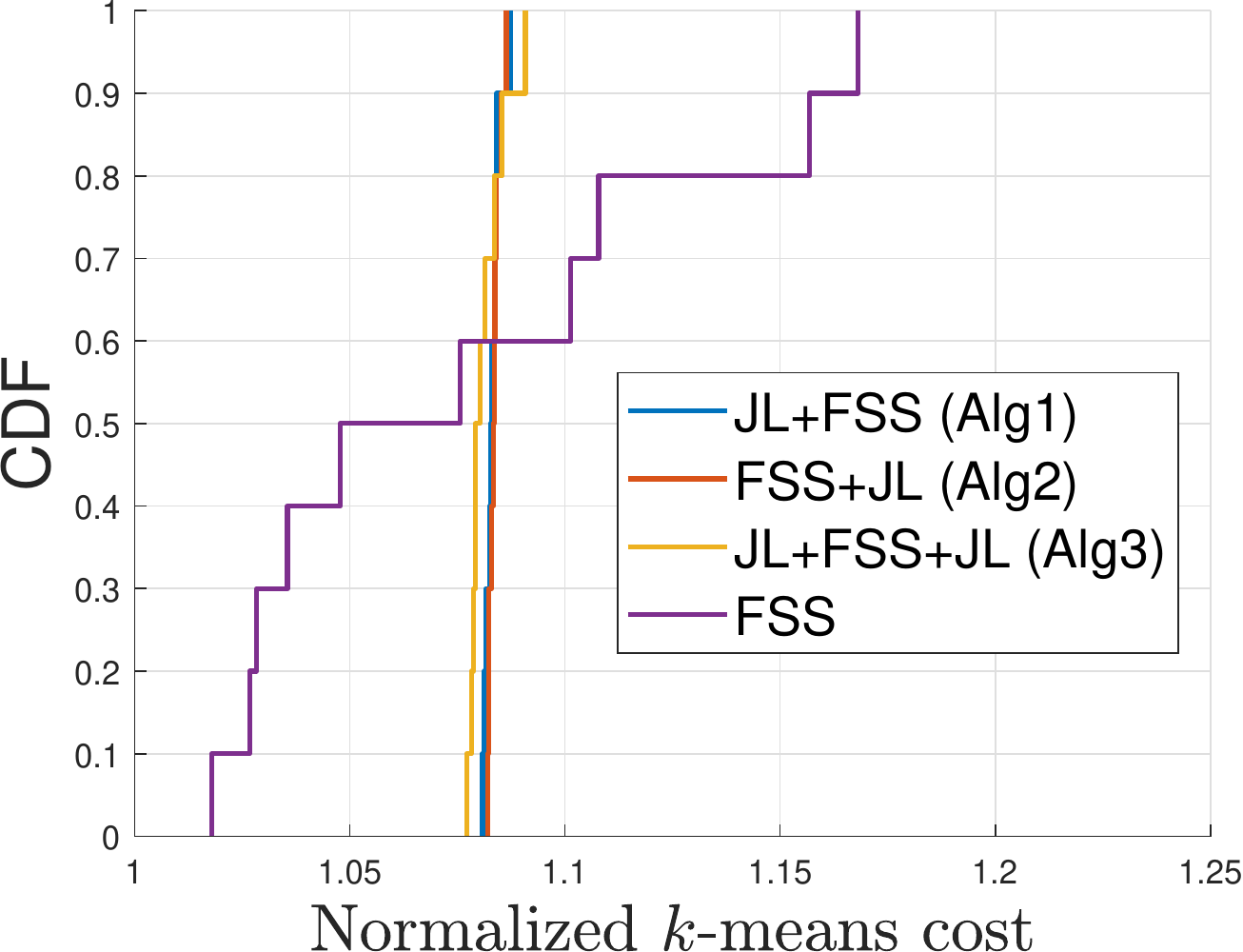}}
\vspace{+.6em}
\end{minipage}
\centering
\begin{minipage}{.49\textwidth}
\centerline{
\includegraphics[width=\textwidth,,height=3.5cm]{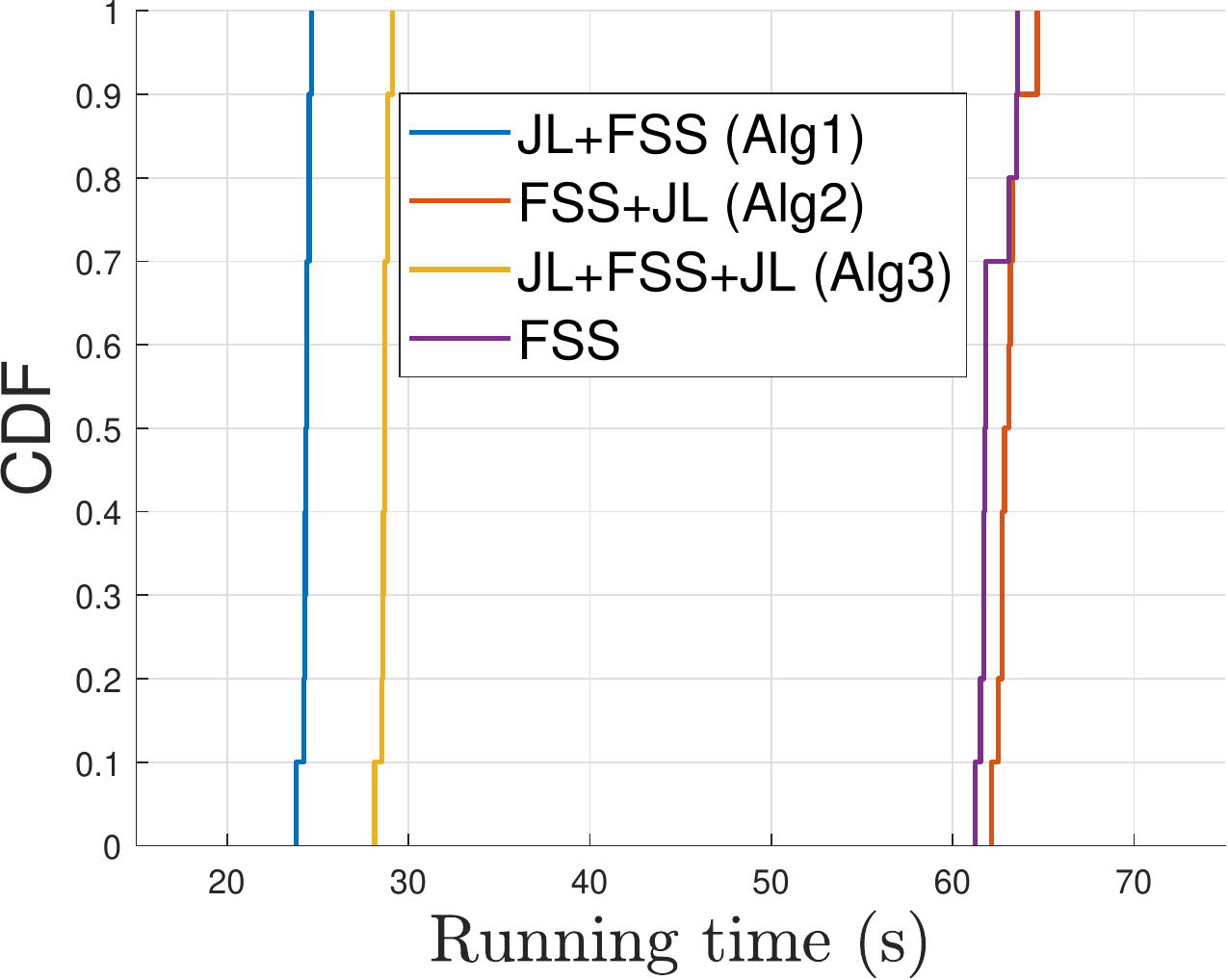}}
\vspace{+.4em}
\end{minipage}
\centerline{\scriptsize (b)  NeurIPS 
}
\end{minipage}
\caption{Single-source case: normalized $k$-means cost and running time}\label{fig:single}
  \vspace{-.5em}
\end{figure}

\begin{table}[b]
\vspace{-1em}
\caption{Single-source Case: Normalized Communication Cost} \label{tab:single source - comm cost}
\centering
\vspace{-.5em}
\begin{tabular}{r| c| c |c |c |c}
\hline
Dataset & NR & FSS & JL+FSS & FSS+JL & JL+FSS+JL\\ \hline
MNIST & 1 & 8.95e-3 & 5.82e-3  & 5.82e-3 & 5.97e-3 \\ \hline
NeurIPS & 1 & 5.87e-3 & 3.60e-3 & 3.59e-3 & 2.84e-3 \\ \hline
\end{tabular}
\vspace{-0.2cm}
\end{table}

\rev{In the case of multiple data sources, $m=10$ data sources cooperatively compute and report a data summary using the evaluated distributed DR/CR algorithms, based on which a server computes $k$-means centers for the union of the $m$ local datasets.}
The results are shown in Figure~\ref{fig:multiple} and Table~\ref{tab:multiple source - comm cost}. We see that the proposed algorithm (JL+BKLW) achieves a $k$-means cost comparable to the benchmark (BKLW), while incurring a lower complexity and a lower communication cost. \rev{This improvement is again thanks to the suitable application of JL projection, which is efficient in both computational complexity and communication cost; the observations are consistent with the analysis in Table~\ref{tab:comparison}.}
Recall that applying JL projection after BKLW will not reduce the communication cost or the complexity as explained after Theorem~\ref{thm:distributed kmeans}. 

\begin{figure}[t]
\begin{minipage}{0.95\linewidth}
\centering
\begin{minipage}{0.49\textwidth}
\centerline{
\includegraphics[width=\textwidth,height=3.5cm]{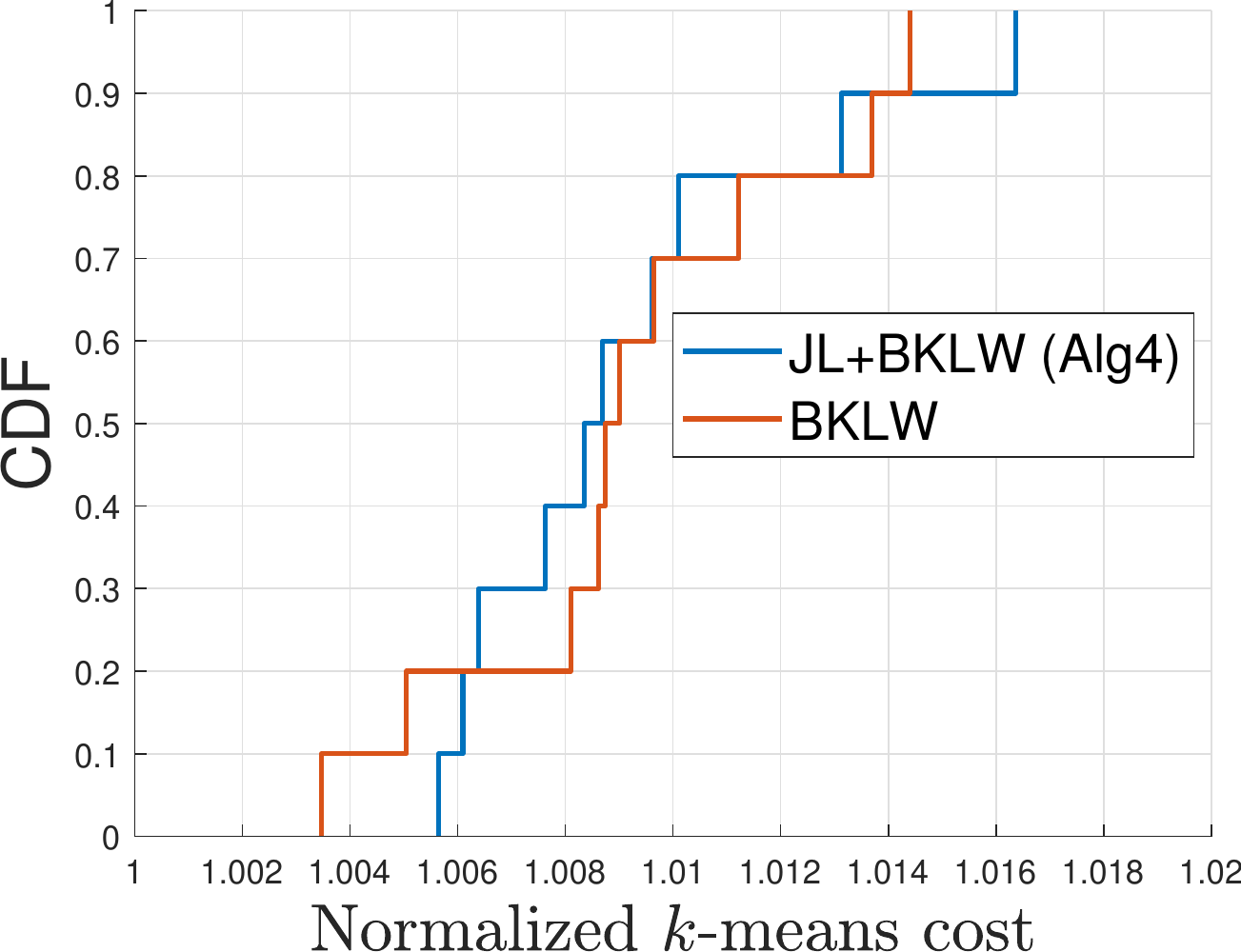}}
\vspace{+.6em}
\end{minipage}
\begin{minipage}{.49\textwidth}
\centerline{
\includegraphics[width=\textwidth,,height=3.5cm]{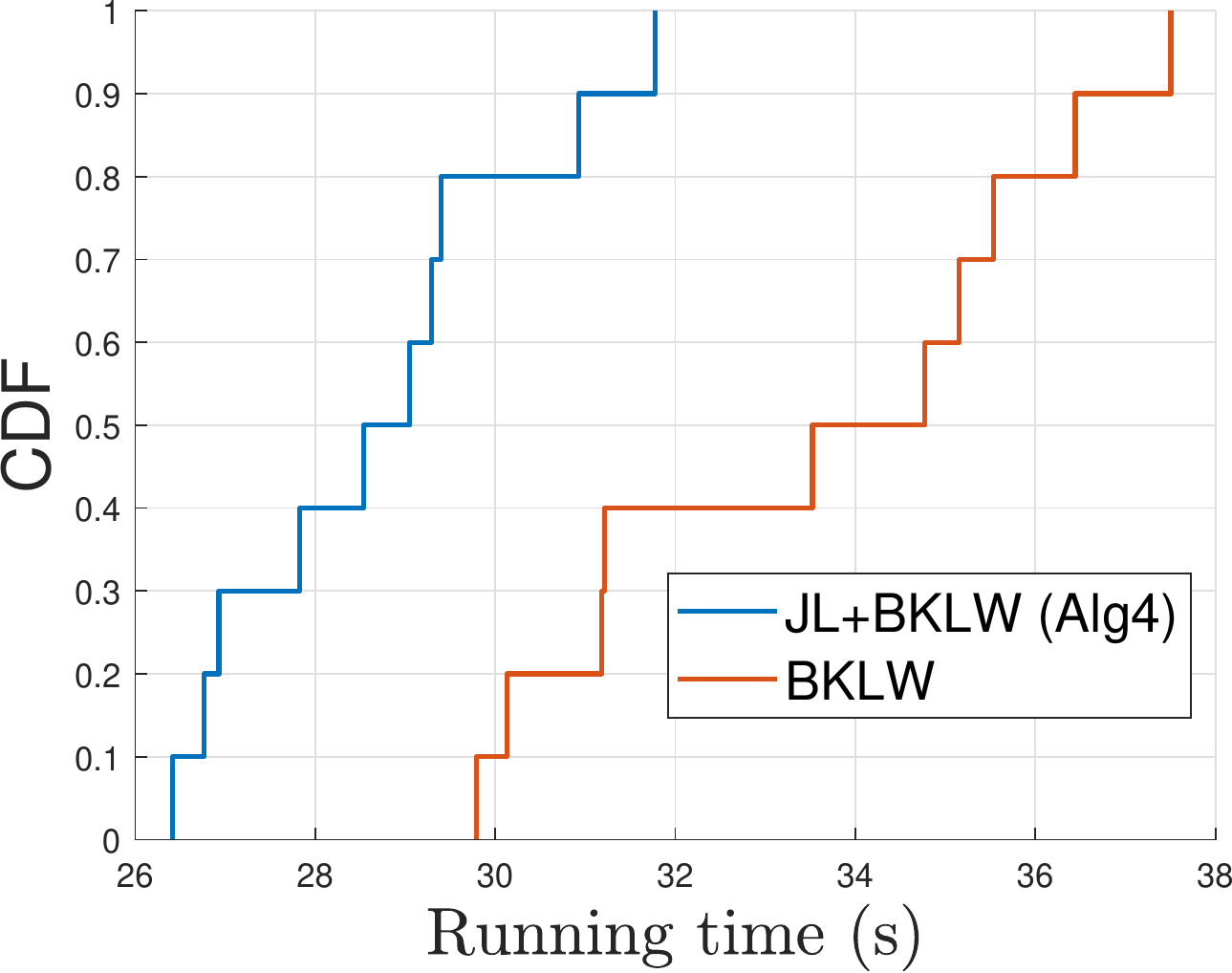}}
\vspace{+.4em}
\end{minipage}
\centerline{\scriptsize (a) MNIST 
}
\end{minipage}\hfill
\vspace{+.4em}
\begin{minipage}{0.95\linewidth}
\centering
\begin{minipage}{0.49\textwidth}
\centerline{
\includegraphics[width=\textwidth,height=3.5cm]{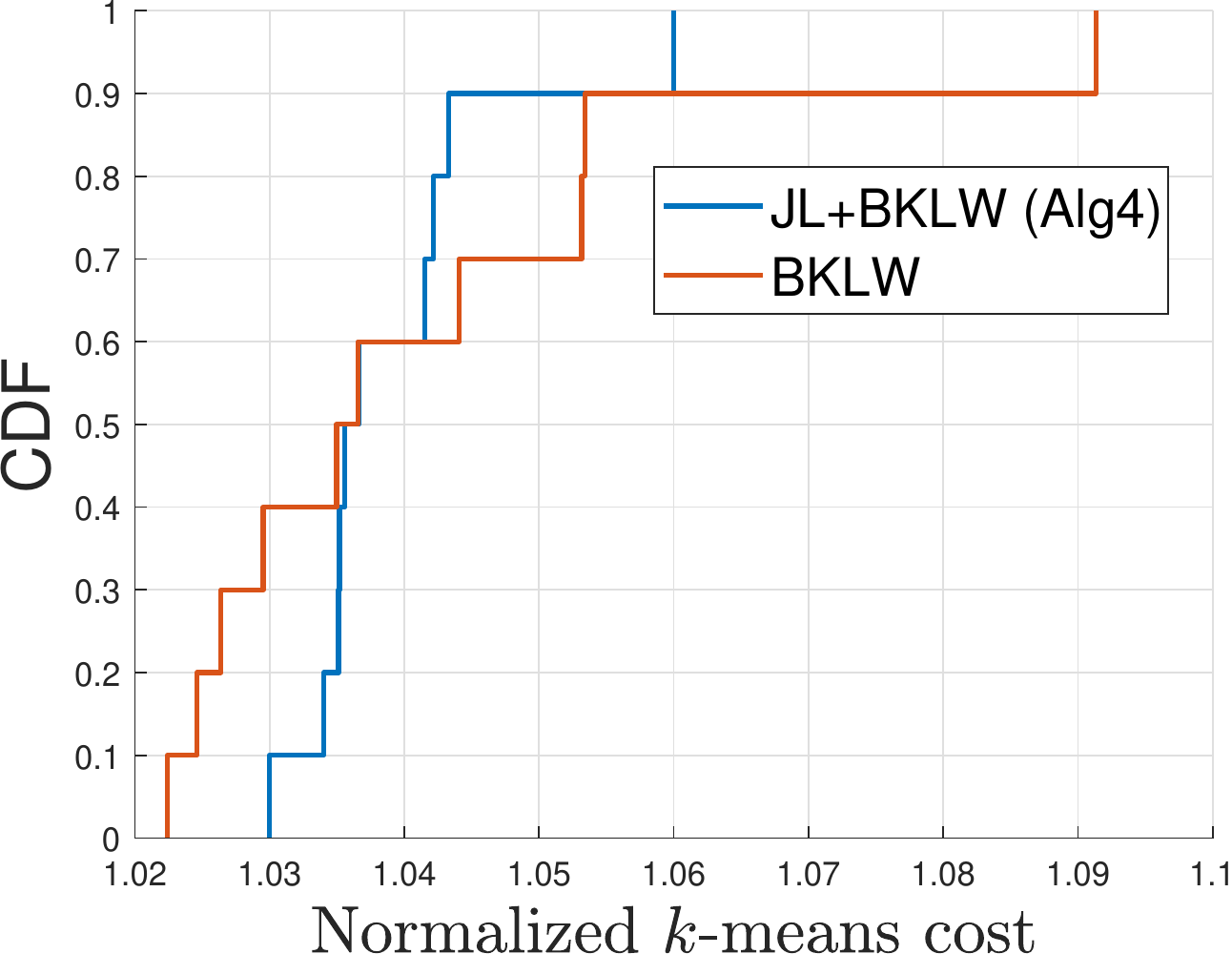}}
\vspace{+.6em}
\end{minipage}
\begin{minipage}{0.49\textwidth}
\centerline{
\includegraphics[width=\textwidth,,height=3.5cm]{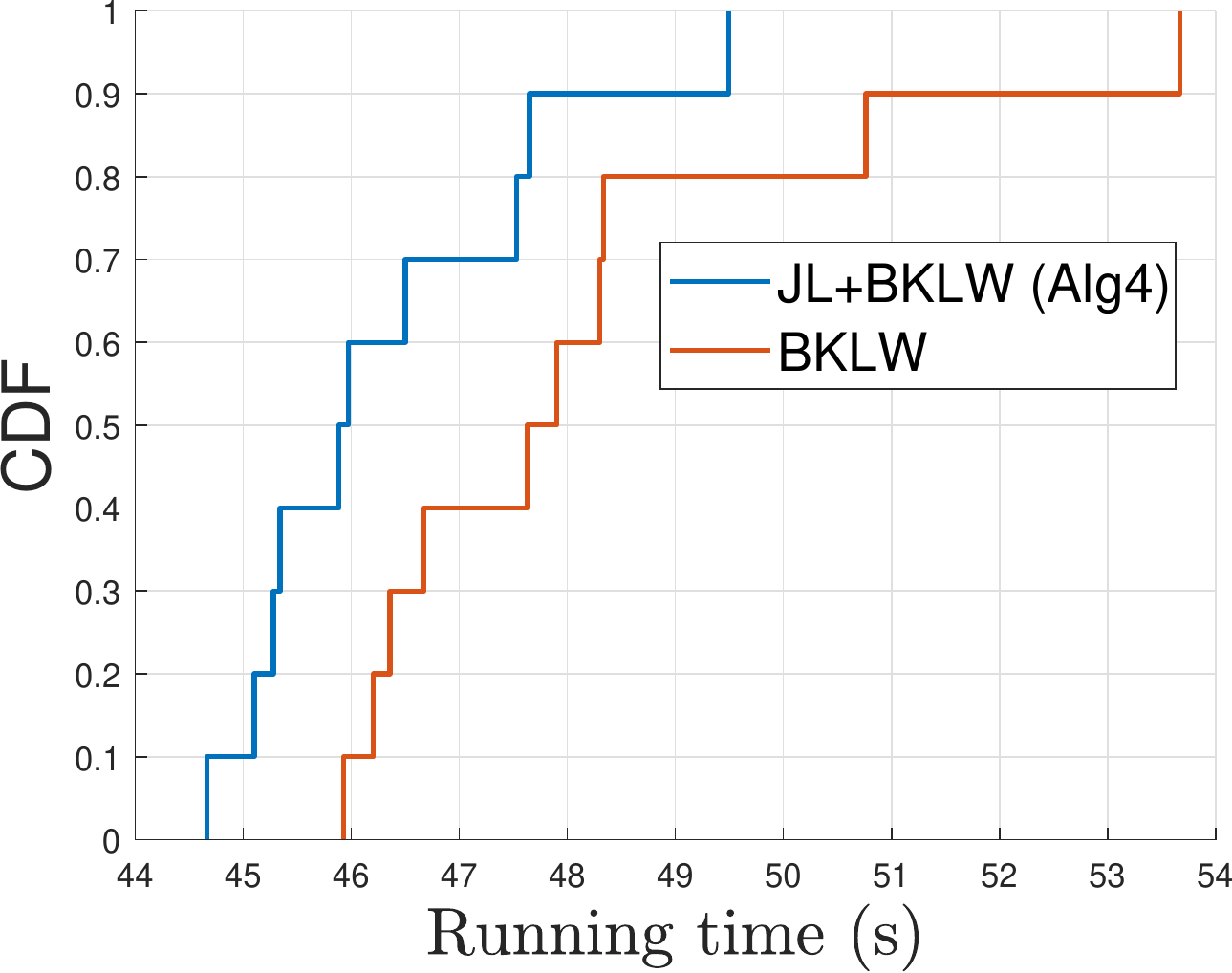}}
\vspace{+.4em}
\end{minipage}
\centerline{\scriptsize (b) NeurIPS 
}
\end{minipage}
\caption{Multiple-source case: normalized $k$-means cost and running time} \label{fig:multiple}
   \vspace{-.5em}
\end{figure}

\begin{table}[t]
\vspace{-.5em}
\caption{Multiple-source Case: Normalized Communication Cost} \label{tab:multiple source - comm cost}
\vspace{-.5em}
\centering
\begin{tabular}{r| c| c |c }
\hline
Dataset & NR & BKLW & JL+BKLW \\ \hline
MNIST & 1 & 1.97e-2 & 1.69e-2 \\ \hline
NeurIPS  & 1 & 1.28e-2 & 1.05e-2 \\ \hline
\end{tabular}
 \vspace{-0.5em}
\end{table}


\subsection{Results for Joint DR, CR, and QT} \label{sec: Results for Joint DR, CR and QT}
We assume double precision for the original dataset before applying QT. 
For tractability, in the experiments we set all the $\epsilon$ values except $\epsilon_{QT}$ to be equal when solving \eqref{prob: DR, CR, QT}. 

\begin{figure}[t]
\vspace{-.5em}
\begin{minipage}{0.95\linewidth}
\begin{minipage}{0.49\textwidth}
\centerline{
\includegraphics[width=\textwidth,height=3.5cm]{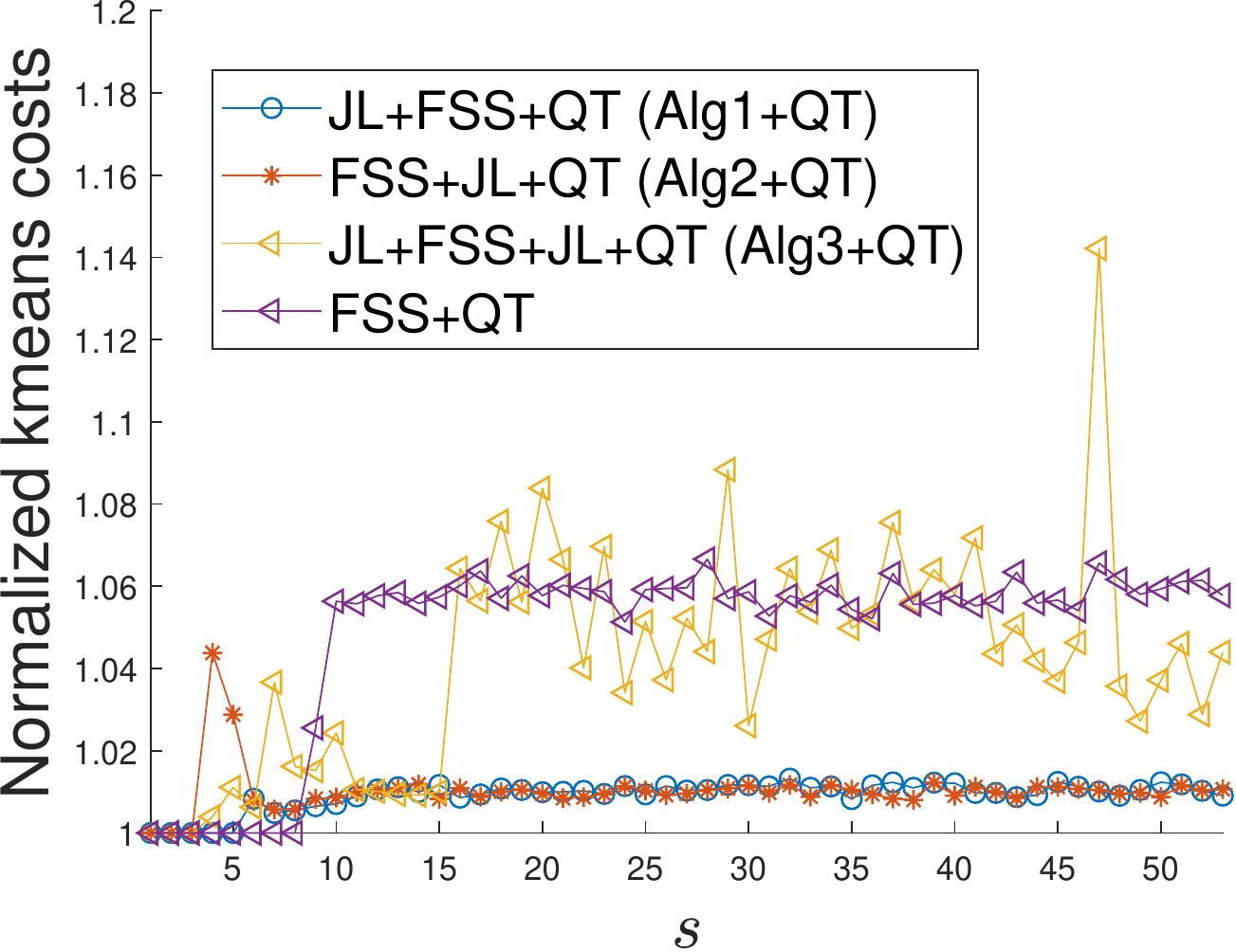}}
\vspace{+.6em}
\centerline{\scriptsize (a) Normalized $k$-means cost}
\end{minipage}
\centering
\begin{minipage}{.49\textwidth}
\centerline{
\includegraphics[width=\textwidth,,height=3.5cm]{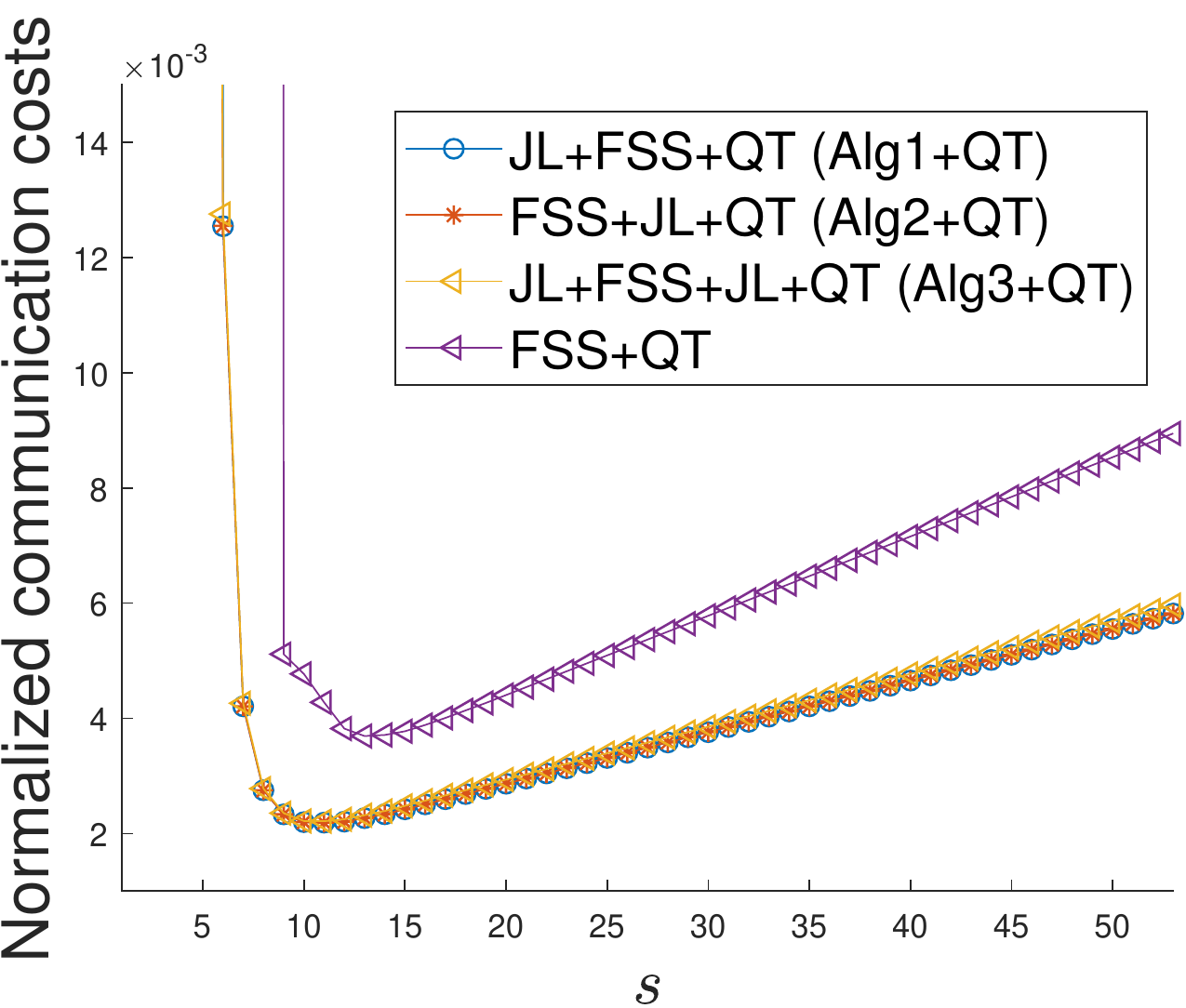}}
\vspace{+.4em}
\centerline{\scriptsize (b) Normalized communication cost}
\end{minipage}
\end{minipage}\hfill
\vspace{+.4em}
 \begin{minipage}{0.95\linewidth}
\begin{minipage}{0.49\textwidth}
\centerline{
\includegraphics[width=\textwidth,height=3.5cm]{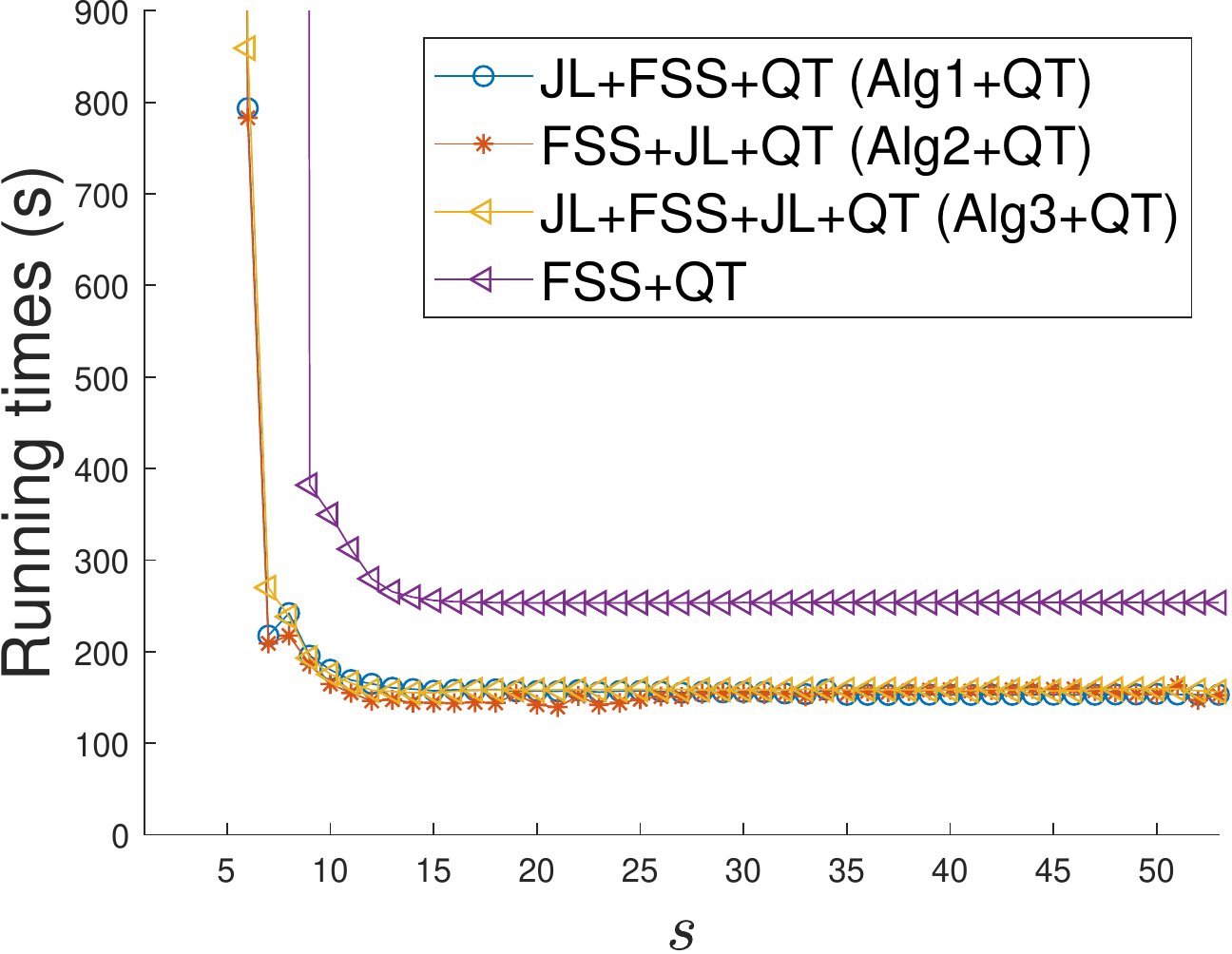}}
\vspace{+.6em}
\centerline{\scriptsize (c) Running time (s)}
\end{minipage}
\centering
\end{minipage}
\caption{Single-source case with quantization: MNIST }
\label{fig: extension - mnist single}
\vspace{-.5em}
\end{figure}

\begin{figure}[t]
\vspace{-.5em}
\begin{minipage}{0.95\linewidth}
\begin{minipage}{0.49\textwidth}
\centerline{
\includegraphics[width=\textwidth,height=3.5cm]{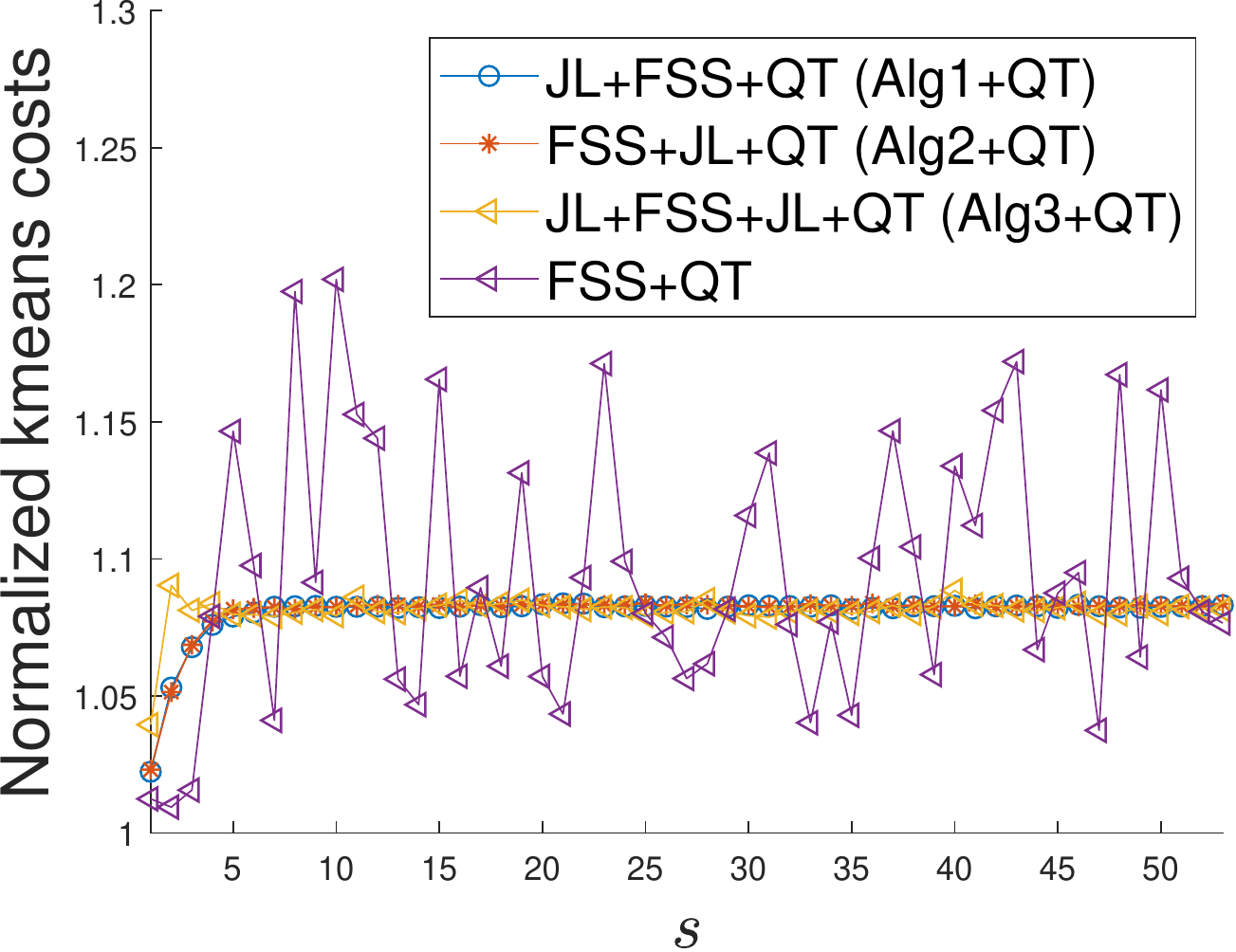}}
\vspace{+.6em}
\centerline{\scriptsize (a) Normalized $k$-means cost}
\end{minipage}
\centering
\begin{minipage}{.49\textwidth}
\centerline{
\includegraphics[width=\textwidth,,height=3.5cm]{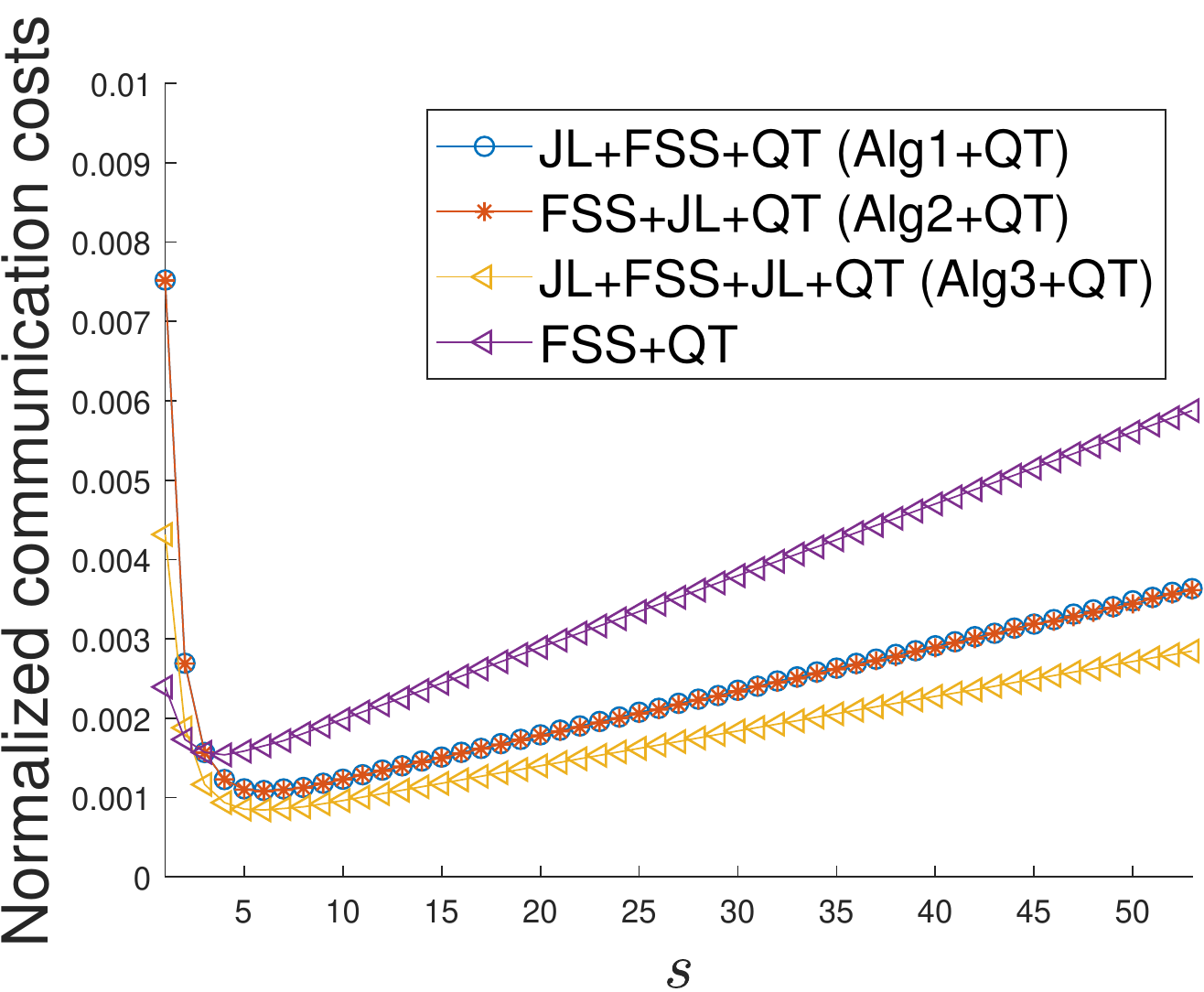}}
\vspace{+.4em}
\centerline{\scriptsize (b) Normalized communication cost}
\end{minipage}
\end{minipage}\hfill
\vspace{+.4em}
 \begin{minipage}{0.95\linewidth}
\begin{minipage}{0.49\textwidth}
\centerline{
\includegraphics[width=\textwidth,height=3.5cm]{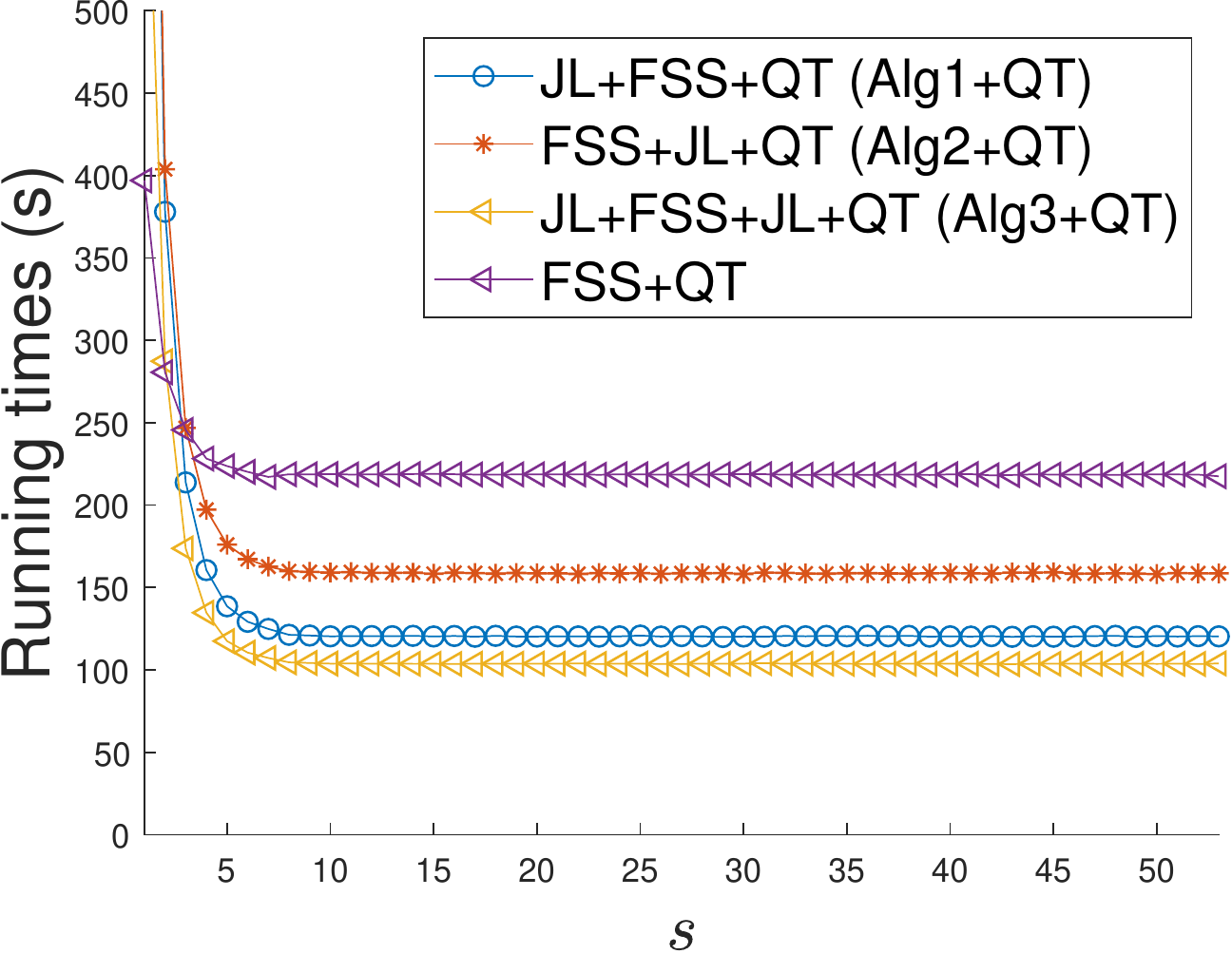}}
\vspace{+.6em}
\centerline{\scriptsize (c) Running time (s)}
\end{minipage}
\centering
\end{minipage}
\caption{Single-source case with quantization: NeurIPS}
\label{fig: extension - nips single}
\vspace{-.5em}
\end{figure}

\subsubsection{Evaluated Algorithms}\label{subsubsec:DR+CR+QT: Evaluated Algorithms}

In the case of a single data source, we evaluate the following:
\begin{itemize}
\item ``FSS+QT'': the quantization-added version of FSS~\cite{Feldman13SODA:report},
    \item ``JL+FSS+QT": the quantization-added version of Algorithm \ref{Alg:JL DR+CR}, 
    \item ``FSS+JL+QT": the quantization-added version of Algorithm \ref{Alg:FSS CR+JL DR}, and 
    \item ``JL+FSS+JL+QT": the quantization-added version of Algorithm \ref{Alg:JL+FSS+JL}. 
\end{itemize}
In the case of multiple data sources, we evaluate the following:\looseness=-1
\begin{itemize}
    \item ``BKLW+QT": the quantization-added version of BKLW~\cite{Balcan14NIPS}, and 
    \item ``JL+BKLW+QT": the quantization-added version of Algorithm \ref{Alg:distributed kmeans}. 
\end{itemize}
For each algorithm, we construct a data summary under each configuration that corresponds to a possible number of significant bits $s$, and then solve $k$-means based on the data summary. 
Since the IEEE Standard 754 floating number representation \cite{IEEE754} consists of 53 significant bits, we enumerate $s=1,\ldots, 53$. 

\begin{figure}[t]
\vspace{-.5em}
\begin{minipage}{0.95\linewidth}
\begin{minipage}{0.49\textwidth}
\centerline{
\includegraphics[width=\textwidth,height=3.5cm]{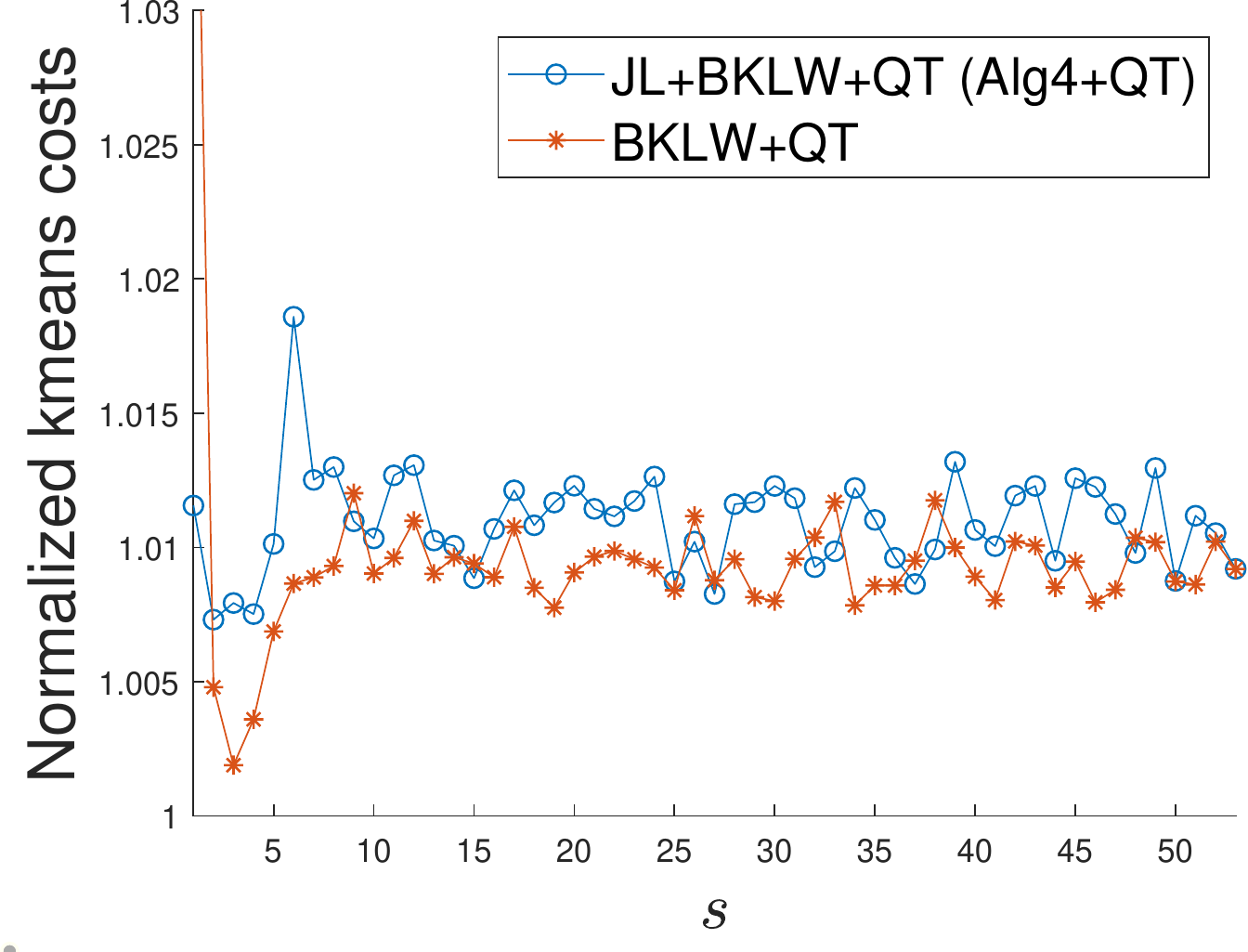}}
\vspace{+.6em}
\centerline{\scriptsize (a) Normalized $k$-means cost}
\end{minipage}
\centering
\begin{minipage}{.49\textwidth}
\centerline{
\includegraphics[width=\textwidth,,height=3.5cm]{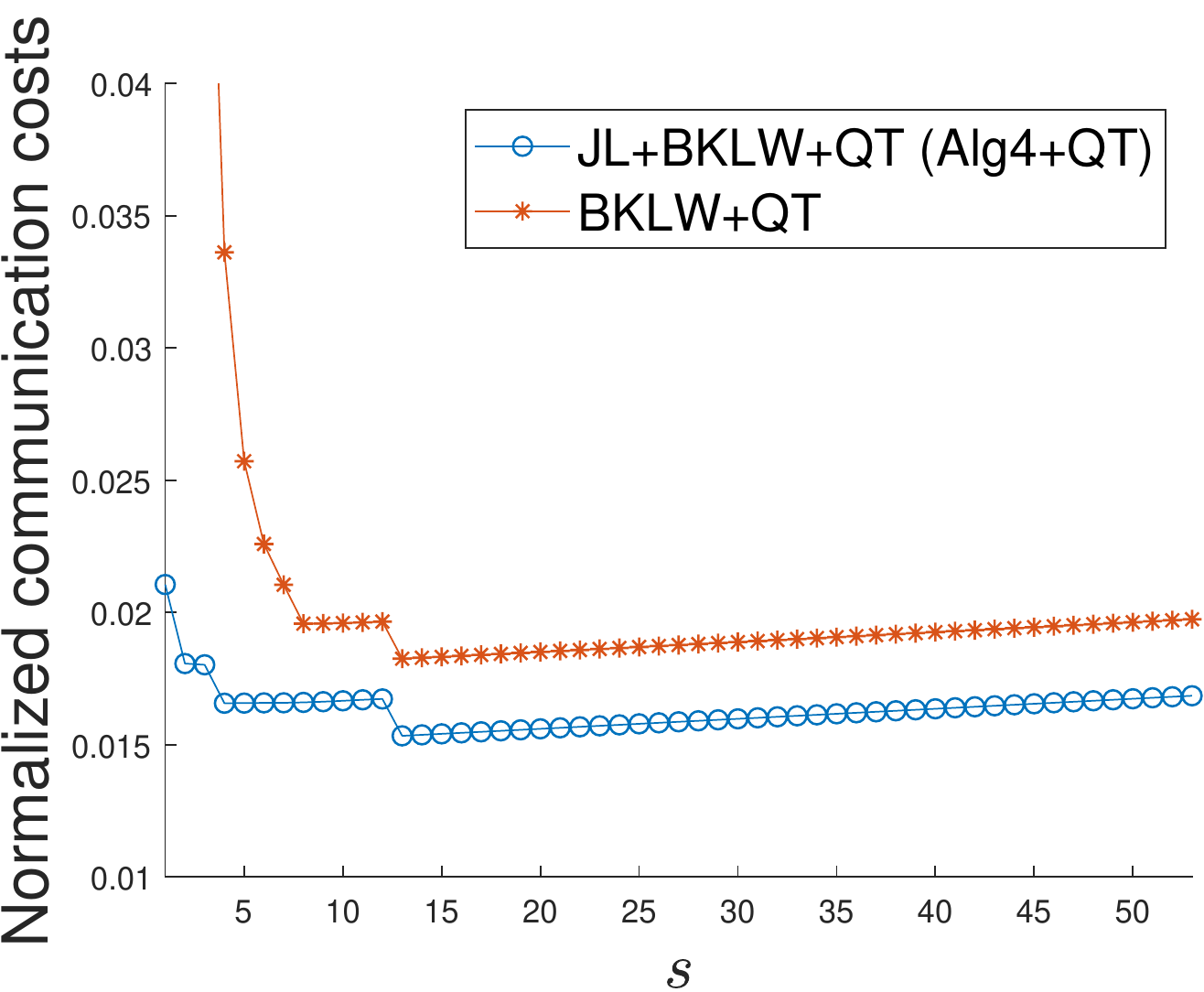}}
\vspace{+.4em}
\centerline{\scriptsize (b) Normalized communication cost}
\end{minipage}
\end{minipage}\hfill
\vspace{+.4em}
 \begin{minipage}{0.95\linewidth}
\begin{minipage}{0.49\textwidth}
\centerline{
\includegraphics[width=\textwidth,height=3.5cm]{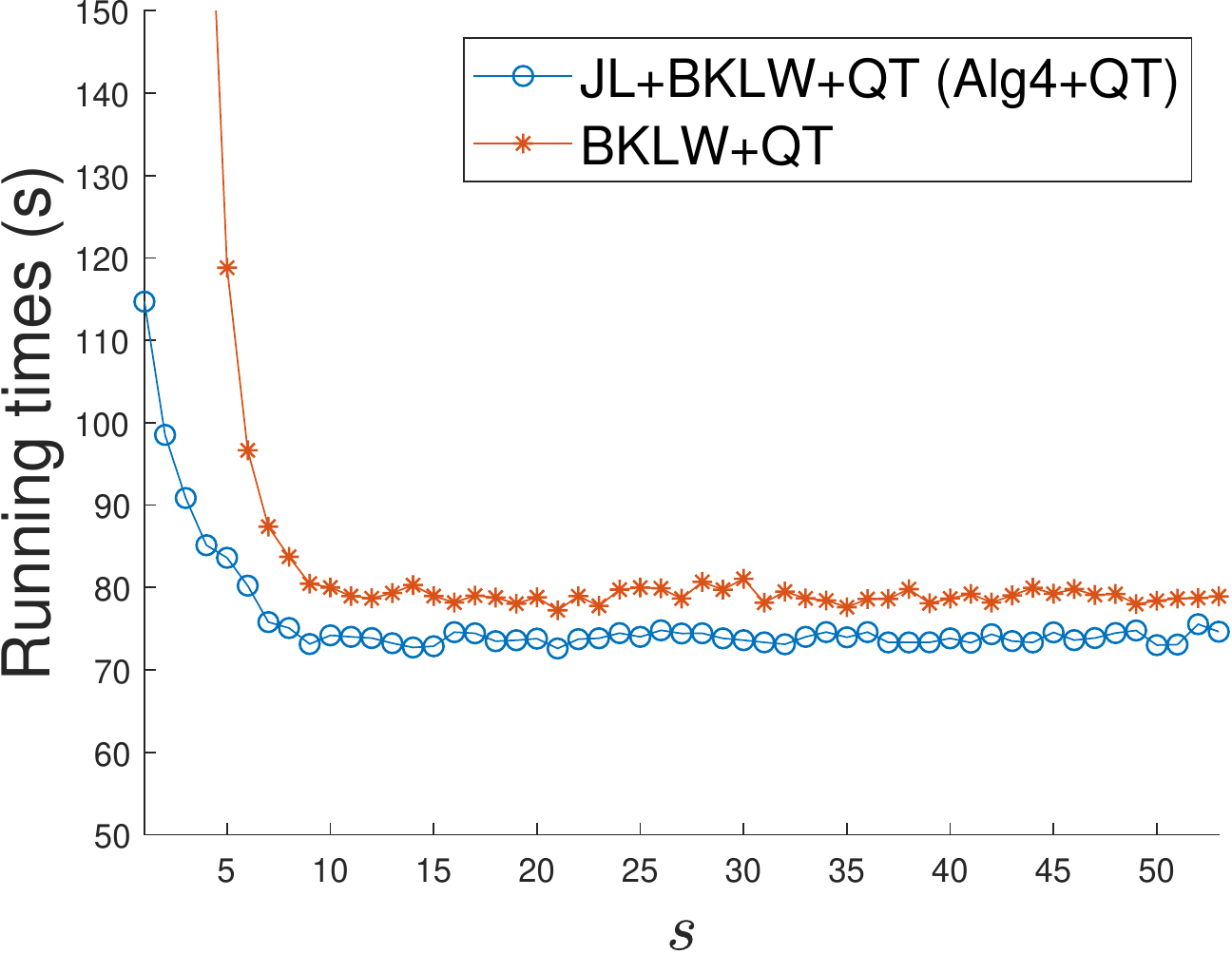}}
\vspace{+.6em}
\centerline{\scriptsize (c) Running time (s)}
\end{minipage}
\centering
\end{minipage}
\caption{Multiple-source case with quantization: MNIST}
\label{fig: extension - mnist multiple}
\vspace{-.5em}
\end{figure}

\begin{figure}[t]
\vspace{-.5em}
\begin{minipage}{0.95\linewidth}
\begin{minipage}{0.49\textwidth}
\centerline{
\includegraphics[width=\textwidth,height=3.5cm]{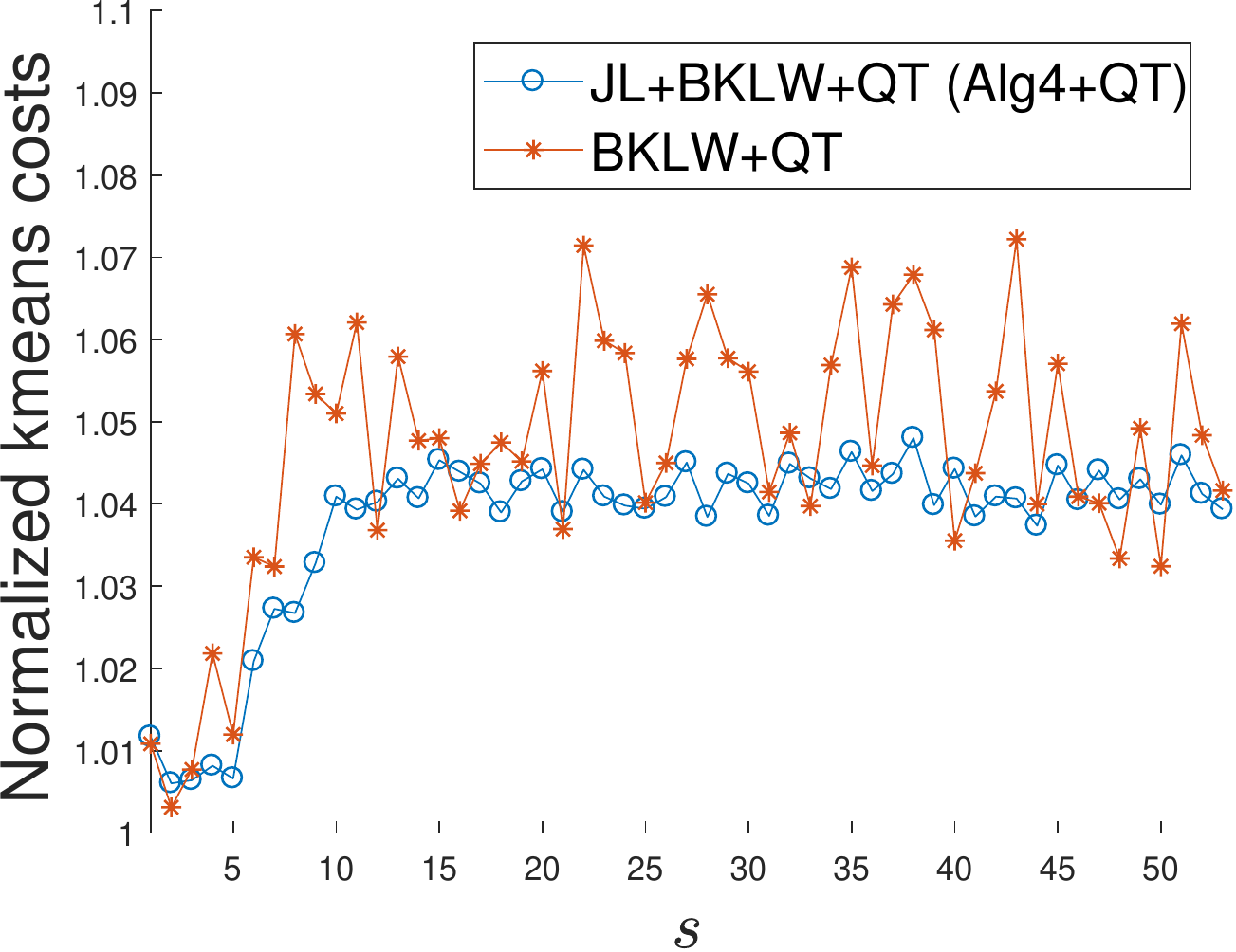}}
\vspace{+.6em}
\centerline{\scriptsize (a) Normalized $k$-means cost}
\end{minipage}
\centering
\begin{minipage}{.49\textwidth}
\centerline{
\includegraphics[width=\textwidth,,height=3.5cm]{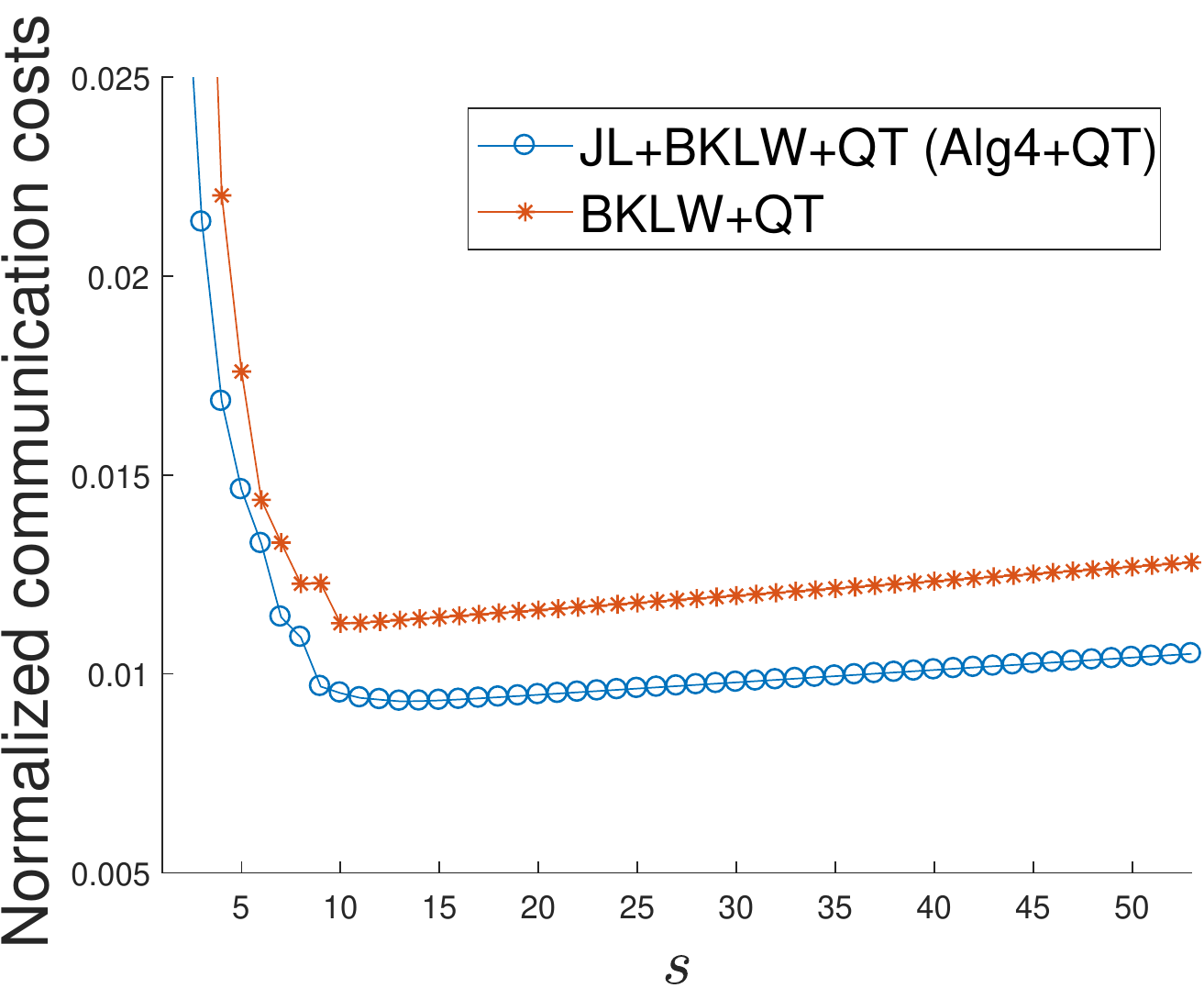}}
\vspace{+.4em}
\centerline{\scriptsize (b) Normalized communication cost}
\end{minipage}
\end{minipage}\hfill
\vspace{+.4em}
 \begin{minipage}{0.95\linewidth}
\begin{minipage}{0.49\textwidth}
\centerline{
\includegraphics[width=\textwidth,height=3.5cm]{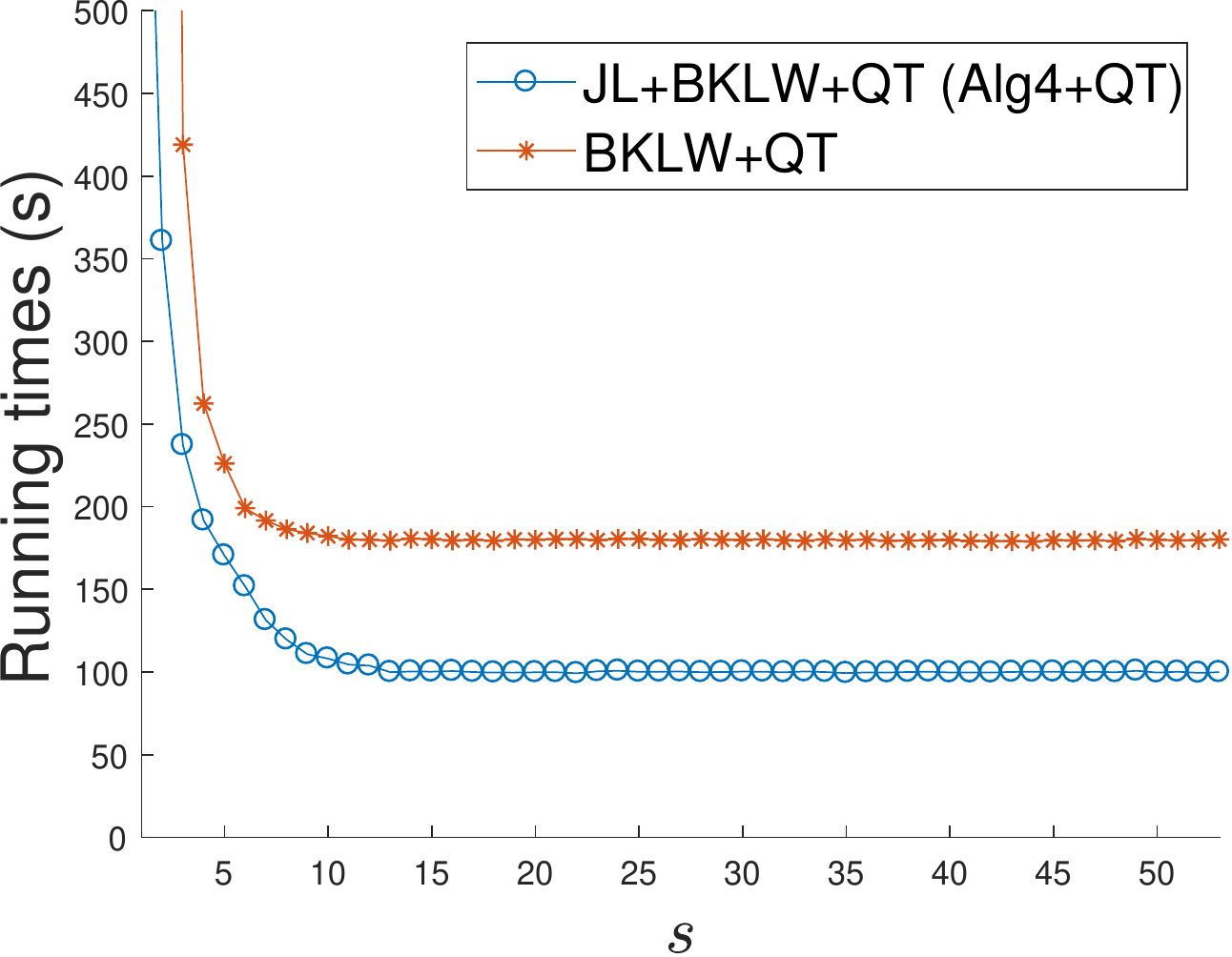}}
\vspace{+.6em}
\centerline{\scriptsize (c) Running time (s)}
\end{minipage}
\centering
\end{minipage}
\caption{Multiple-source case with quantization: NeurIPS}
\label{fig: extension - nips multiple}
\vspace{-.5em}
\end{figure}

\subsubsection{Results}\label{subsubsec:Joint DR, CR, QT: Results}

    The results for the single data source scenario are given in Figures~\ref{fig: extension - mnist single}--\ref{fig: extension - nips single}. We have the following key observations: 
    (i) Compared with the methods without quantization (\rev{the right-most points under $s=53$}), 
    adding suitably configured quantization can further reduce the communication cost by $2/3$ without increasing the $k$-means cost or the running time. \rev{This is because the cluster structure for $k$-means clustering has certain robustness to minor shifts of data points (caused by rounding off a few least significant bits).} 
    (ii) However, it is nontrivial to find the optimal configuration to achieve a comparable $k$-means cost and running time with the least communication cost, as very small and very large values of $s$ both lead to suboptimal performance. \rev{Intuitively, setting $s$ too large will fail to take advantage of the communication cost saving due to quantization, and setting it too small will cause too much quantization error and leave no room of error for DR/CR.} 
    (iii) When the dimensionality is not too high (e.g., MNIST), a three-step procedure such as JL+FSS+QT or FSS+JL+QT suffices; for a high-dimensional dataset such as NeurIPS, the four-step procedure  JL+FSS+JL+QT 
    can further reduce the communication cost and the running time while achieving a comparable $k$-means cost, \rev{which is consistent with the predicted advantage of JL+FSS+JL in the regime of $n,d\gg 1$ as shown in Table~\ref{tab:comparison}}. 
    
    The results for the multiple data source scenario are given in Figures~\ref{fig: extension - mnist multiple}--\ref{fig: extension - nips multiple}. 
    We have similar observations as in the single-source case: (i) \rev{Compared with no quantization (the right-most points under $s=53$)}, adding suitably configured quantization can further reduce the communication cost by $10\%$ without increasing the $k$-means cost or the running time. (ii) 
    Choosing a proper configuration (by selecting the optimal number of significant bits to retain in quantization) is nontrivial. 
    (iii) Compared with BKLW+QT, JL+BKLW+QT can reduce both the  communication cost and the running time while achieving a similar $k$-means cost, 
    which is consistent with the comparison between BKLW and JL+BKLW (Figure~\ref{fig:multiple} and Table~\ref{tab:multiple source - comm cost}) \rev{as well as the theoretical prediction in Table~\ref{tab:comparison}. This result again demonstrates the benefit of properly combining existing DR/CR methods with JL projection}. 

\subsection{Summary of Observations}

Our experimental results imply the following observations: \begin{itemize}
    \item Solving $k$-means based on data summaries generated by DR/CR methods can provide a reasonably good solution at a drastically reduced communication cost without incurring a high complexity at data sources. \looseness=-1
    \item Compared with state-of-the-art algorithms, suitable combination of DR and CR can effectively reduce the communication cost and the complexity while providing a $k$-means solution of a similar quality. 
    \item Augmenting DR and CR with suitably configured quantization can further reduce the communication cost without adversely affecting the other metrics. 
\end{itemize}

\section{Conclusion}\label{sec:Conclusion}
In this paper, we considered the problem of using data reduction methods to efficiently compute the $k$-means centers for a large high-dimensional dataset located at remote data source(s), with focus on DR and CR. Through a comprehensive analysis of the approximation error, the communication cost, and the complexity of various combinations of state-of-the-art DR/CR methods, we proved that it is possible to achieve a near-optimal approximation of $k$-means at a near-linear complexity at the data source(s) and a very low (constant or logarithmic) communication cost. In the process, we developed algorithms based on carefully designed combinations of existing DR/CR methods that outperformed two state-of-the-art algorithms in the scenarios of a single data source and multiple data sources, respectively. 
We also demonstrated how to combine DR/CR methods with quantizers to further reduce the communication cost without compromising the other performance metrics. 
Our findings were validated through experiments on real datasets. 

\ifCLASSOPTIONcompsoc
  \section*{Acknowledgments}
\else
  \section*{Acknowledgment}
\fi

This research was partly sponsored by the U.S. Army Research Laboratory and the U.K. Ministry of Defence under Agreement Number W911NF-16-3-0001. The views and conclusions contained in this document are those of the authors and should not be interpreted as representing the official policies, either expressed or implied, of the U.S. Army Research Laboratory, the U.S. Government, the U.K. Ministry of Defence or the U.K. Government. The U.S. and U.K. Governments are authorized to reproduce and distribute reprints for Government purposes notwithstanding any copyright notation hereon.

%



\bibliographystyle{IEEEtran}
\bibliography{TPDS_R1}

%

\begin{IEEEbiography}[{\includegraphics[width=1in,height=1.25in,clip,keepaspectratio]{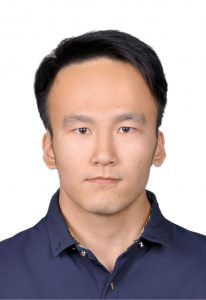}}]{Hanlin Lu}
(S'19) received the Ph.D. degree in Computer Science and Engineering from Pennsylvania State University in 2021. He is currently a Research Scientist at ByteDance, Mountain View, CA, USA. His research interests include coreset construction and distributed machine learning training.
\end{IEEEbiography}

\begin{IEEEbiography}[{\includegraphics[width=1in,height=1.25in,clip,keepaspectratio]{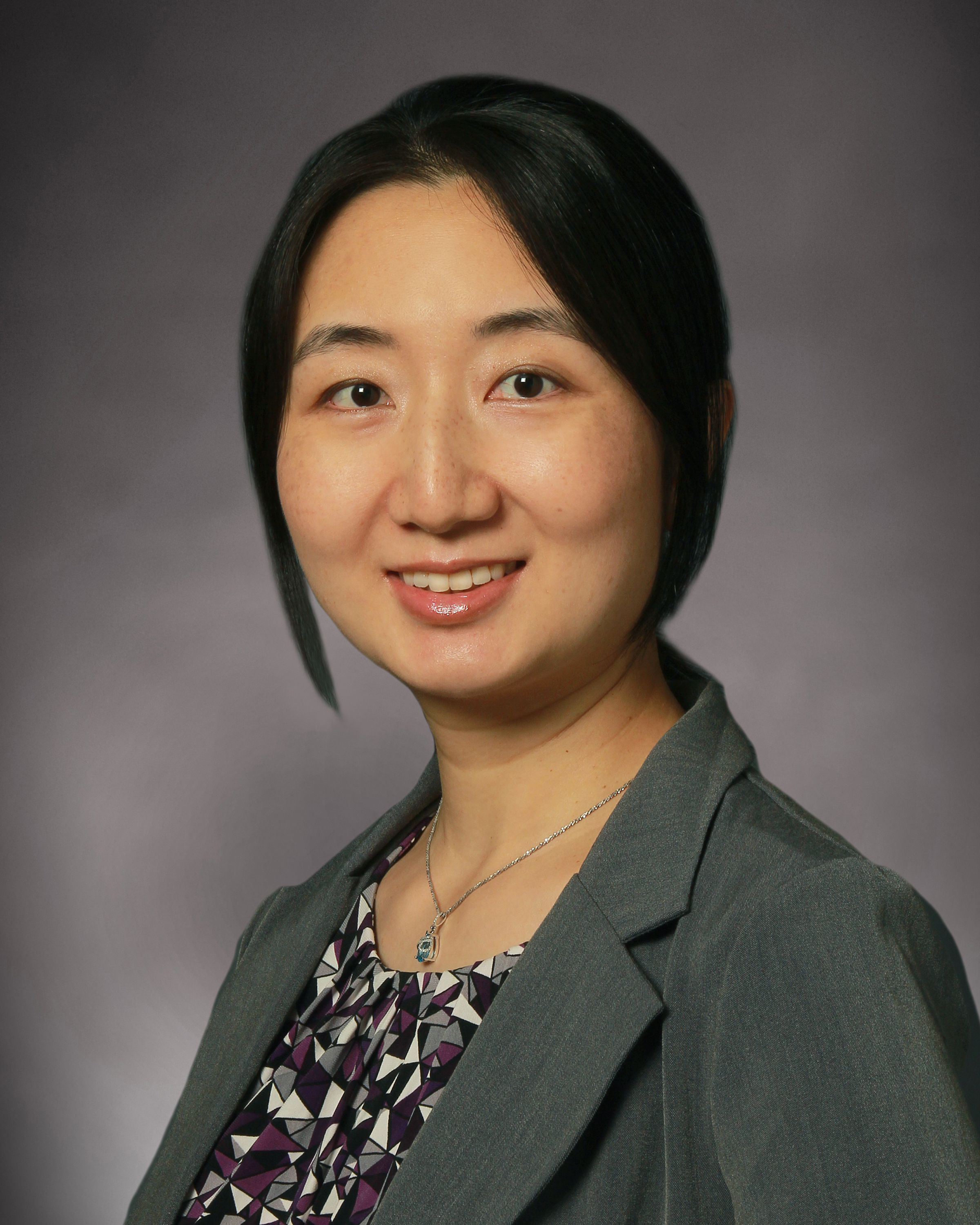}}]{Ting He} (SM’13) is an Associate Professor in the School of Electrical Engineering and Computer Science at Pennsylvania State University, University Park, PA. Her work is in the broad areas of computer networking, network modeling and optimization, and machine learning. Dr. He is a senior member of IEEE, an Associate Editor for IEEE Transactions on Communications (2017-2020) and IEEE/ACM Transactions on Networking (2017-2021), a TPC Co-Chair of IEEE ICCCN (2022), and an Area TPC Chair of IEEE INFOCOM (2021). She received multiple Outstanding Contributor Awards from IBM, multiple awards for Military Impact, Commercial Prosperity, and Collaboratively Complete Publications from ITA, and multiple paper awards from ICDCS, SIGMETRICS, ICASSP, and IEEE Communications Society. 
\end{IEEEbiography}

\begin{IEEEbiography}[{\includegraphics[width=1in,height=1.25in,clip,keepaspectratio]{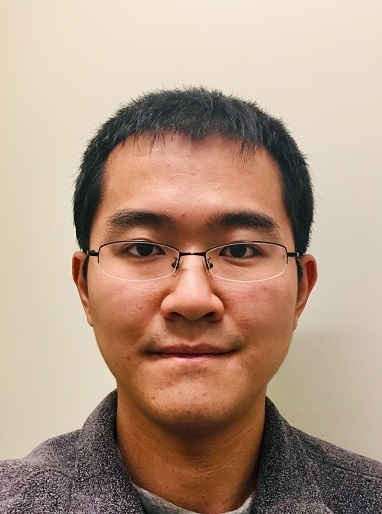}}]{Shiqiang Wang} (S’13–M’15) received his Ph.D. from the Department of Electrical and Electronic Engineering, Imperial College London, United Kingdom, in 2015. Before that, he received his master's and bachelor's degrees at Northeastern University, China, in 2011 and 2009, respectively. He has been a Research Staff Member at IBM T. J. Watson Research Center, NY, USA since 2016. His current research focuses on the intersection of distributed computing, machine learning, networking, and optimization, with a broad range of applications including data analytics, edge-based artificial intelligence (Edge AI), Internet of Things (IoT), and future wireless systems. Dr. Wang serves as an associate editor of the IEEE Transactions on Mobile Computing. He received the IEEE Communications Society (ComSoc) Leonard G. Abraham Prize in 2021, IEEE ComSoc Best Young Professional Award in Industry in 2021, IBM Outstanding Technical Achievement Awards (OTAA) in 2019 and 2021, multiple Invention Achievement Awards from IBM since 2016, Best Paper Finalist of the IEEE International Conference on Image Processing (ICIP) 2019, and Best Student Paper Award of the Network and Information Sciences International Technology Alliance (NIS-ITA) in 2015.
\end{IEEEbiography}

\begin{IEEEbiography}[{\includegraphics[width=1in,height=1.25in,clip,keepaspectratio]{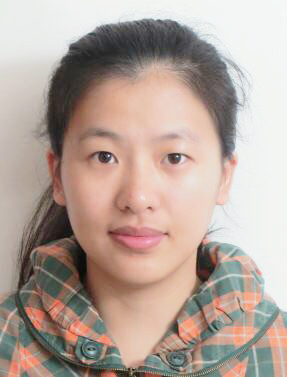}}]{Changchang Liu} received the Ph.D. degree in electrical engineering from Princeton University. She is currently a Research Staff Member with the Department of Distributed AI, IBM Thomas J. Watson Research Center, Yorktown Heights, NY, USA. Her current research interests include federated learning, big data privacy, and security.
\end{IEEEbiography}

\begin{IEEEbiography}[{\includegraphics[width=1in,height=1.25in,clip,keepaspectratio]{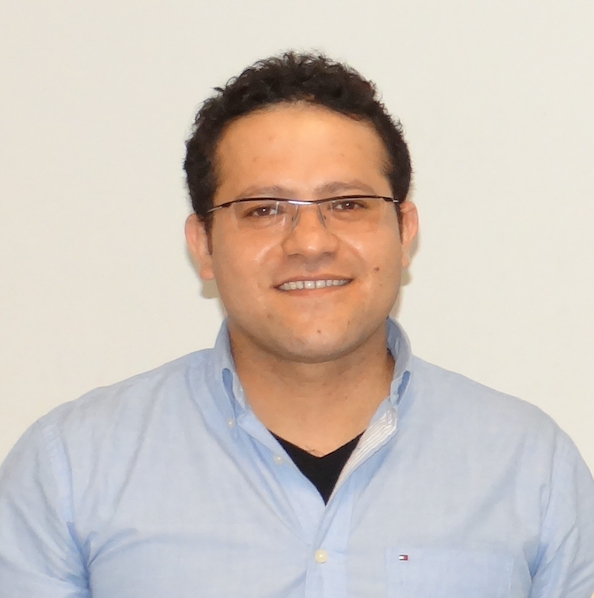}}]{Mehrdad Mahdavi} is an Assistant Professor of the Computer Science Department at the Penn State. He received the Ph.D. degree in Computer Science from Michigan State University in 2014.  Before joining PSU in 2018, he was a Research Assistant Professor at Toyota Technological Institute, at University of Chicago. His research interests lie at the interface of machine learning and optimization with a focus on developing theoretically principled and practically efficient algorithms for learning from massive datasets and complex domains. He has won the Mark Fulk Best Student Paper award at  Conference on Learning Theory (COLT) in 2012.
\end{IEEEbiography}

\begin{IEEEbiography}[{\includegraphics[width=1in,height=1.25in,clip,keepaspectratio]{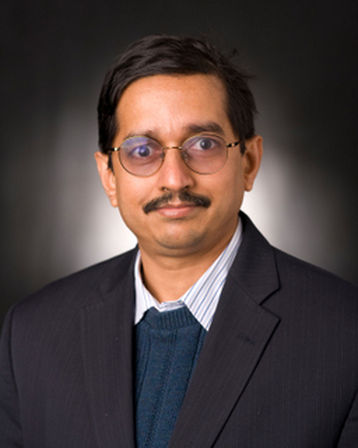}}]{Vijaykrishnan Narayanan} (F'11) is the Robert Noll Chair Professor of Computer Science and Engineering and Electrical Engineering at the Pennsylvania State University. His research interests are in Embedded System Design, Computer Architecture and Power-Aware Systems. He is a fellow of National Academy of Inventors, IEEE, and ACM.
\end{IEEEbiography}

\begin{IEEEbiography}[{\includegraphics[width=1in,height=1.25in,clip,keepaspectratio]{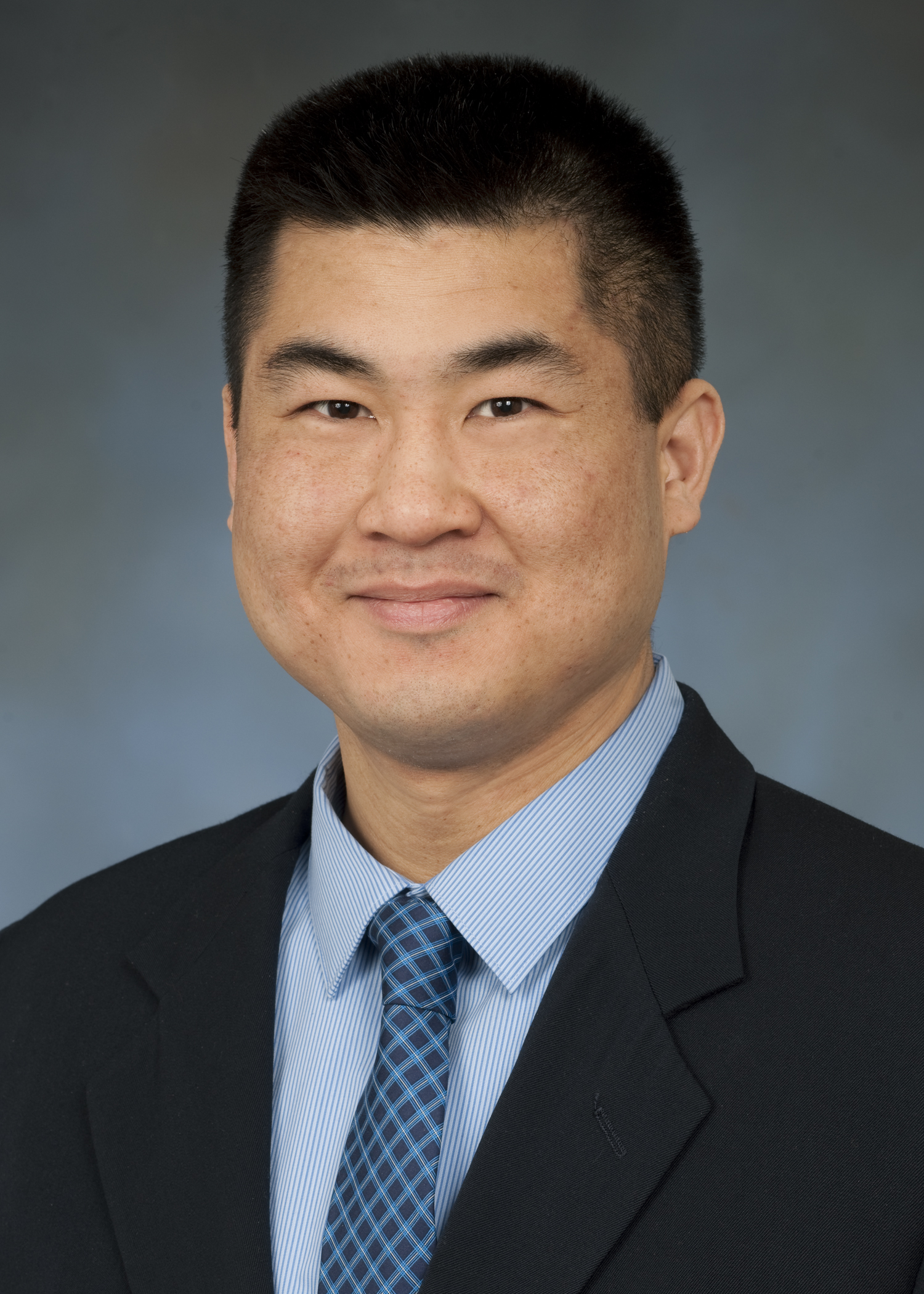}}]{Kevin S. Chan} (S’02–M’09–SM’18) received the B.S. degree in electrical and computer engineering and engineering and public policy from Carnegie Mellon University, Pittsburgh, PA, USA and the M.S. and Ph.D. degrees in electrical and computer engineering from the Georgia Institute of Technology, Atlanta, GA, USA. He is currently an Electronics Engineer with the Computational and Information Sciences Directorate, U.S. Army Combat Capabilities Development Command, Army Research Laboratory, Adelphi, MD, USA. He is actively involved in research on network science, distributed analytics, and cybersecurity. He received the 2021 IEEE Communications Society Leonard G. Abraham Prize and multiple best paper awards. He is the Co-Editor of the IEEE Communications Magazine—Military Communications and Networks Series.
\end{IEEEbiography}

\begin{IEEEbiography}[{\includegraphics[width=1in,height=1.25in,clip,keepaspectratio]{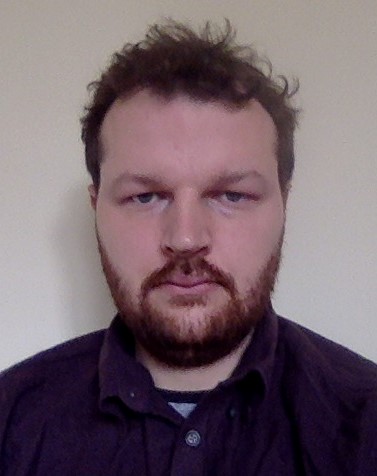}}]{Stephen Pasteris}
gained a BA+MA in Mathematics from Kings College of the University of Cambridge. After completing  his BA he then went on to gain a PhD in Computer Science from University College London: his thesis focusing on the development of efficient algorithms for machine learning on networked data. Stephen is now a Research Associate at University College London where he primarily researches online machine learning.
\end{IEEEbiography}








%

\clearpage
\appendices
\section{Proofs} \label{appendix: proofs}

\subsection{Proof of Theorem~\ref{coro:Feldmand13SODA}:} 
\begin{proof}
For 1), let $X^*$ denote the optimal $k$-means centers of $P$. Since $\Sbf$ is an $\epsilon$-coreset with probability $\geq 1-\delta$ (Theorem~\ref{thm:epsilon-coreset, Delta > 0}), by Definition~\ref{def:epsilon-coreset}, the following holds with probability $\geq 1-\delta$:
\begin{align}
\cost(P,X) \leq {1\over 1-\epsilon}\cost(\Sbf, X) \leq {1\over 1-\epsilon} \cost(\Sbf, X^*) \nonumber \\
\leq {1+\epsilon\over 1-\epsilon}\cost(P, X^*).
\end{align}

For 2), the cost of transferring $(S, \Delta, w)$ is dominated by the cost of transferring $S$. Since $S$ lies in a $d'$-dimensional subspace spanned by the columns of $V^{(d')}$, it suffices to transmit the coordinates of each point in $S$ in this subspace together with $V^{(d')}$. The former incurs a cost of $O(|S|\cdot d')$, and the latter incurs a cost of $O(d d')$. Plugging $d'=O(k/\epsilon^2)$ and $|S|=\tilde{O}(k^3/\epsilon^4)$ from  Theorem~\ref{thm:epsilon-coreset, Delta > 0} yields the overall communication cost as $O(dk/\epsilon^2)$.
\end{proof}

\subsection{Proof of Lemma~\ref{lem:JL projection}}
\begin{proof}
Let $\delta' = \delta/(2nk)$. By the JL Lemma (Lemma~\ref{lem:JL}), there exists $d' = O(\epsilon^{-2}\log(1/\delta')) = O(\epsilon^{-2}\log(nk/\delta))$, such that every $x\in\mathbb{R}^d$ satisfies $\|\pi(x)\| \approx_{1+\epsilon} \|x\|$ with probability $\geq 1-\delta'$. By the union bound, this implies that with probability $\geq 1-\delta$, every $p-x_i$ for $p\in P$ and $x_i\in X\cup X^*$ satisfies $\|\pi(p) - \pi(x_i)\| \approx_{1+\epsilon} \|p-x_i\|$. Therefore, with probability $\geq 1-\delta$, 
\begin{align}
\cost(\pi(P),\pi(X)) & = \sum_{p\in P} \min_{x_i\in X}\|\pi(p) - \pi(x_i)\|^2 \label{eq:JL projection proof 1} \\
&\leq \sum_{p\in P} \min_{x_i\in X} (1+\epsilon)^2 \|p - x_i\|^2 \nonumber \\
& = (1+\epsilon)^2 \cost(P, X), \label{eq:JL projection proof 2}
\end{align}
\begin{align}
\mbox{(\ref{eq:JL projection proof 1})} &\geq \sum_{p\in P}\min_{x_i\in X} {1\over (1+\epsilon)^2} \|p - x_i\|^2 \nonumber \\   
    & = {1\over (1+\epsilon)^2} \cost(P,X). \label{eq:JL projection proof 3}
\end{align}
Combining (\ref{eq:JL projection proof 2}) and (\ref{eq:JL projection proof 3}) proves (\ref{eq:JL projection cost(P,X)}). 
Similar argument will prove (\ref{eq:JL projection cost(P,X*)}). 
\end{proof}

\subsection{Proof of Theorem~\ref{thm:JL-based DR+CR}}
\begin{proof}
For 1), let $X^*$ be the optimal $k$-means centers of $P$ and $\Sbf':= (S',\Delta,w)$ be generated in line~\ref{JL DR+CR: S'} of Algorithm~\ref{Alg:JL DR+CR}. With probability at least $(1-\delta)^2$, $\pi_1$ satisfies (\ref{eq:JL projection cost(P,X)}, \ref{eq:JL projection cost(P,X*)}) and $\pi_2$ generates an $\epsilon$-coreset. Thus, with probability at least $(1-\delta)^2$,
\begin{align}
\cost(P,X) &\leq (1+\epsilon)^2 \cost(\pi_1(P), X') \label{eq:JL DR+CR proof 1} \\
&\leq {(1+\epsilon)^2\over 1-\epsilon} \cost(\Sbf', X') \label{eq:JL DR+CR proof 2} \\
&\leq {(1+\epsilon)^2\over 1-\epsilon} \cost(\Sbf', \pi_1(X^*)) \label{eq:JL DR+CR proof 3} \\
&\leq {(1+\epsilon)^3\over 1-\epsilon} \cost(\pi_1(P), \pi_1(X^*)) \label{eq:JL DR+CR proof 4} \\
& \leq {(1+\epsilon)^5\over 1-\epsilon} \cost(P, X^*), \label{eq:JL DR+CR proof 5}
\end{align}
where (\ref{eq:JL DR+CR proof 1}) is by (\ref{eq:JL projection cost(P,X)}) and that $\pi_1(X) = X'$, (\ref{eq:JL DR+CR proof 2}) is because $\Sbf'$ is an $\epsilon$-coreset of $\pi_1(P)$, (\ref{eq:JL DR+CR proof 3}) is because $X'$ minimizes $\cost(\Sbf',\cdot)$, (\ref{eq:JL DR+CR proof 4}) is again because $\Sbf'$ is an $\epsilon$-coreset of $\pi_1(P)$, and (\ref{eq:JL DR+CR proof 5}) is by (\ref{eq:JL projection cost(P,X*)}).

For 2), the communication cost is dominated by  transmitting $S'$. By Lemma~\ref{lem:JL projection}, the dimension of $P'$ is $d'=O(\epsilon^{-2}\log(nk/\delta)) = O(\epsilon^{-2}\log{n})$. By Theorem~\ref{thm:epsilon-coreset, Delta > 0}, the cardinality of $S'$ is $|S'| = O(k^3 \epsilon^{-4} \log^2(k)  \log(1/\delta))$. Moreover, points in $S'$ lie in a $\tilde{d}$-dimensional subspace for $\tilde{d} = O(k/\epsilon^2)$. Thus, it suffices to transmit the coordinates of points in $S'$ in the $\tilde{d}$-dimensional subspace and a basis of the subspace. 
Thus, the total communication cost is
\begin{align}
O((|S'|+d')\tilde{d}) &= O\left({k^4\over \epsilon^6}\log^2(k) \log({1\over \delta}) + {k\over \epsilon^4}\log{n}\right) \nonumber\\
&= O\left({k\log{n}\over \epsilon^4}\right).
\end{align}

For 3), note that for a given projection matrix $\Pi\in \mathbb{R}^{d\times d'}$ such that $\pi_1(P):= A_P \Pi$, line~\ref{JL DR+CR: P'} takes $O(ndd')=O(nd\epsilon^{-2} \log{n} )$ time, where we have plugged in $d'= O(\epsilon^{-2}\log{n})$. By Theorem~\ref{thm:epsilon-coreset, Delta > 0}, line~\ref{JL DR+CR: S'} takes time
\begin{align}
&O\left(\min(nd'^2, n^2 d') + {nk\over \epsilon^2}\big(d'+k\log({1\over \delta})\big)\right) \nonumber\\
= &O\left({n\over \epsilon^2}\Big({\log^2{n}\over \epsilon^2} + {k\log{n}\over \epsilon^2} + k^2\log({1\over \delta}) \Big)  \right).
\end{align}

Thus, the total complexity at the data source is:
\begin{align}
&O\left({n\over \epsilon^2}\Big({\log^2{n}\over \epsilon^2}+{k\log{n}\over \epsilon^2} + d\log{n} + k^2 \log({1\over \delta}) \Big) \right) \nonumber\\
= &O\left({nd\over \epsilon^2}\log^2{n} \right) = \tilde{O}\left({nd\over \epsilon^2}\right).
\end{align}

\end{proof}

\subsection{Proof of Lemma~\ref{lem:JL projection for S}}
\begin{proof}
The proof is analogous to that of Lemma~\ref{lem:JL projection}. Let $\delta' = \delta/(2n'k)$. Then there exists $d' = O(\epsilon^{-2}\log(1/\delta')) = O(\epsilon^{-2}\log(n'k/\delta))$, such that every $x\in \mathbb{R}^d$ satisfies $\|\pi(x)\|\approx_{1+\epsilon} \|x\|$ with probability $\geq 1-\delta'$. By the union bound, this implies that with probability $\geq 1-\delta$, every $p\in S$ and $x_i\in X\cup X^*$ satisfy $\|\pi(p)-\pi(x_i)\| \approx_{1+\epsilon}\|p-x_i\|$. Therefore,
\begin{align}
&\cost((\pi(S),\Delta,w), \pi(X)) \nonumber\\
&= \sum_{p\in S}w(p)\cdot \min_{x_i\in X} \|\pi(p) - \pi(x_i)\|^2 + \Delta \label{eq:JL projection for S proof}\\
&\leq (1+\epsilon)^2 \left(\sum_{p\in S}w(p)\cdot \min_{x_i\in X} \|p - x_i\|^2 + \Delta \right) \nonumber\\
&= (1+\epsilon)^2 \cost(\Sbf, X),
\end{align}
\begin{align}
\mbox{(\ref{eq:JL projection for S proof})} & \geq {1\over (1+\epsilon)^2} \left(\sum_{p\in S}w(p)\cdot \min_{x_i\in X} \|p - x_i\|^2 + \Delta \right) \nonumber\\
& =  {1\over (1+\epsilon)^2} \cost(\Sbf, X),
\end{align}
which prove (\ref{eq:JL for S, cost(S,X)}). Similar argument proves (\ref{eq:JL for S, cost(S,X*)}).
\end{proof}

\subsection{Proof of Theorem~\ref{thm:FSS CR+JL DR}}
\begin{proof}
For 1), let $X^*$ be the optimal $k$-means centers of $P$ and $\Sbf:= (S,\Delta,w)$ generated in line~\ref{FSS CR+DR:S} of Algorithm~\ref{Alg:FSS CR+JL DR}. With probability at least $(1-\delta)^2$, $\Sbf$ is an $\epsilon$-coreset of $P$, and $\pi_1$ satisfies (\ref{eq:JL for S, cost(S,X)}, \ref{eq:JL for S, cost(S,X*)}). Thus, with this probability, 
\begin{align}
\cost(P,X) & \leq {1\over 1-\epsilon}\cost(\Sbf, X) \label{eq:FSS CR+DR proof 1}\\
&\leq {(1+\epsilon)^2\over 1-\epsilon} \cost((S',\Delta,w), X') \label{eq:FSS CR+DR proof 2}\\
&\leq {(1+\epsilon)^2\over 1-\epsilon} \cost((S',\Delta,w), \pi_1(X^*)) \label{eq:FSS CR+DR proof 3}\\
&\leq {(1+\epsilon)^4\over 1-\epsilon} \cost(\Sbf, X^*) \label{eq:FSS CR+DR proof 4}\\
&\leq {(1+\epsilon)^5\over 1-\epsilon} \cost(P, X^*), \label{eq:FSS CR+DR proof 5}
\end{align}
where (\ref{eq:FSS CR+DR proof 1}) is because $\Sbf$ is an $\epsilon$-coreset of $P$, (\ref{eq:FSS CR+DR proof 2}) is due to (\ref{eq:JL for S, cost(S,X)}) (note that $\pi_1(S) = S'$ and $\pi_1(X) = X'$), (\ref{eq:FSS CR+DR proof 3}) is because $X'$ minimizes $\cost((S',\Delta,w),\cdot)$, (\ref{eq:FSS CR+DR proof 4}) is due to (\ref{eq:JL for S, cost(S,X*)}), and (\ref{eq:FSS CR+DR proof 5}) is again because $\Sbf$ is an $\epsilon$-coreset of $P$.

For 2), note that by Theorem~\ref{thm:epsilon-coreset, Delta > 0}, the cardinality of $S$ needs to be $n'=O(k^3 \epsilon^{-4} \log^2({k}) \log(1/\delta))$. By Lemma~\ref{lem:JL projection for S}, the dimension of $S'$ needs to be $d'=O(\epsilon^{-2}\log(n'k/\delta))$. Thus, the cost of transmitting $(S',\Delta,w)$, dominated by the cost of transmitting $S'$, is
\begin{align}
O(n' d') &= O\left({k^3\log^2{k}\over \epsilon^6} \log({1\over \delta}) \Big(\log{k} + \log({1\over \epsilon}) + \log({1\over \delta}) \Big) \right) \nonumber \\
&=\tilde{O}\left({k^3\over \epsilon^6}\right).
\end{align}


For 3), we know from Theorem~\ref{thm:epsilon-coreset, Delta > 0} that line~\ref{FSS CR+DR:S} of Algorithm~\ref{Alg:FSS CR+JL DR} takes time $O(\min(nd^2,n^2d)+nk \epsilon^{-2}(d+k\log(1/\delta)))$. Given a projection matrix $\Pi\in \mathbb{R}^{d\times d'}$ such that $\pi_1(S) := A_S \Pi$, line~\ref{FSS CR+DR:S'} takes time $O(n' d d')$. Thus, the total complexity at the data source is

\begin{align}
&O\hspace{-.15em}\left(\hspace{-.15em}\min\hspace{-.05em}(\hspace{-.15em}nd^2\hspace{-.25em}, \hspace{-.0em}n^2d\hspace{-.05em}) \hspace{-.25em}+\hspace{-.25em} {k\over \epsilon^2}nd \hspace{-.25em}+\hspace{-.25em} {k^2\log{k}\over \epsilon^2}n \hspace{-.25em}+\hspace{-.25em} {k^3 \hspace{-.25em}\log^3\hspace{-.25em}{k}(\log{k}\hspace{-.25em}+\hspace{-.25em} \log({1\over \epsilon}))\over \epsilon^6}d \hspace{-.25em}\right) \nonumber \\
&=O\left(nd \cdot \min(n,d)\right).
\end{align}

\end{proof}

\subsection{Proof of Theorem~\ref{thm:JL+FSS+JL}}
\begin{proof}
Let $n':=|S|$, $d'$ be the dimension after $\pi^{(1)}_1$, and $d''$ be the dimension after $\pi^{(2)}_1$. Let $X^*$ be the optimal $k$-means centers for $P$.

For 1), note that with probability $\geq (1-\delta)^3$, $\pi^{(1)}_1$ and $\pi^{(2)}_1$ will preserve the $k$-means cost up to a multiplicative factor of $(1+\epsilon)^2$, and $\pi_2$ will generate an $\epsilon$-coreset of $P'$. Thus, with this probability, we have
\begin{align}
\cost(P,X) &\leq (1+\epsilon)^2 \cost(P', \pi^{(1)}_1(X)) \label{eq:JL+FSS+JL proof 1} \\
&\leq {(1+\epsilon)^2\over 1-\epsilon} \cost((S, \Delta,w), \pi^{(1)}_1(X)) \label{eq:JL+FSS+JL proof 2} \\
&\hspace{-1em}\leq {(1+\epsilon)^4\over 1-\epsilon} \cost((S',\Delta,w), \pi^{(2)}_1\circ \pi^{(1)}_1(X)) \label{eq:JL+FSS+JL proof 3} \\
&\hspace{-1em}\leq {(1+\epsilon)^4\over 1-\epsilon} \cost((S',\Delta,w), \pi^{(2)}_1\circ \pi^{(1)}_1(X^*)) \label{eq:JL+FSS+JL proof 4} \\
&\leq {(1+\epsilon)^6\over 1-\epsilon} \cost((S,\Delta,w), \pi^{(1)}_1(X^*)) \label{eq:JL+FSS+JL proof 5} \\
&\leq {(1+\epsilon)^7\over 1-\epsilon} \cost(P', \pi^{(1)}_1(X^*)) \label{eq:JL+FSS+JL proof 6} \\
&\leq {(1+\epsilon)^9\over 1-\epsilon} \cost(P, X^*), \label{eq:JL+FSS+JL proof 7} 
\end{align}
where (\ref{eq:JL+FSS+JL proof 1}) is by Lemma~\ref{lem:JL projection}, (\ref{eq:JL+FSS+JL proof 2}) is because $(S,\Delta,w)$ is an $\epsilon$-coreset of $P'$, (\ref{eq:JL+FSS+JL proof 3}) is by Lemma~\ref{lem:JL projection for S}, (\ref{eq:JL+FSS+JL proof 4}) is because $\pi^{(2)}_1\circ \pi^{(1)}_1(X) = X'$, which is optimal in minimizing $\cost((S',\Delta,w), \cdot)$, (\ref{eq:JL+FSS+JL proof 5}) is by Lemma~\ref{lem:JL projection for S}, (\ref{eq:JL+FSS+JL proof 6}) is because $(S,\Delta,w)$ is an $\epsilon$-coreset of $P'$, and (\ref{eq:JL+FSS+JL proof 7}) is by Lemma~\ref{lem:JL projection}. 

For 2), note that by Theorem~\ref{thm:epsilon-coreset, Delta > 0}, the cardinality of the coreset constructed by FSS is $n'=O(k^3\log^2{k} \epsilon^{-4}\log(1/\delta))$, which is independent of the dimension of the input dataset. Thus, the communication cost remains the same as that of Algorithm~\ref{Alg:FSS CR+JL DR}, which is $\tilde{O}(k^3/\epsilon^6)$. 

For 3), note that the first JL projection $\pi^{(1)}_1$ takes $O(nd d')$ time, where $d' = O(\log{n}/\epsilon^2)$ by Lemma~\ref{lem:JL projection}, and the second JL projection $\pi^{(2)}_1$ takes $O(n' d' d'')$ time, where $n'$ is specified by Theorem~\ref{thm:epsilon-coreset, Delta > 0} as above and $d''=O(\epsilon^{-2}\log(n'k/\delta))$ by Lemma~\ref{lem:JL projection for S}. Moreover, from the proof of Theorem~\ref{thm:JL-based DR+CR}, we know that applying FSS after a JL projection takes $O({n\over \epsilon^2}(\log^2{n}/\epsilon^2 + k\log{n}/\epsilon^2 + k^2\log{1\over \delta}))$ time. Thus, the total complexity at the data source is
\begin{align}
&O\left(nd d' + {n\over \epsilon^2}\Big({\log^2{n}\over \epsilon^2} + {k\log{n}\over \epsilon^2} + k^2\log{1\over \delta}\Big) + n' d' d''\right) \nonumber \\
&\hspace{-1.25em}= O\left({nd\log{n}\over \epsilon^2} + {n\log^2{n}\over \epsilon^4}\right) = \tilde{O}\left({nd\over \epsilon^2}\right). 
\end{align}
\end{proof}

\subsection{Proof of Theorem~\ref{thm:BKLW}}
\begin{proof}
For 1), let $P:=\bigcup_{i=1}^m P_i$, $\tilde{P}:= \bigcup_{i=1}^m \tilde{P}_i$, and $\Sbf:=(S, 0, w)$ be the output of disSS. Let $X^*$ be the optimal $k$-means centers of $P$. By Theorem~\ref{thm:disSS}, we know that with probability $\geq 1-\delta$, 
\begin{align}
    \cost(\tilde{P},X) \leq {1\over 1-\epsilon}\cost(\Sbf, X) &\leq {1\over 1-\epsilon}\cost(\Sbf, X^*) \nonumber\\
    &\leq {1+\epsilon\over 1-\epsilon}\cost(\tilde{P}, X^*),\label{eq:BKLW proof 1}
\end{align}
where the second inequality is because $X$ is optimal for $\Sbf$. 
Moreover, by Theorem~\ref{thm:disPCA}, we have 
\begin{align}
    &(1-\epsilon)\cost(P,X)-\Delta \leq \cost(\tilde{P},X), \label{eq:BKLW proof 2}\\
    &\cost(\tilde{P},X^*) \leq (1+\epsilon)\cost(P, X^*) - \Delta. \label{eq:BKLW proof 3}
\end{align}
Combining (\ref{eq:BKLW proof 1}, \ref{eq:BKLW proof 2}, \ref{eq:BKLW proof 3}) yields
\begin{align}
    (1-\epsilon)\cost(P,X)-\Delta &\leq {1+\epsilon\over 1-\epsilon} \cdot \left( (1+\epsilon)\cost(P, X^*) - \Delta \right) \nonumber \\
    &\leq {(1+\epsilon)^2\over 1-\epsilon} \cost(P, X^*) - \Delta,
\end{align}
which gives the desired approximation factor. 

For 2), note that disPCA incurs a cost of $O(m\cdot (k/\epsilon^2) \cdot d)$ for transmitting $O(k/\epsilon^2)$ vectors in $\mathbb{R}^d$ from each of the $m$ data sources, and disSS incurs a cost of $O\left({k\over \epsilon^2} \cdot (\epsilon^{-4}({k^2\over \epsilon^2} + \log{1\over \delta}) + mk\log{mk\over \delta}) \right)$ for transmitting $O(\epsilon^{-4}({k^2\over \epsilon^2} + \log{1\over \delta}) + mk\log{mk\over \delta})$ vectors in $\mathbb{R}^{O(k/\epsilon^2)}$. For $d \gg m, k, 1/\epsilon,$ and $1/\delta$, the total communication cost is dominated by the cost of disPCA.

For 3), as computing the local SVD at data source $i$ takes $O(n_i d \cdot \min(n_i,d))$ time, the complexity of disPCA at the data sources is $O(nd \cdot \min(n, d))$. The complexity of disSS at data source $i$ is dominated by the computation of bicriteria approximation of $\tilde{P}_i$, which takes $O(n_i t_2 k \log{1\over \delta}) = O(n k^2 \epsilon^{-2}\log{1\over \delta})$ according to \cite{ADK09}. For $\min(n,d)\gg m, k, 1/\epsilon,$ and $1/\delta$, the overall complexity is dominated by that of disPCA. 
\end{proof}

\subsection{Proof of Lemma~\ref{lem:BKLW-based CR method}}
\begin{proof}
Let $\tilde{P}$ be the projection of $P$ using the principal components computed by disPCA. Then by Theorem~\ref{thm:disPCA}, there exists $\Delta\geq 0$ such that 
\begin{align}
\hspace{-.5em}    (1\hspace{-.25em}-\hspace{-.25em}\epsilon)\cost(P, X) \hspace{-.05em} \leq\hspace{-.05em}  \cost(\tilde{P}, X)\hspace{-.25em} +\hspace{-.25em} \Delta \hspace{-.05em}\leq\hspace{-.05em} (1\hspace{-.25em}+\hspace{-.25em}\epsilon) \cost(P, X). \label{eq:proof disPCA}
\end{align}
Moreover, by Theorem~\ref{thm:disSS}, $\Sbf$ is an $\epsilon$-coreset of $\tilde{P}$ with probability at least $1-\delta$. 
Multiplying (\ref{eq:proof disPCA}) by $1-\epsilon$, we have \looseness=-1
\begin{align}
    (1-\epsilon)^2 \cost(P,X) &\leq (1-\epsilon)\cost(\tilde{P}, X) + (1-\epsilon)\Delta \label{eq:proof BKLW 3} \\
    &\leq \cost(\Sbf, X)+\Delta, \label{eq:proof BKLW 1}
\end{align}
where we can obtain (\ref{eq:proof BKLW 1}) from (\ref{eq:proof BKLW 3}) because $\Sbf$ is an $\epsilon$-coreset of $\tilde{P}$. Similarly, multiplying (\ref{eq:proof disPCA}) by $1+\epsilon$, we have
\begin{align}
    (1+\epsilon)^2 \cost(P,X) &\geq (1+\epsilon)\cost(\tilde{P}, X) + (1+\epsilon)\Delta \nonumber \\
    &\geq \cost(\Sbf, X) + \Delta. \label{eq:proof BKLW 2}
\end{align}
Combining (\ref{eq:proof BKLW 1}, \ref{eq:proof BKLW 2}) yields the desired bound.  
\end{proof}

\subsection{Proof of Theorem~\ref{thm:distributed kmeans}}
\begin{proof}
For 1), let $\Sbf':=(\bigcup_{i=1}^m S'_i, \Delta, w)$, where $(\bigcup_{i=1}^m S'_i, 0, w)$ is the overall coreset constructed by line~\ref{distributed: pi_2} of Algorithm~\ref{Alg:distributed kmeans}, and $\Delta$ is a constant satisfying Lemma~\ref{lem:BKLW-based CR method} for the input dataset $\{P'_i\}_{i=1}^m$ as in line~\ref{distributed: pi_2} of Algorithm~\ref{Alg:distributed kmeans}. Let $P:=\bigcup_{i=1}^m P_i$, and $X^*$ be the optimal $k$-means centers for $P$. Then with probability $\geq (1-\delta)^2$, we have

\begin{align}
\cost(P,X) &\leq (1+\epsilon)^2 \cost(\pi_1(P), X') \label{eq:distributed, proof 1}\\
&\leq {(1+\epsilon)^2\over (1-\epsilon)^2} \cost(\Sbf', X') \label{eq:distributed, proof 2}\\
&\leq {(1+\epsilon)^2\over (1-\epsilon)^2} \cost(\Sbf', \pi_1(X^*)) \label{eq:distributed, proof 3}\\
&\leq {(1+\epsilon)^4\over (1-\epsilon)^2} \cost(\pi_1(P), \pi_1(X^*)) \label{eq:distributed, proof 4}\\
&\leq {(1+\epsilon)^6\over (1-\epsilon)^2} \cost(P, X^*), \label{eq:distributed, proof 5}
\end{align}
where (\ref{eq:distributed, proof 1}) is by Lemma~\ref{lem:JL projection} (note that $\pi_1(X) = X'$), (\ref{eq:distributed, proof 2}) is by Lemma~\ref{lem:BKLW-based CR method} (note that $\cost(\Sbf', X') = \cost((\bigcup_{i=1}^m S'_i, 0, w), X') + \Delta$), (\ref{eq:distributed, proof 3}) is because $X'$ is optimal in minimizing $\cost(\Sbf', \cdot)$, (\ref{eq:distributed, proof 4}) is again by Lemma~\ref{lem:BKLW-based CR method}, and (\ref{eq:distributed, proof 5}) is again by Lemma~\ref{lem:JL projection}. 

For 2), only line~\ref{distributed: pi_2} incurs communication cost. By Theorem~\ref{thm:BKLW}, we know that applying BKLW to a distributed dataset $\{P'_i\}_{i=1}^m$ with dimension $d'$ incurs a cost of $O(mkd'/ \epsilon^{2})$, and by Lemma~\ref{lem:JL projection}, we know that $d' = O(\log{n}/\epsilon^2)$, which yields the desired result. 

For 3), the JL projection at each data source incurs a complexity of $O(ndd') = O(nd \log{n}/\epsilon^2)$. By Theorem~\ref{thm:BKLW}, applying BKLW incurs a complexity of $O(nd' \cdot \min(n, d')) = O(n\log^2{n}/\epsilon^4)$ at each data source. Together, the complexity is $O({nd\over \epsilon^2}\log{n} + {n\over \epsilon^4}\log^2{n}) = \tilde{O}(nd/\epsilon^4)$. 
\end{proof}

\subsection{Proof of Theorem~\ref{thm: DR, CR, QT}}
\begin{proof}
We only present the proof for Algorithm~\ref{Alg:JL+FSS+JL} with the incorporation of quantization, as the proofs for the other algorithms are similar.  Consider a coreset $(S,\Delta,w)$ and a set of $k$-means centers $X$. If we quantize $S$ into $S_{QT}$ with a maximum quantization error of $\Delta_{QT}$, then for each coreset point $q \in S$ and its quantized version $q'\in S_{QT}$, we have $\|q-q'\| \leq \Delta_{QT}$. On the other hand, from \cite{lu2020robust}, the $k$-means cost function is $2\Delta_D$-Lipschitz-continuous, which yields $|\cost(q, X)-\cost(q', X)| \leq 2\Delta_D\Delta_{QT}$. Thus, the difference in the $k$-means cost between the original and the quantized coresets is bounded by
\begin{align}\label{eq:quanitzation - coreset}
\left|\cost((S,\Delta,w),X)-\cost((S_{QT},\Delta,w),X)\right| \nonumber \\
\leq 2\Delta_D \Delta_{QT} \sum_{q\in S}w(q),    
\end{align}
as $\cost((S,\Delta,w),X) = \sum_{q\in S}w(q) \cost(q,X) + \Delta$.

Following the arguments in the proof of Theorem~\ref{thm:JL+FSS+JL}, we see that  with probability $\geq (1-\delta)^3$:
\begin{align}
&\cost(P,X) \nonumber \\
&\leq (1+\epsilon_1^{(1)})^2 \cost(P', \pi^{(1)}_1(X))  \\
&\leq {(1+\epsilon_1^{(1)})^2\over 1-\epsilon_2} \cost((S, \Delta,w), \pi^{(1)}_1(X)) \\
&\leq {(1+\epsilon_1^{(1)})^2(1+\epsilon_1^{(2)})^2\over 1-\epsilon_2} \cost((S',\Delta,w), \pi^{(2)}_1\circ \pi^{(1)}_1(X)) \\
&\leq {(1+\epsilon_1^{(1)})^2(1+\epsilon_1^{(2)})^2\over 1-\epsilon_2} \cdot \nonumber \\ 
&~~~~(\cost((S'_{QT},\Delta,w), \pi^{(2)}_1\circ \pi^{(1)}_1(X)) + 2n\Delta_D\Delta_{QT}) \label{eq:JL+FSS+JL+qt proof 4}\\
&\leq {(1+\epsilon_1^{(1)})^2(1+\epsilon_1^{(2)})^2\over 1-\epsilon_2} \cdot \nonumber \\ 
&~~~~(\cost((S'_{QT},\Delta,w), \pi^{(2)}_1\circ \pi^{(1)}_1(X^*))  + 2n\Delta_D\Delta_{QT} ) \label{eq:JL+FSS+JL+qt proof 5} \\
&\leq {(1+\epsilon_1^{(1)})^2(1+\epsilon_1^{(2)})^2\over 1-\epsilon_2} \cdot \nonumber \\
&~~~~(\cost((S',\Delta,w), \pi^{(2)}_1\circ \pi^{(1)}_1(X^*))  + 4n\Delta_D\Delta_{QT} ) \label{eq:JL+FSS+JL+qt proof 6} \\
&\leq {(1+\epsilon_1^{(1)})^2(1+\epsilon_1^{(2)})^4\over 1-\epsilon_2} \cost((S,\Delta,w), \pi^{(1)}_1(X^*)) \nonumber \\
&~~~~+ {(1+\epsilon_1^{(1)})^2(1+\epsilon_1^{(2)})^2\over 1-\epsilon_2} 4n\Delta_D\Delta_{QT}\\
&\leq {(1+\epsilon_1^{(1)})^2 (1+\epsilon_2) (1+\epsilon_1^{(2)})^4\over 1-\epsilon_2} \cost(P', \pi^{(1)}_1(X^*)) \nonumber \\
&~~~~+ {(1+\epsilon_1^{(1)})^2(1+\epsilon_1^{(2)})^2\over 1-\epsilon_2} 4n\Delta_D\Delta_{QT} \\
&\leq {(1+\epsilon_1^{(1)})^4 (1+\epsilon_2) (1+\epsilon_1^{(2)})^4\over 1-\epsilon_2} \cost(P, X^*) \nonumber \\
&~~~~+ {(1+\epsilon_1^{(1)})^2(1+\epsilon_1^{(2)})^2\over 1-\epsilon_2} 4n\Delta_D\Delta_{QT}, 
\end{align}
where \eqref{eq:JL+FSS+JL+qt proof 4} and \eqref{eq:JL+FSS+JL+qt proof 6} are by \eqref{eq:quanitzation - coreset} and the property that the coreset $(S,\Delta,w)$ constructed by sensitivity sampling 
satisfies $\sum_{q\in S}w(q)=n$ (the cardinality of $P$)\footnote{While the  sensitivity sampling procedure in \cite{Feldman13SODA:report} only guarantees that $\mathbb{E}[\sum_{q\in S}w(q)]=n$ (expectation over $S$), a variation of this procedure proposed in \cite{Balcan13NIPS} guarantees $\sum_{q\in S}w(q)=n$ deterministically. FSS based on the sampling procedure in \cite{Balcan13NIPS} still generates an $\epsilon$-coreset (with probability $\geq 1-\delta$) with a constant cardinality (precisely, $O({k^2\over \epsilon^6}\log({1\over \delta})$).  }. 
\end{proof}



\ifCLASSOPTIONcaptionsoff
  \newpage
\fi



\end{document}